\pdfoutput=1
\documentclass[10pt,twocolumn]{article}

\usepackage[margin=0.75in,columnsep=0.25in]{geometry}

\usepackage{float}
\usepackage{placeins} % for \FloatBarrier


\usepackage{microtype}
\usepackage{tabularx}
\usepackage{graphicx}
\usepackage{subcaption}
\usepackage{enumitem}
\usepackage{booktabs} % for professional tables

\usepackage[authoryear,round]{natbib}
\usepackage{hyperref}

\usepackage{url}

\usepackage{amsmath}
\usepackage{amssymb}
\usepackage{mathtools}
\usepackage{amsthm}
\usepackage{xcolor}

\usepackage{algorithm}
\usepackage{algorithmic}
\usepackage[capitalize,noabbrev]{cleveref}

\usepackage{amsmath,amsfonts,bm,physics,bbm}



\usepackage{algorithmicx} %algorithm
\usepackage{algpseudocode}

\usepackage{mathrsfs}
\usepackage{subcaption}   % modern package for subfigures

\DeclareMathOperator\supp{{\mathrm{supp}}}

\renewcommand{\eqref}[1]{\hyperref[#1]{(\ref{#1})}}

\crefname{ALG@line}{line}{lines}
\Crefname{ALG@line}{Line}{Lines}

\makeatletter
\newcommand{\labelline}[1]{%
  \addtocounter{ALG@line}{-1}% go back to the number actually printed
  \refstepcounter{ALG@line}% register this value with cleveref
  \label{#1}%
}
\makeatother

\def\floor#1{\lfloor #1 \rfloor}
\def\1{\bm{1}}

\def\eps{{\epsilon}}

\def\tH{{\textbf{H}}}
\def\tV{{\textbf{V}}}

\def\mI{{\bm{I}}}

\DeclareMathAlphabet{\mathsfit}{\encodingdefault}{\sfdefault}{m}{sl}
\SetMathAlphabet{\mathsfit}{bold}{\encodingdefault}{\sfdefault}{bx}{n}
\newcommand{\tens}[1]{\bm{\mathsfit{#1}}}
\def\tA{{\tens{A}}}

\def\tG{{\tens{G}}}
\def\tH{{\tens{H}}}

\def\tT{{\tens{T}}}

\def\tV{{\tens{V}}}
\def\tW{{\tens{W}}}

\def\gA{{\mathcal{A}}}
\def\gB{{\mathcal{B}}}

\def\gD{{\mathcal{D}}}
\def\gE{{\mathcal{E}}}
\def\gF{{\mathcal{F}}}

\def\gK{{\mathcal{K}}}
\def\gL{{\mathcal{L}}}
\def\gM{{\mathcal{M}}}
\def\gN{{\mathcal{N}}}

\def\gP{{\mathcal{P}}}

\def\gR{{\mathcal{R}}}
\def\gS{{\mathcal{S}}}
\def\gT{{\mathcal{T}}}
\def\gU{{\mathcal{U}}}

\def\gX{{\mathcal{X}}}
\def\gY{{\mathcal{Y}}}

\def\sN{{\mathbb{N}}}

\def\ccut{\text{count1}_{\text{cutoff}}}
\def\ccuttwo{\text{count2}_{\text{cutoff}}}

\newcommand{\prob}[1]{\mathbb{P}\left[#1\right]}
\newtheorem{claim}{Claim}

\newcommand{\E}{\mathbb{E}}

\newcommand{\R}{\mathbb{R}}

\newcommand{\KL}{D_{\mathrm{KL}}}
\newcommand{\Var}{\mathrm{Var}}

\newcommand{\hstarutau}{\tH_\tau^{(u),*}}
\newcommand{\hstarltau}{\tH_\tau^{(l),*}}
\newcommand{\hstarut}{\tH_t^{(u),*}}
\newcommand{\hstarlt}{\tH_t^{(l),*}}

\newcommand{\lambdal}{\lambda_{\text{min}}}
\newcommand{\lambdau}{\lambda_{\text{max}}}

\newcommand{\mone}{\textbf{M1}}
\newcommand{\mtwo}{\textbf{M2}}

\DeclareMathOperator*{\argmax}{arg\,max}
\DeclareMathOperator*{\argmin}{arg\,min}

\crefname{algline}{line}{lines}  % Singular and plural forms
\Crefname{algline}{Line}{Lines}  % Capitalized versions

\newtheorem{observation}{Observation}
\newtheorem*{theorem*}{Theorem}
\newtheorem*{proposition*}{Proposition}
\newtheorem*{corollary*}{Corollary}

\newcommand{\algfootnote}[1]{%
    \leavevmode\unskip\strut\vadjust{\hbox to 0pt{\hss\footnotemark}}%
    \addtocounter{footnote}{-1}\footnotetext{#1}%
}
\theoremstyle{plain}
\newtheorem{theorem}{Theorem}[section]

\newtheorem{lemma}[theorem]{Lemma}
\newtheorem{corollary}[theorem]{Corollary}
\theoremstyle{definition}
\newtheorem{definition}[theorem]{Definition}
\newtheorem{assumption}[theorem]{Assumption}
\theoremstyle{remark}
\newtheorem{remark}[theorem]{Remark}

\usepackage{titling}

\allowdisplaybreaks

\begin{document}

% ============ TITLE ============
\title{Shuffle and Joint Differential Privacy for\\ Generalized Linear Contextual Bandits}

\author{
Sahasrajit Sarmasarkar\\
\textit{Stanford University}\\
\texttt{sahasras@stanford.edu}
}
\date{}

\twocolumn[
  \begin{@twocolumnfalse}
    \maketitle
    \begin{abstract}
We present the first algorithms for generalized linear contextual bandits under shuffle differential privacy and joint differential privacy. While prior work on private contextual bandits has been restricted to linear reward models---which admit closed-form estimators---generalized linear models (GLMs) pose fundamental new challenges: no closed-form estimator exists, requiring private convex optimization; privacy must be tracked across multiple evolving design matrices; and optimization error must be explicitly incorporated into regret analysis.

We address these challenges under two privacy models and context settings. For stochastic contexts, we design a shuffle-DP algorithm achieving $\tilde{O}(d^{3/2}\sqrt{T \log T}/\sqrt{\varepsilon})$ regret in dominant term, differing from the non-private rate by a factor of $\sqrt{d/\varepsilon}$. For adversarial contexts, we provide a joint-DP algorithm with regret $\tilde{O}\!\big(d\sqrt{T} \log T + d^{3/4}\sqrt{T/\varepsilon}\,(\log T)\,(d + \log T)^{1/4}\big)$---matching the non-private rate $\tilde{O}(d\sqrt{T} \log T)$ in the leading term, with privacy contributing only an additive correction. Unlike prior work on locally private GLM bandits, our methods require no spectral assumptions on the context distribution beyond $\ell_2$ boundedness.
    \end{abstract}
    \vspace{1em}
  \end{@twocolumnfalse}
]


\paragraph{Code.}
Available at \url{https://github.com/Sahasrajit123/Shuffle_joint_dp_generalized_linear_contextual_bandits}.

% =========================================================
% Main content
% =========================================================
% ====================
% REVISED INTRODUCTION - Compatible with ICML 2026 template
% Expanded technical overview addresses reviewer concerns about novelty
% ====================

\section{Introduction}

In the contextual bandit problem \citep{NIPS2011_e1d5be1c,pmlr-v15-chu11a}, at each time step $t$, the learner observes a set of $K$ arms $\mathcal{X}_t$, each represented by a $d$-dimensional feature vector. The learner selects an arm $x_t \in \mathcal{X}_t$ and receives a reward $r_t$. The goal is to minimize cumulative regret over $T$ rounds. This problem has been extensively studied for linear models \citep{10.5555/944919.944941,NIPS2011_e1d5be1c}, logistic models \citep{faury2020improvedoptimisticalgorithmslogistic,faury2022jointlyefficientoptimalalgorithms}, and generalized linear models \citep{NIPS2010_c2626d85,li2017provablyoptimalalgorithmsgeneralized,sawarni2024generalizedlinearbanditslimited}.

Given the wide applicability of contextual bandits in settings involving sensitive personal data---such as news recommendation \citep{newsarticlerecomcontextualbandits}, healthcare \citep{healthcarecontextualbandits}, and education \citep{educationcontextualbandits}---protecting both context and reward information from adversaries is paramount. This motivates the study of \emph{differentially private} contextual bandits.

In this paper, we study GLMs with DP constraints under two setups. In \textbf{M1}, the context set $\mathcal{X}_t$ is drawn i.i.d.\ from an unknown distribution $\mathcal{D}$, and we study $(\varepsilon, \delta)$-shuffle DP. In \textbf{M2}, the context sets may be chosen adversarially, and we study $(\varepsilon, \delta)$-joint DP. Under shuffle DP, each user transmits noisy messages to a shuffler that randomly permutes all messages before the learner receives them. Joint DP requires that the action at time step $t$ be differentially private with respect to the context-reward pairs at all other time steps.

\paragraph{Contributions.} Our main contributions are as follows. The simplified rates given in this section assume $T \geq d$. Our complete, all-$T$ regret analyses are given in \Cref{thm:shuffled_regret_glm} of Appendix~D and \Cref{thm:jdp_regret_glm} of Appendix~I.
\begin{itemize}
    \item \textbf{First GLM bandits under shuffle/joint DP.} To the best of our knowledge, this is the first work to address either shuffle DP or joint DP in generalized linear contextual bandits. Prior private GLM bandits \citep{han2021generalizedlinearbanditslocal} study only local DP under restrictive spectral assumptions; we require only $\ell_2$ boundedness.

    \item \textbf{Shuffle DP for stochastic contexts (M1).} We design a shuffle-DP algorithm that achieves regret $\tilde{O}(d^{3/2}\sqrt{T \log T}/\sqrt{\varepsilon})$ (\Cref{inf_thm:privacy_utility_m1}) in dominant term. This differs from the non-private rate of \citet{sawarni2024generalizedlinearbanditslimited} by only a factor of $\sqrt{d/\varepsilon}$.

    \item \textbf{Joint DP for adversarial contexts (M2).} We give a joint-DP algorithm with regret $\tilde{O}\!\big(d\sqrt{T}\log T + d^{3/4}\sqrt{T/\varepsilon}\,(\log T)\,(d+\log T)^{1/4}\big)$ (\Cref{thm:jdp_main}), matching the non-private rate $\tilde{O}(d\sqrt{T}\log T)$ in the leading term; privacy contributes only an additive correction. Such additive privacy-induced regret penalties have also been observed in linear contextual bandits by \citet{pmlr-v247-azize24a}.
    
    \item \textbf{$\kappa$-free leading term.} Our regret bounds inherit the $\kappa$-free leading $\sqrt T$ term from the non-private framework of \citet{sawarni2024generalizedlinearbanditslimited}; preserving this under privacy is not automatic and requires care in how noise enters the design matrix. The parameter $\kappa$ (which can grow as $e^S$ for logistic rewards, where $S$ upper-bounds $\|\theta^\ast\|$) thus appears only in sub-leading terms. Only the benign $\kappa^\ast$ remains in the leading scaling: $\kappa^{\ast} \geq 1/R^2$ in general (where $R$ is the reward bound), with sharper link-specific bounds available (e.g., $\kappa^{\ast} \geq 4$ for logistic). We illustrate this empirically in \cref{sec:simulations}.
\end{itemize}

\begin{remark}[Anytime extension]
While our algorithms require the horizon $T$ to be known, they can be made anytime via the standard doubling trick \citep{besson2018doubling}, incurring only a constant-factor increase in regret. 
\end{remark}

% ====================
% NEW: TECHNICAL OVERVIEW SUBSECTION
% This addresses reviewer concerns about unclear novelty
% ====================

\subsection{Technical Overview}
\label{sec:tech_overview}

Extending private contextual bandits from linear to generalized linear models introduces three fundamental obstacles, each requiring a new technique rather than off-the-shelf composition. We sketch them here; full treatment is in \cref{sec:glm_alg,sec:glm_jdp}.

\paragraph{Challenge 1: No closed-form estimator.}
Linear bandits admit a closed-form ridge estimator $\widehat\theta=(\mathbf X^\top\mathbf X+\lambda\mathbf I)^{-1}\mathbf X^\top\mathbf r$, so privacy reduces to noising the sufficient statistics. GLMs have no such closed form: the MLE requires iterative convex optimization, which both consumes privacy budget and contributes optimization error to the confidence radius. We integrate the shuffle-private optimizer of \citet{Cheu2021ShufflePS} into the estimation subroutine and show that optimization error $\nu$ enters the confidence radius only as $O(\sqrt\nu)$ (\Cref{lemma:bounding_est_theta_theta}), letting us trade iteration count for utility.

\paragraph{Challenge 2: Non-monotone design matrices under continual release.}
In the adversarial setting (M2), the binary-tree mechanism \citep{continualstatsrelease,continualobservation} is needed to release design matrices at every round, but its noise makes the released $\{\tH_t\}$ non-monotone, breaking the standard determinant-doubling rarely-switching argument \citep{NIPS2011_e1d5be1c}. We replace it with a \emph{directional growth criterion} (line~\ref{line:policy_II_switch} of \cref{alg:jdp_glm_adversarial}) which on noiseless monotone matrices implies determinant doubling; \Cref{lemma:policy_switch_bound_general} shows it preserves the $O(d\log T)$ switch bound.

\paragraph{Challenge 3: Privacy leakage through policy-switching times.}
In \cref{alg:jdp_glm_adversarial}, whether round $t$ triggers an exploration switch depends on $\mathcal X_t$ and on $\mathbf V_t$ (which itself depends on all prior switching decisions), so the index set $\mathcal T_o$ is data-dependent and could leak information about other users. We show that $\mathcal T_o\setminus\{t\}$ is $(\varepsilon,\delta)$-indistinguishable across single-row neighbors via a log-likelihood-ratio decomposition over the binary tree (\Cref{lem:switching_privacy}, full proof in \Cref{lemma:utility_privacy_glm_adv_2})---inspired by \citep{kairouz2015compositiontheoremdifferentialprivacy,fullyadaptiveprivacy} and, to our knowledge, new to the bandit literature.

% \paragraph{Removing $\kappa$ from the leading regret term.}
% A key feature of our bounds is that the instance-dependent parameter $\kappa$---which can be exponential in the norm bound on $\theta^\ast$ (e.g., $\kappa=\Theta(e^S)$ for logistic regression, where $S$ upper-bounds $\|\theta^\ast\|$)---does not appear in the leading $\sqrt{T}$ term, which carries only a $1/\sqrt{\kappa^{\ast}}$ factor. This matches the non-private guarantees of \citet{sawarni2024generalizedlinearbanditslimited}; in general $\kappa^{\ast} \geq 1/R^2$ (where $R$ denotes the bound on per-step reward), independent of $d$ and $S$, with sharper link-specific bounds available (e.g., $\kappa^{\ast} \geq 4$ for logistic). We achieve this via a two-phase approach: an initial exploration phase builds a coarse estimate with $\kappa$-dependent regret, and in subsequent rounds $\kappa$ is absorbed into the confidence-set geometry by appropriately scaling the design matrix used for confidence bounds---while the parameter estimate itself is still computed by an ordinary (private) convex optimizer on the standard GLM log-likelihood. We defer the precise formulation to the remarks in \cref{sec:glm_alg} (shuffle) and \cref{sec:glm_jdp} (joint DP), and empirically validate this phenomenon in \cref{sec:simulations,sec:appendix_simulations}.

% ====================
% SECTION 2: RELATED WORK (Enhanced with technique comparison)
% ====================

\section{Related Work}
\label{sec:related}

\paragraph{Non-private GLM bandits.}
Generalized linear contextual bandits have been studied extensively in the non-private setting \citep{NIPS2010_c2626d85,li2017provablyoptimalalgorithmsgeneralized,sawarni2024generalizedlinearbanditslimited}. The key challenge is that, unlike linear models, GLMs lack closed-form estimators for $\widehat{\theta}$, requiring iterative optimization. \citet{sawarni2024generalizedlinearbanditslimited} recently achieved order-optimal regret with limited adaptivity, which inspires our algorithmic framework.

\paragraph{Private linear bandits.}
Differential privacy in linear contextual bandits has been studied under various trust models. Under joint DP, \citet{shariff2018differentiallyprivatecontextuallinear} first achieved sublinear regret using tree-based mechanisms; \citet{azize2024concentrateddifferentialprivacybandits} improved bounds under relaxed assumptions. Under shuffle DP, \citet{pmlr-v162-chowdhury22a,pmlr-v167-garcelon22a} obtained amplified privacy guarantees. All these works exploit closed-form ridge estimators for linear models.

\paragraph{Private GLM bandits.}
%The only prior work on private GLM bandits is \cite{han2021generalizedlinearbanditslocal}, which studies local DP under restrictive spectral assumptions on the context distribution. In addition their regret bound in \cite[Theorem 3.1]{han2021generalizedlinearbanditslocal} scales inversely proportional to constant $\kappa_l$ and $p^\ast$ as defined in \cite[Assumption 3]{han2021generalizedlinearbanditslocal}. In \cite[Appendix A]{han2021generalizedlinearbanditslocal}, the authors show that $\kappa_l$ scales proportional to the smallest eigen values for gaussian context vectors leading to the regret bound depending inversely on the smallest eigen value. Our work has no such dependence and addresses shuffle and joint DP for GLM bandits achieving $\widetilde{O}(\sqrt{T/\eps})$ regret under both privacy models with only $\ell_2$ boundedness assumptions.

The only prior work on private GLM bandits, \citet{han2021generalizedlinearbanditslocal}, considers the local DP setting under restrictive spectral assumptions: their regret bound scales inversely with a constant $\kappa_l$ (\cite[Assumption~3]{han2021generalizedlinearbanditslocal}) which, as shown in \cite[Appendix~A]{han2021generalizedlinearbanditslocal}, is proportional to the minimum eigenvalue of the context covariance for Gaussian contexts---so the regret depends inversely on this minimum eigenvalue. In contrast, our results remove any such spectral dependence and achieve $\widetilde{O}(\sqrt{T/\eps})$ regret under both shuffle and joint DP assuming only $\ell_2$-bounded contexts.

\paragraph{Comparison with prior techniques.}
We highlight the key technical differences from prior private bandit algorithms that necessitate our new techniques:
\begin{enumerate}[leftmargin=*,itemsep=1pt,topsep=2pt]
\item \citet{shariff2018differentiallyprivatecontextuallinear,azize2024concentrateddifferentialprivacybandits} study joint-DP linear bandits using tree-based mechanisms for the ridge estimator. Their analysis exploits the closed-form solution $\widehat{\theta} = (\mathbf{X}^\top\mathbf{X} + \lambda\mathbf{I})^{-1}\mathbf{X}^\top\mathbf{r}$, where adding noise to $\mathbf{X}^\top\mathbf{X}$ and $\mathbf{X}^\top\mathbf{r}$ suffices. For GLMs, we must instead privatize the optimization process itself.
    
    %\item \citet{pmlr-v162-chowdhury22a,pmlr-v167-garcelon22a} study shuffle-DP linear bandits with adversarial contexts under batching. Their regret bounds scale as $T^{3/5}$, while we achieve $T^{1/2}$ by leveraging the structure of GLM confidence sets and the limited-adaptivity framework of \citet{sawarni2024generalizedlinearbanditslimited}.

    \item \citet{sawarni2024generalizedlinearbanditslimited} provide the non-private baseline we build upon. Their key insight---using policy switches based on design matrix determinants to remove $\kappa$ from regret---does not directly transfer to the private setting because (a) the switch times leak information, and (b) noisy matrices are non-monotone.
\end{enumerate}

See Appendix~\ref{app:related} for extended discussion of related work on privacy amplification by shuffling and limited adaptivity in bandits.

% ====================
% SECTION 3: PRIVACY MODELS (Kept mostly the same)
% ====================

\section{Privacy Models}
\label{sec:privacy}

We consider two privacy models for contextual bandits, depending on the context generation process.

\paragraph{Shuffle DP (stochastic contexts).}
In the shuffle model \citep{erlingsson2019amplification,Cheu_2019}, each user applies a local randomizer to their data and sends noisy messages to a trusted shuffler, which randomly permutes all messages before forwarding them to the learner. This amplifies privacy: even weak local randomization yields strong central-DP-level guarantees after shuffling. We adopt the multi-round shuffle protocol of \citet{Cheu2021ShufflePS,pmlr-v162-chowdhury22a}, where users participate across multiple batches.

\paragraph{Joint DP (adversarial contexts).}
Joint DP \citep{mechanism_design_large_games,shariff2018differentiallyprivatecontextuallinear} is a relaxation appropriate when contexts are adversarial. It requires that the action $a_t$ recommended at time $t$ be differentially private with respect to the context-reward pairs $(\mathcal{X}_s, r_s)$ at all other times $s \neq t$. This avoids the impossibility result of \citep{shariff2018differentiallyprivatecontextuallinear}, which shows that standard central DP incurs linear regret in contextual bandits.

Formal definitions of local DP, shuffle DP, and joint DP appear in Appendix~\ref{app:privacy_defs}.

\paragraph{Comparison with prior work.}
Table~\ref{tab:comparison} summarizes how our results relate to prior work on private contextual bandits. All prior shuffle/joint DP results assume linear models with closed-form ridge estimators. Prior GLM work \citep{han2021generalizedlinearbanditslocal} uses only local DP and requires spectral assumptions.

\begin{table}[t]
\centering
\vspace{1mm}
\small
\setlength{\tabcolsep}{4pt} % tighten columns a bit
\begin{tabularx}{\linewidth}{@{}Xcccc@{}}
\toprule
\textbf{Paper} & \textbf{Model} & \textbf{DP Type} & \textbf{Contexts} & \textbf{Regret} \\
\midrule
\citet{shariff2018differentiallyprivatecontextuallinear} & Linear & Joint & Adv. & $\widetilde{O}(\sqrt{T/\eps})$ \\
\citet{pmlr-v162-chowdhury22a} & Linear & Shuffle & Adv. & $\widetilde{O}(T^{3/5}/\sqrt{\eps})$ \\
\citet{han2021generalizedlinearbanditslocal}$^\dagger$ & GLM & Local & Stoch. & $\widetilde{O}(\sqrt{T}/\eps)^\dagger$ \\
\midrule
\textbf{Ours} (Alg.~\ref{alg:shuffled_glm}) & \textbf{GLM} & \textbf{Shuffle} & \textbf{Stoch.} & $\widetilde{O}(\sqrt{T/\eps})$ \\
\textbf{Ours} (Alg.~\ref{alg:jdp_glm_adversarial}) & \textbf{GLM} & \textbf{Joint} & \textbf{Adv.} & $\widetilde{O}(\sqrt{T/\eps})$ \\
\bottomrule
\end{tabularx}
\\[2pt]
\caption{Comparison with prior work on private contextual bandits. All prior shuffle/joint DP results rely on closed-form estimators available only for linear models.}
{\footnotesize $^\dagger$Requires spectral assumptions on the context distribution.}
\label{tab:comparison}
\end{table}

\section{Notation and Preliminaries}{\label{sec:notations_and_prelims}}
A policy $\pi$ is a function that maps any given arm set $\gX$ to a probability distribution over the same set, i.e., 
$\pi(\gX) \in \Delta(\gX)$, where $\Delta(\gX)$ is the probability simplex supported on $\gX$. We denote matrices in bold upper case (e.g., $\mathbf{M}$). $\|x\|$ denotes the $\ell_2$ norm of vector $x$. We write $\|x\|_{\mathbf{M}}$ to denote $\sqrt{x^\top \mathbf{M} x}$ for a positive semi-definite matrix $\mathbf{M}$ and vector $x$. For any two real numbers $a$ and $b$, we denote by $a \wedge b$ the minimum of $a$ and $b$. Throughout, $\widetilde{O}(\cdot)$ denotes big-$O$ notation but suppresses log factors in all relevant parameters. For $m, n \in \sN$ with $m < n$, we denote the set $\{1, \ldots, n\}$ by $[n]$ and $\{m, \ldots, n\}$ by $[m, n]$.

\begin{definition}{\label{def:glm}}
    A generalized linear model (GLM) with parameter vector $\theta^{\ast} \in \R^d$ is a real valued random variable $r$ that belongs to the exponential family of distributions with density function
    $$P(r \mid x) = \exp(r \cdot \langle x,\theta^{\ast}\rangle - b\left(\langle x,\theta^{\ast}\rangle\right) + c(r)). $$
    Function $b(\cdot)$ (called the log-partition function) is assumed to be twice differentiable and $\dot{b}$ is assumed to be monotone. Further, we assume $r \in [0,R]$ and $\|\theta^{\ast}\| \leq S$. \footnote{We assume that $R$ and $S$ are lower bounded by a constant.}
\end{definition}

Standard properties of GLMs include $\E[r] = \dot{b}(\langle x, \theta^\ast \rangle)$ and $\Var(r) = \ddot{b}(\langle x, \theta^\ast \rangle)$. We define the link function as $\mu(z) := \dot{b}(z)$. This formulation is standard in the contextual bandit literature \citep{NIPS2010_c2626d85,sawarni2024generalizedlinearbanditslimited}.

Next we describe the two contextual bandit problems with GLM rewards that we address in this paper. Let $T \in \sN$ be the total number of rounds. At each round $t \in [T]$, we receive an arm set $\gX_t \subset \R^d$ with $K = |\gX_t|$ arms, and must select an arm $x_t \in \gX_t$. Following this selection, we receive a reward $r_t$ sampled from the GLM distribution $P(r|x)$ with unknown parameter $\theta^\ast$.

\paragraph{Problem 1 ($\mone$).} In this setting, we assume that at each round $t$, the context set $\gX_t \subseteq \R^d$ is drawn from an unknown distribution $\gD$. The algorithm decides the lengths of each batch a priori, where the policy remains fixed within each batch and updates only at the start of each batch. We study $(\eps,\delta)$-shuffle differential privacy in this setup. The goal is to minimize the expected cumulative regret:
\begin{equation}{\label{eq:cumulative_regret_setupone}}
    R_T = \E\left[\sum_{t=1}^T \max_{x \in \gX_t} \mu(\langle x, \theta^\ast \rangle) - \sum_{t=1}^T \mu(\langle x_t, \theta^\ast \rangle)\right],
\end{equation}
where the expectation is over the randomness of the algorithm, the distribution of rewards $r_t$, and the distribution of arm sets $\gD$.

\paragraph{Problem 2 ($\mtwo$).} In this setting, we make no assumptions on the context sets $\gX_t$---they may be chosen adversarially. The goal is to minimize the cumulative regret while satisfying $(\eps,\delta)$-joint differential privacy:
\begin{equation}{\label{eq:cumulative_regret_setuptwo}}
    R_T = \sum_{t=1}^T \max_{x \in \gX_t} \mu(\langle x, \theta^\ast \rangle) - \sum_{t=1}^T \mu(\langle x_t, \theta^\ast \rangle).
\end{equation}

For both problems, the Maximum Likelihood Estimator (MLE) of $\theta^\ast$ is obtained by minimizing the cumulative log-loss. For an arm $x$, reward $r$, and parameter $\theta \in \mathbb{R}^d$, the log-loss is
$\ell(\theta, x, r) := -r \cdot \langle x, \theta \rangle + \int_0^{\langle x, \theta \rangle} \mu(z)\, dz$.
After $t$ rounds, we define the regularized MLE objective as
\[
\gL_t(\theta;\lambda) := \sum_{s=1}^t \ell(\theta, x_s, r_s) + \frac{\lambda}{2}\|\theta\|_2^2,
\]
where $\lambda > 0$ is the regularization parameter. This loss is convex and $(R+RS)$-Lipschitz over $\{\theta: \|\theta\| \le S\}$.

\subsection{Instance-dependent non-linear parameters}{\label{sec:instance_dependent_non_linear}}

Following \citet{faury2020improvedoptimisticalgorithmslogistic,faury2022jointlyefficientoptimalalgorithms,sawarni2024generalizedlinearbanditslimited}, we define instance-dependent parameters that capture the non-linearity of the link function. The parameter $\kappa$ can be exponentially large in $S$---for example, $\kappa = \Theta(e^S)$ in logistic regression. For any arm set $\gX$, let $x^{\ast} = \argmax_{x \in \gX} \mu(\langle x, \theta^{\ast}\rangle)$ denote the optimal arm. In setup $\mone$, define:
\begin{align}{\label{eq:kapa_kappa_hat_mone}}
     \kappa := \max_{\gX \in \supp(\gD)} \max_{x \in \gX} \frac{1}{\dot{\mu}(\langle x, \theta^{\ast}\rangle)}, \quad  & \frac{1}{\hat \kappa} = \E_{\gX \sim \gD} \left[\dot{\mu} ( \langle x^{\ast},\theta^{\ast} \rangle)\right],  \nonumber\\ 
     & \hspace{-6em} 1/\kappa^{\ast} := \max_{\gX \in \supp(\gD)} \max_{x \in \gX} {\dot{\mu}(\langle x, \theta^{\ast}\rangle)}.
\end{align}
In setup $\mtwo$, we redefine these parameters as:
\begin{equation}{\label{eq:kapa_mtwo}}
    \kappa := \max_{x \in \cup_{t \in [T]} \gX_t} \frac{1}{\dot{\mu}(\langle x, \theta^{\ast}\rangle)}, \quad 1/\kappa^{\ast} = \max_{x \in \cup_{t \in [T]} \gX_t} {\dot{\mu}(\langle x, \theta^{\ast}\rangle)}.
\end{equation}
Note that $\hat \kappa \geq \kappa^{\ast}$. Unlike $\kappa$, which can grow exponentially in $S$, the parameter $\kappa^{\ast}$ is benign: in general $\kappa^{\ast} \geq 1/R^2$ (since $\dot\mu(\langle x, \theta^\ast\rangle) = \Var(r \mid x) \leq R^2$ for rewards bounded by $R$), independent of $d$ and $S$. For specific link functions one often has a sharper bound: e.g., $\kappa^{\ast} \geq 4$ for logistic rewards (since $\sup_{z} \dot\mu(z) = 1/4$). Similar to the non-private bounds in \citet{sawarni2024generalizedlinearbanditslimited}, our regret bounds have no $\kappa$ dependence in the dominant $\widetilde{O}(\sqrt{T})$ term. In practice, it suffices to have access to an upper bound on $\kappa$ and lower bounds on $\hat{\kappa}$ and $\kappa^{\ast}$ as the true values may not be known apriori. A consolidated table of symbols used throughout the paper is given at the start of the appendix (\Cref{sec:notation_summary}).

\begin{assumption}[Boundedness]\label{ass:boundedness}
    The reward $r_t$ at each time step $t$ is bounded by a known constant $R$. Every context vector has $\ell_2$ norm bounded by $1$, i.e., $\|x\| \leq 1$ for all $x \in \gX_t$ and $t \in [T]$. The unknown parameter satisfies $\|\theta^{\ast}\| \leq S$ for a known constant $S$. We further assume $R,S \ge 1$.\footnote{Any absolute constant lower bound would suffice.}
\end{assumption}

\subsection{Optimal design}

In the stochastic-context regime, batches are determined ahead of time. Previous work has leveraged $G$-optimal design to construct policies for both linear \citep{ruan2021linearbanditslimitedadaptivity} and generalized linear \citep{sawarni2024generalizedlinearbanditslimited} reward settings. We restate the fundamental lemma, which follows from the General Equivalence Theorem \citep{Kiefer_Wolfowitz_1960}.

\begin{lemma}[\citealp{ruan2021linearbanditslimitedadaptivity}] For any subset $\gX \subseteq \R^d$, there exists a distribution $\gK_\gX$ supported on $\gX$ such that for any $\eps > 0$,
\begin{align}
\max_{x \in \gX} x^\top \left(\eps \mI + \E_{y \sim \gK_\gX} y y^\top \right)^{-1} x \leq d. \label{eq:thm-KW}
\end{align}
Furthermore, if $\gX$ is finite, one can find a distribution achieving the relaxed bound $2d$ in time $\mathrm{poly}(|\gX|)$.
\end{lemma}

\begin{remark}
While \citet{ruan2021linearbanditslimitedadaptivity,sawarni2024generalizedlinearbanditslimited} use distributional optimal design, we use $G$-optimal design because we could not compute the former under differential privacy constraints. This choice incurs an additional $\sqrt{d}$ factor in the leading regret term.
\end{remark}

\section{Shuffle-DP GLM Bandits with Stochastic Contexts}\label{sec:glm_alg}

We present our algorithm for the stochastic-context setting ($\mone$). The algorithm partitions the horizon into batches, uses $G$-optimal design for arm selection, and privately aggregates statistics via the shuffle model.

Our design adapts the batched framework of \citet{sawarni2024generalizedlinearbanditslimited}, with three key modifications for privacy:

\begin{itemize}[leftmargin=*,itemsep=2pt]
    \item \textbf{Private covariance estimation.} Design matrices $\tV$ and $\tH_k$ are computed via shuffled vector summation \citep{Cheu2021ShufflePS}, which aggregates noisy outer products across users.
    
    \item \textbf{Private parameter estimation.} The parameter estimate $\hat\theta_k$ at each batch is computed using a shuffle-private convex optimizer on the regularized negative log-likelihood.
    
    %\item \textbf{Privacy-aware batch sizes.} Batch sizes are scaled to absorb both statistical error and privacy-induced noise, ensuring confidence bounds remain valid.
\end{itemize}

\paragraph{Batch structure.}
We partition the horizon into $M \leq \log \log T$ batches 
$\mathcal{B}_1,\ldots,\mathcal{B}_M$ of lengths 
$B_1,\ldots,B_M$. Following \citet{sawarni2024generalizedlinearbanditslimited}, we define
\begin{align}
    B_1 &= \left(\frac{\sqrt{\kappa}\,e^{3S} d^2 \gamma^2}{S}\,\alpha\right)^{2/3}, 
    \qquad B_2 = \alpha, 
    \label{eq:batch_B1_B2} \\[4pt]
    B_k &= \alpha \sqrt{B_{k-1}}, \quad k \in [3,M],
    \label{eq:batch_Bk}
\end{align}
where $\alpha = T^{\frac{1}{2(1-2^{-M+1})}}$ ensures the total horizon stays within $T$. The first batch $B_1$ serves as a warm-up phase. Unlike the non-private setting of \citet{sawarni2024generalizedlinearbanditslimited}, $\gamma$ must be tuned to account for shuffle-DP noise, resulting in larger batch sizes. The number of batches $M$ is bounded by $O(\log \log T)$ (see \cref{lemma:batch_number_bound})

\paragraph{Privacy-aware parameters.}
Suppressing $\text{poly}(R,S,\log d, \log (1/\delta), \log \log T)$ factors, the parameters scale as
\begin{equation}\label{eq:sigma_nu_scaling_mone}
\sigma = \widetilde{O}\!\left(\tfrac{1}{\kappa^\ast \varepsilon}\right),\quad
\nu = \widetilde{O}\!\left(\tfrac{\sqrt d (\log T)^{3/2}}{\varepsilon}\right),
\end{equation}
\begin{equation}\label{eq:lambda_gamma_scaling_mone}
\lambda = \widetilde{O}\!\left(d \log T + \tfrac{\sqrt {\max(d,\log T)}}{\varepsilon}\right),\quad
\gamma = \tilde{O}(\sqrt \lambda + \sqrt \nu).
\end{equation}
Here $\sigma$ is the subgaussian parameter controlling the noise introduced by the shuffled vector summation protocol, and $\nu$ denotes the optimization error of the shuffle-private convex optimizer (see \Cref{sec:assumptions_shuffled_glm}). Exact closed-form definitions are deferred to \Cref{sec:params_shuffled_glm}.
% \ gamma &= 30 RS \sqrt{dR\log T 
%     + 12\sigma \sqrt{\log (2dMT/\delta)}} 
%     + \sqrt{4RS \nu}.

% \paragraph{Warm-up and estimation.}
% In the warm-up batch $\mathcal{B}_1$, the algorithm samples arms via $G$-optimal design to initialize its estimate. We form
% \[
% \tV \approx \sum_{t \in \mathcal{B}_1} x_t x_t^\top + \lambda \mathbf{I},
% \]
% where $\tV$ is a noisy estimate of the empirical covariance computed via the shuffled vector summation protocol (\Cref{sec:vector_summation_protocol}). At the end of each batch $k$, the parameter estimate $\hat\theta_k$ is updated using a shuffled convex optimizer on the loss $\ell(\cdot)$.

\paragraph{Warm-up and estimation.}
In the warm-up batch $\mathcal{B}_1$, the algorithm samples arms via $G$-optimal design to build an initial estimate. Each user $t$ locally randomizes their outer product $x_t x_t^\top$ using the shuffle protocol $R_2$, the shuffler permutes all messages, and the analyzer aggregates them to form:
\[
\tV \approx \sum_{t \in \mathcal{B}_1} x_t x_t^\top + \lambda \mathbf{I}.
\]
This construction ensures $(\varepsilon/2, \delta/2)$-shuffle DP for the covariance estimate. The parameter estimate $\hat\theta_1$ is then computed via the shuffle-private convex optimizer on the regularized negative log-likelihood. A detailed construction of the vector summation protocols is given in \cref{sec:vector_summation_protocol} and the shuffle convex optimizer is given in \cref{sec:shuffle_convex_optimizer}

\paragraph{Arm scaling and elimination.}
For batches $k \geq 2$, arms are scaled before selection:
\begin{align}
    \tilde{\mathcal{X}}_t 
    &:= \left\{
        \sqrt{\tfrac{\dot{\mu}(\langle x, \hat\theta_1 \rangle)}{\beta(x)}}\, x 
        : x \in \mathcal{X}_t
    \right\},
    \label{eq:arm_scaling} \\[4pt]
    \beta(x) 
    &:= \exp\!\Bigl(R \min\!\bigl(2S, \gamma \sqrt{\kappa}\,\|x\|_{\tV^{-1}}\bigr)\Bigr).
    \label{eq:beta_def}
\end{align}
The scaling $\beta(x)$ serves two purposes: (i) it down-weights arms with high uncertainty (large $\|x\|_{\tV^{-1}}$), preventing them from dominating the design matrix, and (ii) the exponential form ensures that $\dot{\mu}(\langle x,\hat\theta_1\rangle)/\beta(x) \leq 1/\kappa^\ast$, which is crucial for controlling the sensitivity of the shuffled summation protocol.
Arms with $\text{UCB}_j(x) < \max_{y} \text{LCB}_j(y)$ are eliminated. A $G$-optimal design is applied to the scaled, non-eliminated arms $\tilde{\mathcal{X}}_t$. The weighted design matrix is then computed as
\[
\tH_k \approx 
\sum_{t \in \mathcal{B}_k}
\frac{\dot{\mu}(\langle x_t, \hat\theta_1 \rangle)}{\beta(x_t)} \, x_t x_t^\top + \lambda \mathbf{I},
\]
using the shuffled vector summation protocol. A detailed description is given in \cref{sec:vector_summation_protocol}

\paragraph{Confidence bounds.}
Upper and lower confidence bounds are defined as
\begin{align}
    \text{UCB}_k(x) &:= \langle x, \hat \theta_k \rangle + \gamma \sqrt{\kappa}\,\|x\|_{\mathbf{V}_k^{-1}}, \label{eq:ucb_def} \\
    \text{LCB}_k(x) &:= \langle x, \hat \theta_k \rangle - \gamma \sqrt{\kappa}\,\|x\|_{\mathbf{V}_k^{-1}}, \label{eq:lcb_def}
\end{align}
where $\mathbf{V}_1 = \tV$ and $\mathbf{V}_k = \tH_k$ for $k \geq 2$. These bounds follow \citet{sawarni2024generalizedlinearbanditslimited} with $\gamma$ adjusted for privacy noise and remain valid under our noisy shuffled estimates.

The parameters can be defined (see \cref{line:lambda_l_u_defn} for exact definitions) as $\lambda_{\max/\min} := \lambda \pm C\log(d)\sqrt{\max (d,\log T)}/\varepsilon$, with $+$ for max and $-$ for min.\footnotemark

\footnotetext{
We choose $\lambdal$ and $\lambdau$ such that the perturbation matrix satisfies
$\lambdal \mI \preceq \tV_k - \sum_{s \in \gB_k}
\frac{\dot{\mu}(\langle x_s, \hat\theta_1\rangle)}{\beta(x_s)}
x_s x_s^\top \preceq \lambdau \mI$ for $k \ge 2$, and
$\lambdal \mI \preceq \tV_1 - \sum_{s \in \gB_1} x_s x_s^\top \preceq \lambdau \mI$
for $k = 1$.
These bounds follow from matrix concentration inequalities \citep{vershynin_high-dimensional_2018}; see \cref{obs:bounds_lambda_min_lambda_max}.
}

% \paragraph{Loss function.}
% For GLM rewards, the negative log-likelihood is
% \[
% \ell(\theta, r, x) := -r \langle x, \theta \rangle 
% + b(\langle x,\theta\rangle) - c(r),
% \]
% which is convex in $\theta$ and $(R+RS)$-Lipschitz over $\{\theta: \|\theta\| \le S\}$. The regularized empirical loss is $\gL_{\gD}(\theta) = \sum_{(r,x)\in\gD} \ell(\theta,r,x) + \tfrac{\lambdau}{2}\|\theta\|^2$.

\paragraph{Protocol instantiation.}
We use the shuffle convex optimizer $\mathcal{P}_{\text{GD}}$ and vector summation protocol $(R_2,S_2,A_2)$ of \citet{Cheu2021ShufflePS,Cheu_2019}. The optimizer runs for $\tfrac{\varepsilon^2 B_k^2}{d \log^3(B_k d/\delta)}$ iterations with step size $\eta=2S/(RS \sqrt{T})$. The summation protocol privately aggregates outer products $x_t x_t^\top$ via shuffled bit summation. The resulting algorithm is $(\varepsilon,\delta)$-shuffle DP; since shuffle DP implies joint DP via the Billboard lemma \citep{mechanism_design_large_games}, the same guarantee holds in the joint-DP sense as well. Full details are in \Cref{sec:proof_theorem1_regret_ass}.

The total privacy budget splits as $(\varepsilon/2, \delta/2)$ for covariance estimation and $(\varepsilon/2, \delta/2)$ for parameter estimation; composition across batches is handled via the Billboard lemma since the policy within each batch is fixed before observing data.

Here $R_2$ denotes the local randomizer and $A_2$ the analyzer from the shuffled vector summation protocol (\Cref{sec:vector_summation_protocol}). $\gP_{\text{GD}}$ denotes the shuffle private convex optimizer and is presented in \cref{sec:shuffle_convex_optimizer}.

%--- Algorithm first, then theorem ---
\begin{algorithm}[!tb] 
    \caption{\textsc{ShuffleGLM}: Shuffle-DP GLM Bandits}
    \label{alg:shuffled_glm} 
    \textbf{Input:} Batches $M$, horizon $T$, privacy $(\varepsilon, \delta)$, parameters $\lambda, \gamma$ as in \eqref{eq:lambda_gamma_scaling_mone} (closed-form values in \Cref{sec:params_shuffled_glm}).
    \begin{algorithmic}[1]    
    \State Initialize batch lengths $B_1, \ldots, B_M$ via \eqref{eq:batch_B1_B2}--\eqref{eq:batch_Bk}.

    \Statex \textcolor{gray}{\textit{// Warm-up batch}}
    \For{$t \in \gB_1$} \label{line:warmup_start}
        \State Observe arm set $\gX_t$, sample $x_t \sim \pi^G(\gX_t)$, observe reward $r_t$.
        \State \textsc{LocalRandomizer:} send
        \Statex \hspace{2em} $M_{t} \gets R_2(x_t x_t^\top,\, \Delta{=}1,\, \varepsilon/2,\, \delta/2)$.
    \EndFor \label{line:warmup_end}
    \State \textsc{Shuffler:} compute
    \Statex \hspace{1em} $\tV \gets A_2(\text{Shuffle}(\{M_t\}_{t \in \gB_1}),\, \varepsilon/2,\, \delta/2) + \lambda \mathbf{I}$.
    \State \textsc{PrivateOptimizer:}
    \Statex \hspace{1em} $\hat\theta_1 \gets \gP_{\text{GD}}\bigl(\{(r_t, x_t)\}_{t \in \gB_1};\, \varepsilon/2,\, \delta/2,\, \ell(\cdot) + \tfrac{\lambdau}{2B_1}\|\theta\|^2\bigr)$.

    \Statex \textcolor{gray}{\textit{// Subsequent batches}}
    \For{$k = 2$ to $M$}
        \For{$t \in \gB_k$}
            \State Observe arm set $\gX_t$.
            \For{$j = 1$ to $k-1$} \Comment{Eliminate suboptimal arms}
                \State $\gX_t \gets \gX_t \setminus \{ x : \text{UCB}_j(x) < \max_{y \in \gX_t} \text{LCB}_j(y) \}$. \labelline{line:elimination_glm_stochastic}
            \EndFor
            \State Scale arms: $\tilde{\gX}_t \gets \{\sqrt{\dot{\mu}(\langle x, \hat\theta_1 \rangle)/\beta(x)} \cdot x : x \in \gX_t\}$. \label{line:modified_arms}
            \State Sample $x_t \sim \pi^G(\tilde{\gX}_t)$, observe reward $r_t$. \label{line:g_opt_sampling_glm}
            \State \textsc{LocalRandomizer:} send\footnotemark
            \Statex \hspace{2em} $M_{t} \gets R_2\!\left(\tfrac{\dot{\mu}(\langle x_t, \hat\theta_1 \rangle)}{\beta(x_t)} x_t x_t^\top,\, \Delta{=}\tfrac{1}{\kappa^\ast},\, \varepsilon/2,\, \delta/2\right)$.
        \EndFor
        \State \textsc{Shuffler:} compute
        \Statex \hspace{2em} $\tH_k \gets A_2(\text{Shuffle}(\{M_t\}_{t \in \gB_k}),\, \varepsilon/2,\, \delta/2) + \lambda \mathbf{I}$.
        \State \textsc{PrivateOptimizer:} \label{line:batch_end}
        \Statex \hspace{2em} $\hat\theta_k \gets \gP_{\text{GD}}\bigl(\{(r_t, x_t)\}_{t \in \gB_k};\, \varepsilon/2,\, \delta/2,\, \ell(\cdot) + \tfrac{\lambdau}{2B_k}\|\theta\|^2\bigr)$.
    \EndFor
    \end{algorithmic}
\end{algorithm}

\footnotetext{The bound $\Delta = 1/\kappa^\ast$ holds since $\tfrac{\dot{\mu}(\langle x_t, \hat\theta_1 \rangle)}{\beta(x_t)} \leq \dot{\mu}(\langle x_t, \theta^\ast \rangle) \leq 1/\kappa^\ast$ with high probability (\Cref{lemma:H_bound_H_star}).}

\begin{theorem}[Regret bound for \Cref{alg:shuffled_glm}; full version in \Cref{thm:shuffled_regret_glm}]
\label{inf_thm:privacy_utility_m1}
For $\varepsilon < 5$, $\delta < 1/2$ and $T \geq d$, \Cref{alg:shuffled_glm} satisfies $(\varepsilon,\delta)$-shuffle DP and achieves expected regret
\[
\widetilde{O}\!\left(\left(\frac{d^{3/2}\sqrt{T}}{\sqrt{\kappa^\ast}} + \frac{d^{5/4}\sqrt{T}}{\sqrt \kappa^\ast \sqrt{\varepsilon}} + \frac{d^2 T^{1/3} e^{3S} \kappa^{1/3}}{\varepsilon^{2/3}}\right)\sqrt {\log T}\right) .
\]
\end{theorem}

The full all-$T$ statement along with the proof, is given in \Cref{thm:shuffled_regret_glm} of \Cref{sec:shuffled_privacy_instantiation_utility_regret_tradeoff}.

Of the three terms in the bound, two scale as $\sqrt T$ and one as $T^{1/3}$. The two $\sqrt T$ terms carry no $\kappa$ and no exponential-in-$S$ factor, matching the functional form of the non-private rate of \citet{sawarni2024generalizedlinearbanditslimited} up to a $\sqrt d$ overhead from privatization. The $e^{3S}$ factor is confined to the sub-leading $T^{1/3}$ term. The gap to the non-private rate is at most $\widetilde O(\sqrt{d/\varepsilon})$ considering only the dominant term.
\begin{remark}[Removing $\kappa$ from the leading term]
The estimator $\hat\theta_k$ itself is computed by an ordinary (shuffle-private) convex optimizer on the standard GLM negative log-likelihood---no $\kappa$-aware reweighting of the loss. Instead, $\kappa$ is absorbed into the \emph{confidence-set geometry}: from batch $k\ge 2$ onward we build confidence ellipsoids against the Hessian-weighted design matrix $\tH_k$ of \eqref{eq:arm_scaling}, rather than the naive design matrix $\sum_s x_s x_s^\top$. In \cref{lemma:H_bound_H_star}, we later show that the design matrix $\tH_k$ is bounded by the hessian of the regularized MLE loss $\gL(\theta^\ast)$ which enables us to remove $\kappa$ dependencies in the regret after the first batch. %The $\dot\mu(\langle x_s,\hat\theta_1\rangle)$-weighting in $\tH_k$---implemented via the arm scaling in \eqref{eq:arm_scaling}---cancels the $\kappa = 1/\min\dot\mu$ factor that would otherwise multiply the leading $\sqrt{T}$ term. The warm-up batch $\gB_1$ incurs $\kappa$-dependent regret, but it only serves to produce $\hat\theta_1$ accurate enough that this reweighting is valid in subsequent batches.
\end{remark}

The full proof is in Appendix \ref{sec:shuffled_privacy_instantiation_utility_regret_tradeoff}.

% %--- Moved technical details to end ---
% \paragraph{Technical requirements.}
% The analysis requires: (i) $\tV$ and $\{\tH_k\}_{k \geq 2}$ deviate from their true design matrices by symmetric, zero-mean sub-Gaussian noise with variance proxy $\tilde\sigma^2$; and (ii) each $\hat\theta_k$ achieves empirical loss within $\nu$ of the constrained minimum. Both conditions are satisfied by our shuffled protocol instantiations (\Cref{sec:vector_summation_protocol,sec:shuffle_convex_optimizer}).

\section{Joint-DP GLM Bandits with Adversarial Contexts}{\label{sec:glm_jdp}}

We present \Cref{alg:jdp_glm_adversarial} for setting \textbf{M2}: GLM bandits with adversarial contexts under $(\varepsilon,\delta)$-joint differential privacy. Our design adapts the rarely-switching framework of \citet{sawarni2024generalizedlinearbanditslimited}, but faces three key technical challenges under privacy, each requiring a new ingredient:

\begin{itemize}[leftmargin=*,itemsep=2pt]
    \item \textbf{Private design matrix release.} The design matrices $\tV$ (for Criterion I rounds, line~\ref{line:policy_I_switch}) and $\tH_t$ (for Criterion II rounds, line~\ref{line:policy_II_switch}) must be released at every round. We use the binary-tree mechanism \citep{continualstatsrelease,continualobservation} to achieve this under continual observation.
    
    \item \textbf{Non-monotone determinant switching.} The noise added by the binary-tree mechanism in lines~\ref{line:noise_addition_R_t_step_I} and~\ref{line:noise_addition_R_t_step_II} makes $\{\tH_t\}$ non-monotone, breaking the standard rarely-switching argument. We modify the switching criterion at line~\ref{line:policy_II_switch} to compare the current noisy design matrix against the design matrix at the last policy update in PSD order.

    % We modify the switching criterion at line~\ref{line:policy_II_switch} to compare against the \emph{running maximum} determinant.
    
    \item \textbf{Privacy of switching times.} The index set $\gT_o$ of exploration rounds depends on the data. We prove that $\gT_o \setminus \{t\}$ is $(\varepsilon,\delta)$-indistinguishable across datasets differing at time $t$, using a log-likelihood ratio decomposition across the binary tree (\Cref{lem:switching_privacy}).
\end{itemize}

The algorithm has two switching criteria. \emph{Criterion I} triggers when some arm has high uncertainty under $\tV$: we select the most uncertain arm and update $\tV$. Otherwise, we use a scaled design matrix $\tH_t$ incorporating the link derivative to remove $\kappa$ from the dominant regret term. \emph{Criterion II} triggers a policy update when $\tH_{t} \npreceq 2\tH_{\tau}$ (directional growth criterion).

\begin{algorithm}[!htb]
    \small
    \caption{Joint-DP Generalized Linear Bandits with Adversarial Contexts}
    \label{alg:jdp_glm_adversarial} 
    \textbf{Input:} Privacy $(\varepsilon, \delta)$, horizon $T$, error probability $\zeta$, parameters $\lambda, \gamma, \beta$ as in \eqref{eq:lambda_scaling_jdp}--\eqref{eq:gamma_scaling_jdp} (closed-form values in \Cref{sec:params_jdp_glm_adversarial}).\\
    \textbf{Initialize:} $\tV, \tH_1 \gets \lambda \mI$; $\gT_o \gets \emptyset$; $\tau \gets 0$; tree mechanisms $\mathcal{B}_V, \mathcal{B}_H$
    
    \begin{algorithmic}[1]
    \For{$t \gets 1$ \textbf{to} $T$}
        \State Observe $\gX_t$; query $\tV, \tH_t$ from $\gB_V, \gB_H$
        
        \State \Comment{\textbf{Criterion I:} Exploration (See \Cref{lem:switching_privacy,lemma:utility_privacy_glm_adv_2})}
        \If{$\max_{x \in \gX_t} \|x\|^2_{\tV^{-1}} \geq \frac{1}{\gamma^2 \kappa R^2}$} \labelline{line:policy_I_switch}
            \State Play $x_t = \argmax_{x \in \gX_t} \|x\|_{\tV^{-1}}$; observe $r_t$
            \State $\gT_o \gets \gT_o \cup \{t\}$
            
            \State \textsc{PrivateUpdateExplore}$(\tV, \tH_t, x_t)$ \labelline{line:noise_addition_R_t_step_I}
            \State Compute $\hat \theta_o$ via DP-SGD optimizer on $\gT_o$ data \labelline{line:epsilon_0_delta_0_optimizer}
        \Else
            \State \Comment{\textbf{Criterion II:} Policy Update (Directional growth criterion)}
            \If{$\tH_t \npreceq 2 \tH_{\tau}$} \labelline{line:policy_II_switch}
                \State $\tau \gets t$; compute $\hat \theta_{\tau}$ via DP-SGD optimizer on $[t{-}1] \setminus \gT_o$ data
            \EndIf
            
            \State $\gX_t \gets \{ x \in \gX_t \mid \text{UCB}_o(x) \ge \max_{z} \text{LCB}_o(z) \}$ \Comment{Arm Elimination (Standard Gap-Based \eqref{eq:ucb_o}, \eqref{eq:lcb_o})}\labelline{line:elimination_glm_adversarial}
            \State Play $x_t = \argmax_{x \in \gX_t} \text{UCB}(x, \tH_\tau, \hat \theta_\tau)$, observe $r_t$ \Comment{\eqref{eq:ucb_general}} \labelline{line:arm_selected_regular}
            
            \State \textsc{PrivateUpdateExploit}$(\tV, \tH_t, x_t, \hat \theta_o)$ \labelline{line:noise_addition_R_t_step_II}
        \EndIf
    \EndFor
    
    \Statex \hrulefill \Comment{\textbf{Privacy Mechanisms}}
    
    \State \textbf{Function} \textsc{PrivateUpdateExplore}$(\tV, \tH, x)$
    \State \quad Update $\tV \gets \tV + x x^{\top} + \gR^V$ via $\gB_V$ %\Comment{Adds noise to $\log T$ nodes}
    \State \quad Update $\tH \gets \tH + \gR^H$ via $\gB_H$
    
    \State \textbf{Function} \textsc{PrivateUpdateExploit}$(\tV, \tH, x, \hat \theta)$
    \State \quad Update $\tH \gets \tH + \frac{\dot{\mu}(\langle x, \hat \theta\rangle)}{e} x x^{\top} + \gR^H$ via $\gB_H$ \Comment{Scaled update}
    \State \quad Update $\tV \gets \tV + \gR^V$ via $\gB_V$
    \end{algorithmic}
\end{algorithm}

\paragraph{Optimizer and loss details.}
Define $\ccut = 8dR^2 \kappa \gamma^2 \log T$ and $\ccuttwo = 4\log_2\!\left(1 + \frac{TR^3}{d}\right)$, where $\lambdau = \frac{4\sqrt d + 2}{4\sqrt d + 1}\,\lambda$ and $\lambdal = \frac{4\sqrt d}{4\sqrt d + 1}\,\lambda$. \Cref{lemma:utility_privacy_glm_adv_3} shows that, with high probability, $\ccut$ and $\ccuttwo$ upper-bound the number of exploration rounds (Criterion~I) and policy switches (Criterion~II), respectively.
The DP-SGD optimizer---one instantiation of which, $\mathcal{P}_{\text{GD}}$, is given in \Cref{sec:shuffle_convex_optimizer}---minimizes a regularized loss $\gL(\theta;\lambda_{\text{max}})$ over the constraint sets $\Theta_I := \{\theta: \|\theta\|_2 \leq S\}$ under Criterion~I and $\Theta_{II} := \{\theta: \|\theta - \hat \theta_o\|_{\tV} \leq \gamma \sqrt{\kappa}\}$ under Criterion~II. Its zero-Concentrated DP \citep{cryptoeprint:concentratedDP} privacy budget is set to
$\rho := \frac{\varepsilon^2}{144c \log(8/\delta)}$,
with $c=\ccut$ under Criterion~I and $c=\ccuttwo$ under Criterion~II. By additivity of zCDP over the at most $c$ optimizer calls, followed by the standard zCDP-to-DP conversion, this yields a cumulative privacy budget of $(\varepsilon/3,\delta/3)$.

%By advanced composition~\citep[Theorem~3.20]{dworkdpbook}, this yields a cumulative privacy budget of $(\varepsilon/3, \delta/3)$ across all optimizer calls under each of Criterion~I and Criterion~II using the standard zero-CDP to DP conversion.

\paragraph{Binary-tree mechanism.} To release $\tV$ and $\tH_t$ at every round, we use the binary-tree mechanism \citep{continualstatsrelease,continualobservation}. The mechanism maintains a complete binary tree over $[T]$ where each node stores a partial sum plus independent Gaussian noise. Each data point at time $t$ is added to the $\log T$ nodes on the path from leaf $t$ to the root. To query the prefix sum up to time $t$, we sum at most $\log T$ disjoint nodes. Since changing one data point affects only $\log T$ nodes (each with independent noise), Gaussian composition yields joint $(\varepsilon/3, \delta/3)$-DP for all released prefix sums. The notation $\gR^V, \gR^H$ in the algorithm represents the implicit noise in these updates—while individual $\gR$ do not have a simple closed-form distribution, the prefix sums are well-characterized as sums of independent Gaussians. A detailed description is presented in \Cref{sec:privacy_analysis_regret_glm_adversarial}.

\subsection{Confidence Bounds and Parameters}

The confidence bounds used in arm elimination and selection are:
\begin{align}
    \text{UCB}_o(x) &= \langle x, \hat \theta_o \rangle + \gamma \sqrt{\kappa}\,\|x\|_{\tV^{-1}}, \label{eq:ucb_o} \\
    \text{LCB}_o(x) &= \langle x, \hat \theta_o \rangle - \gamma \sqrt{\kappa}\,\|x\|_{\tV^{-1}}, \label{eq:lcb_o} \\
    \text{UCB}(x, \tH, \theta) &= \langle x, \theta\rangle + \beta \,\|x\|_{\tH^{-1}}. \label{eq:ucb_general}
\end{align}

The algorithm uses a regularization parameter $\lambda$ and two confidence-width scalings $\beta$ (for Criterion II) and $\gamma$ (for Criterion I). They scale as
\begin{equation}\label{eq:lambda_scaling_jdp}
\lambda = \widetilde{O}\!\left(d \log T + \tfrac{\sqrt d}{\varepsilon} (\log T) \max\left(\sqrt{\log T}, \sqrt d\right)\right),\quad
\end{equation}

\begin{equation}\label{eq:beta_scaling_jdp}
    \beta = \widetilde{O}\!\left(\sqrt{d \log T} + \tfrac{d^{1/4} (\log T)^{1/2} (d + \log T)^{1/4}}{ \sqrt \varepsilon}\right),
\end{equation}

\begin{equation}\label{eq:gamma_scaling_jdp}
\gamma = \widetilde{O}\!\left(\sqrt {d \log T} + \tfrac{d^{1/4} (\log T)^{1/2} (d + \log T)^{1/4}}{\sqrt {\varepsilon\kappa^\ast}} + \tfrac{d \sqrt \kappa \log T}{\varepsilon}\right).
\end{equation}
Exact closed-form expressions are deferred to \Cref{sec:params_jdp_glm_adversarial}; they account for the binary-tree noise, the privacy cost of switching, and the DP-SGD optimizer error.

\subsection{Privacy of Switching Times}
\label{sec:switching_privacy}
A key challenge is that $\gT_o$ depends on the data: round $t$ is added to $\gT_o$ iff $\max_{x \in \gX_t} \|x\|^2_{\tV^{-1}} \geq 1/(\gamma^2 \kappa R^2)$, where $\tV$ depends on all prior rounds in $\gT_o$. Thus the identity of which rounds trigger Criterion I could leak information.

\paragraph{Why only Criterion~I needs this analysis.}
Criterion~I (line~\ref{line:policy_I_switch}) reads the \emph{raw} context set $\gX_t$ at time $t$, so its trigger depends on data that has not yet passed through any privacy mechanism --- hence the indicator vector itself must be shown to be DP. By contrast, Criterion~II (line~\ref{line:policy_II_switch}) is evaluated solely on the \emph{noisy} design matrices $\tH_t, \tH_\tau$, both of which are released by the binary-tree mechanism and are already $(\varepsilon/3,\delta/3)$-DP; Criterion~II switches are therefore differentially private by post-processing and require no separate analysis. We thus show the following.

% \paragraph{Key insight.} We show that for any two adjacent datasets $\gD, \gD'$ differing only at round $t$, the transcript $\gT_o \setminus \{t\}$ is $(\varepsilon, \delta)$-indistinguishable. The proof decomposes the log-likelihood ratio across the binary tree structure.

\begin{lemma}[Privacy of switching indicators; part of \Cref{lemma:utility_privacy_glm_adv_2}]
\label{lem:switching_privacy}
Assume $\delta \ge \zeta$. Let $\gD, \gD'$ differ only at index $t$. Then $\gT_o \setminus \{t\}$ is $(\varepsilon/3, \delta/3)$ indistinguishable w.r.t.\ the change from $\gD$ to $\gD'$. (The tighter bound $(\varepsilon/3, \delta/6 + \zeta/6)$ is shown in \Cref{lemma:utility_privacy_glm_adv_2}.)
\end{lemma}

The proof follows from the observation that whether round $s \neq t$ joins $\gT_o$ depends on $\gD$ vs.\ $\gD'$ only through the noisy matrix $\tV_s$, which differs in at most $\log T$ binary-tree nodes. Gaussian composition over these nodes yields the bound; see \Cref{sec:utility_privacy_switching_I_II}.

\subsection{Main Result}

\begin{theorem}[Joint-DP regret bound; full version in \Cref{thm:jdp_regret_glm}]
\label{thm:jdp_main}
Fix $\delta > \zeta$, $\delta \leq \min\left(1/e, \frac{d}{R^3}\right)$ and $d \leq T$. \Cref{alg:jdp_glm_adversarial} satisfies $(\varepsilon, \delta)$-joint differential privacy. Moreover, with probability at least $1 - \zeta$, its regret is bounded by \footnote{In this display, $\widetilde O(\cdot)$ hides polynomial factors in $R$, $S$, $\log(1/\zeta)$ and $\log T$, and lower-order additive terms that are polynomial in $\log T$. We do not hide multiplicative powers of $\log T$
appearing in the leading $\sqrt{T}$ term.}
\[
R_T \leq \widetilde{O}\!\left(\sqrt{T} \left(\frac{d \log T}{\sqrt{\kappa^\ast}} + \frac{d^{3/4}\,\log T\,(\log T + d)^{1/4}}{\sqrt{\varepsilon \kappa^{\ast} \min\left(\kappa^\ast,1\right)}}\right)\right).
\]
\end{theorem}

The bound holds for all $T$ (no regime assumption); the proof is in \Cref{sec:jdp_instantiation_utility_regret_tradeoff}.

\paragraph{Comparison with non-private bound.}
The non-private algorithm of \citet{sawarni2024generalizedlinearbanditslimited} achieves $\widetilde{O}(d\sqrt{T/\kappa^\ast}\,\log T)$. Our bound matches this leading-in-$T$ rate, with privacy contributing only an additive $\widetilde{O}\!\left(\frac{d^{3/4}\sqrt{T}\,(\log T)\,(d+\log T)^{1/4}}{\sqrt{\kappa^\ast \min(\kappa^\ast ,1)\,\varepsilon}}\right)$ correction. Neither term carries $\kappa$ or $e^S$ dependence; the benign parameter $\kappa^{\ast}$ enters only as a $1/\sqrt{\kappa^{\ast}}$ factor in the leading term and $1/\kappa^{\ast}$ in the privacy correction. In general $\kappa^{\ast} \geq 1/R^2$ (where $R$ is the reward bound), independent of $d$ and $S$; for specific link functions one often has a sharper bound (e.g., $\kappa^{\ast} \geq 4$ for logistic).

\begin{remark}[Removing $\kappa$ from the leading term]
\label{rem:kappa_removal_jdp}
The estimator $\hat\theta_o$ itself is computed by an ordinary DP-SGD convex optimizer on the standard GLM log-likelihood; $\kappa$ is absorbed into the confidence-set geometry via the Hessian-weighted design matrix $\tH_t$, which reweights each $x_t x_t^\top$ by $\dot\mu(\langle x_t, \hat\theta_o\rangle)/e \leq \dot\mu(\langle x_t, \theta^\ast\rangle)$. This implies that $\tH_t$ is bounded (up to constant factors) by the Hessian of the regularized MLE loss, which removes $\kappa$ from the regret in non-exploration rounds. The number of exploration rounds (Criterion I) is bounded by $|\gT_o| \leq 8dR^2 \kappa \gamma^2 \log T$ via \Cref{lemma:utility_privacy_glm_adv_2_policyswitch}; this $\kappa$-dependent cost is only logarithmic in $T$ and therefore contributes only to sub-leading regret terms.
\end{remark}

\begin{remark}[Gap in privacy-dependent term]
Our upper bound has a privacy-dependent term of $\widetilde{O}(d^{3/4}\sqrt{T}\,(\log T)\,(d+\log T)^{1/4}/\sqrt{\varepsilon})$, while the lower bound (\Cref{sec:private_lower_bound}) suggests $\Omega(d/\varepsilon)$. We note that this gap is not specific to GLMs: even for the simpler setting of joint-DP \emph{linear} contextual bandits, the same gap persisted and was posed as an open problem in \citet{pmlr-v247-azize24a}. Very recently, \citet{chen2025covariancematrixstatisticalcomplexity} closed this gap for linear bandits by developing minimax theory for private linear regression under general covariates. However, extending these techniques to GLMs is nontrivial, since their analysis relies crucially on the closed-form structure of the least-squares estimator, which is unavailable in the GLM setting.
\end{remark}

\paragraph{Privacy composition.} The privacy guarantee combines three components: (i) binary-tree mechanism for $\tV$ and $\tH_t$ ($\varepsilon/3, \delta/3$); (ii) switching indicator privacy from \Cref{lem:switching_privacy} ($\varepsilon/3, \delta/3$); (iii) private convex optimizer ($\varepsilon/3, \delta/3$). By composition, the total privacy budget is $(\varepsilon, \delta)$. Full details are in \Cref{sec:privacy_analysis_regret_glm_adversarial}.

\begin{remark}[Computational complexity]
\label{rem:complexity}
The per-round complexity consists of: (i) Criterion I check: $O(d^2)$ per arm using cached $\tV^{-1}$; (ii) binary-tree update: $O(d^2 \log T)$; (iii) criterion II check: $O(d^3)$ and the private optimizer: called $O(d \log T)$ times total. The amortized per-round complexity is $O(d^3 + Kd^2 + d^2 \log T)$.
\end{remark}
\section{Lower Bound}{\label{sec:private_lower_bound}}

\paragraph{Non-private baseline.}
\citet{pmlr-v130-abeille21a} show in Theorem~2 that, for generalized linear bandits with fixed contexts, any algorithm incurs regret at least $\Omega(d\sqrt{T/\kappa^{\ast}})$. The argument is statistical and transfers to both stochastic and adversarial contexts, but it does not capture the additional cost paid for privacy.

\paragraph{Privacy regime for the lower bound.}
To isolate the effect of privacy, we prove the lower bound in an \emph{easier} regime than ours: contexts are fixed and public, and only the reward sequence is protected under central $(\varepsilon,\delta)$-DP. Since revealing contexts can only help the learner, any lower bound here immediately transfers to our (harder) settings: any shuffle-DP or joint-DP algorithm for GLM bandits, run with contexts publicly revealed, remains central-DP with respect to rewards and hence inherits the bound. The lower bound therefore applies to both stochastic (shuffle-DP) and adversarial (joint-DP) settings.

\newcommand{\lowerboundjdp}{%
There exists a fixed context set $\mathcal{X}$ such that,
for any algorithm whose \emph{public action sequence} $X_{1:T}$
is $(\varepsilon,\delta)$-differentially private with respect
to the reward sequence under central DP, we have
\[
\mathbb{E}[R_T] \;\;\ge\;\; c\,\frac{d}{\varepsilon},
\qquad
\text{for all }
T \;\ge\; C\,\frac{d}{\varepsilon^2+\delta \log(1/\delta)},
\]
for universal constants $c,C>0$.
}

\begin{lemma}[Private lower bound under fixed public contexts]
\label{lemma:lb-jdp-glm}
\lowerboundjdp
\end{lemma}

The proof (see \cref{sec:private_lower_bound_proof}) follows a standard Assouad construction with $d$ coordinate arms and mutual-information bounds under DP. Combining this with the non-private $\Omega(d\sqrt{T/\kappa^\ast})$ bound yields $\mathbb{E}[R_T]\;\ge\;\Omega\!\left(d\sqrt{\tfrac{T}{\kappa^{\ast}}} \;+\; \tfrac{d}{\varepsilon}\right)$
for GLM contextual bandits under $(\varepsilon,\delta)$-joint DP, in both stochastic and adversarial settings when $T = \Omega(d/\varepsilon^2)$. The leading-in-$T$ term matches our joint-DP upper bound $\tilde{O}\!\big(d\sqrt{T/\kappa^\ast}\,\log T + d^{3/4}\sqrt{T/\varepsilon}\,(\log T)\,(d+\log T)^{1/4}\big)$, but a gap remains in the privacy correction ($d^{3/4}\sqrt{T/\varepsilon}\,(\log T)\,(d+\log T)^{1/4}$ upper vs.\ $d/\varepsilon$ lower); as noted in \cref{sec:glm_jdp}, this gap is not specific to GLMs.

%Here \(r>0\) is a \emph{design radius} controlling the norm of each context vector and chosen independent of $T$.

%\section{Simulations}\label{sec:simulations}

\section{Simulations}\label{sec:simulations}

We evaluate \cref{alg:jdp_glm_adversarial} against non-private baselines: \textbf{GLOC} \citep{scalablegeneralizedlinear}, \textbf{GLM-UCB} \citep{NIPS2010_c2626d85}, and \textbf{RS-GLinUCB} \citep{sawarni2024generalizedlinearbanditslimited}. 
We consider the probit model with $K=20$ arms, dimension $d=3$, horizon $T=5{,}000$ rounds (averaged over 10 runs), and vary $S \in \{2, 2.5, 3\}$ to obtain different $\kappa$ values. The parameter $\theta^{\ast}$ is sampled uniformly from a sphere of radius $S$, and arms are drawn from the unit ball.

Table~\ref{tab:kappa_main} shows cumulative regret as $\kappa$ increases. Methods with $\kappa$-dependent bounds (GLM-UCB, GLOC) degrade sharply as $\kappa$ grows, while both our private algorithm and the non-private $\kappa$-free baseline RS-GLinUCB remain stable. As $\varepsilon$ increases, the gap between Private-GLM and RS-GLinUCB narrows, illustrating the privacy-utility tradeoff. Additional experiments are in Appendix~\ref{sec:appendix_simulations}.

\begin{table}[h]
\centering
\small
\hspace*{-2.5em}
\begin{tabular}{lccc}
\toprule
Algorithm & $\kappa \approx 18.39$ & $\kappa \approx 56.44$ & $\kappa \approx 222.15$ \\
\midrule
GLM-UCB & 1400.27 & 1943.23 & 2323.06 \\
GLOC & 384.68 & 631.60 & 1346.59 \\
Private-GLM (ours, $\varepsilon=4$) & 674.17 & 721.26 & 744.35 \\
Private-GLM (ours, $\varepsilon=6$) & 516.58 & 546.97 & 559.04 \\
Private-GLM (ours, $\varepsilon=8$) & 425.53 & 445.27 & 452.60 \\
RS-GLinUCB & 289.69 & 170.10 & 200.62 \\
\bottomrule
\end{tabular}
\caption{Cumulative regret on synthetic probit bandits ($d=3$, $K=20$, $T=5000$, $\delta=2 \times 10^{-2}$).}
\label{tab:kappa_main}
\end{table}
\section{Conclusion}

We initiated the study of generalized linear contextual bandits under shuffle differential privacy and joint differential privacy. Both algorithms achieve regret comparable to their non-private analogues, and the leading $\sqrt T$ term remains free of the instance-specific parameter $\kappa$. Under shuffle DP with stochastic contexts, the inability to compute a distribution-optimal design privately forces us to use $G$-optimal design instead, incurring an additional $\sqrt{d}$ factor in the regret. By contrast, under joint DP with adversarial contexts the algorithm is near-optimal---matching the non-private rate up to poly-logarithmic factors in $T$ with no extra dimension dependence. Future work includes tightening these regret bounds for generalized linear models and establishing matching lower bounds for this privacy setting.

\paragraph{Acknowledgements.}
The author thanks Ayush Sawarni for technical discussions on batched algorithms in generalized linear bandits, Justin Whitehouse for discussions on log-likelihood analysis for differential privacy, and Mohak Goyal and Siddharth Chandak for suggestions on exposition.

% ---------------- References ----------------
\bibliographystyle{plainnat}
\bibliography{refs}

% =========================================================
% Supplementary / Appendix
% =========================================================
\clearpage
\onecolumn

\appendix

\paragraph{Appendix roadmap.}
The appendix is organized around the two privacy models---shuffled privacy and joint differential privacy (JDP)---that underlie our results:
\begin{itemize}
    \item \textbf{Synthetic experiments.}
    \Cref{sec:appendix_simulations} presents synthetic experiments for \cref{alg:jdp_glm_adversarial}.
    \item \textbf{Privacy definitions.}
    Appendix~\ref{app:privacy_defs} collects the formal definitions of local, shuffle, and joint differential privacy used throughout the paper.
    \item \textbf{Extended related work.}
    Appendix~\ref{app:related} discusses additional related work on privacy amplification by shuffling and limited-adaptivity bandits.
    \item \textbf{Shuffled privacy, stochastic contexts.}
    \Cref{sec:shuffled_privacy_instantiation_utility_regret_tradeoff,sec:proof_theorem1_regret_ass,sec:proof_main_theorem_glm_shuffled} present the full privacy and regret analysis of \Cref{alg:shuffled_glm} for generalized linear contextual bandits.
    \item \textbf{Background material.}
    Standard facts on generalized linear models and optimal experimental designs are collected in \Cref{sec:appendix_glm_properties,sec:optimal_design_guarantees}.
    \item \textbf{Joint differential privacy, adversarial contexts.}
    A detailed JDP treatment for generalized linear contextual bandits with adversarial contexts appears in
    \Cref{sec:jdp_instantiation_utility_regret_tradeoff,sec:glm_jdp_regret_analysis_assumption,sec:confidence_bounds_non_switching_criterion_I,sec:bounding_instantaneous_regret,sec:proof_main_theorem_jdp_regret}.
    Supporting lemmas and privacy proofs for policy-switching arguments are given in
    \Cref{sec:useful_lemmas_adversarial_contexts,sec:utility_privacy_switching_I_II}.
    \item \textbf{Lower bound.}
    \Cref{sec:private_lower_bound_proof} gives the detailed proof of the joint-DP lower bound stated in \cref{sec:private_lower_bound}.
    \item \textbf{Approximate optimization.}
    \Cref{sec:convex_relaxation_error_solution} bounds the error in the estimated parameter~$\hat\theta$ that arises from approximate convex minimization due to privacy constraints.
    \item \textbf{Shuffle-private primitives.}
    \Cref{sec:vector_summation_protocol,sec:shuffle_convex_optimizer} summarize the shuffle-private vector-summation protocol and the shuffle-private convex optimizer of \citet{Cheu2021ShufflePS}, both of which are employed in our private bandit algorithms.
\end{itemize}

\section{Notation summary}
\label{sec:notation_summary}

For convenience, we collect the symbols used throughout the paper and proofs. Generic notation introduced in passing in \cref{sec:notations_and_prelims} ($\|x\|$, $\|x\|_M$, $\Delta(\gX)$, $\widetilde O(\cdot)$, $a\wedge b$, $[n]$, $[m,n]$) is not repeated here.

\renewcommand{\arraystretch}{1.15}
\begin{center}
\begin{tabular}{@{}p{0.20\linewidth} p{0.55\linewidth} p{0.18\linewidth}@{}}
\toprule
\textbf{Symbol} & \textbf{Meaning} & \textbf{Defined in} \\
\midrule
\multicolumn{3}{@{}l}{\textit{Group 1: Sets, indices, horizon}}\\
\midrule
$T,\, K,\, d$ & horizon, number of arms per round, feature dimension & \cref{sec:notations_and_prelims}\\
$\gX_t,\; x_t \in \gX_t$ & context set / pulled arm at round $t$ & \cref{sec:notations_and_prelims}\\
$\widetilde{\gX}_t$ & scaled arm set under arm reweighting & \eqref{eq:arm_scaling}\\
$\gB_k,\, B_k,\, M$ & $k$-th batch, its size, total number of batches & \eqref{eq:batch_B1_B2}--\eqref{eq:batch_Bk}\\
$\gT_o$ & set of exploration rounds (Criterion I) & \cref{alg:jdp_glm_adversarial}\\
$\gD$ & context distribution (setting M1) & \cref{sec:notations_and_prelims}\\
$\gK_\gX$ & G-optimal design distribution on $\gX$ & \cref{sec:notations_and_prelims}\\
$\pi^G$ & sampling policy from $\gK_\gX$ & \cref{sec:glm_alg}\\
\midrule
\multicolumn{3}{@{}l}{\textit{Group 2: GLM and instance constants}}\\
\midrule
$\theta^\ast,\, \|\theta^\ast\|\le S$ & unknown true parameter; norm bound & \cref{def:glm}\\
$r_t \in [0,R]$ & reward, with known bound $R$ & \cref{ass:boundedness}\\
$b(\cdot),\, \mu(\cdot)=\dot b(\cdot)$ & log-partition function, link (mean) function & \cref{def:glm}\\
$\kappa$ & $1/\min_{(\gX,x)}\dot\mu(\langle x,\theta^\ast\rangle)$; can be $\Theta(e^S)$ & \eqref{eq:kapa_kappa_hat_mone},\eqref{eq:kapa_mtwo}\\
$\hat\kappa$ & $1/\E[\dot\mu(\langle x^\ast,\theta^\ast\rangle)]$ at the optimal arm & \eqref{eq:kapa_kappa_hat_mone}\\
$\kappa^\ast$ & $1/\max_{(\gX,x)}\dot\mu(\langle x,\theta^\ast\rangle)$; benign, $\geq 1/R^2$ & \eqref{eq:kapa_kappa_hat_mone},\eqref{eq:kapa_mtwo}\\
\midrule
\multicolumn{3}{@{}l}{\textit{Group 3: Privacy parameters}}\\
\midrule
$\varepsilon,\, \delta$ & target $(\varepsilon,\delta)$-DP parameters & \cref{sec:privacy}\\
$\zeta$ & failure probability used in JDP analysis & \cref{thm:jdp_main}\\
$\sigma$ & sub-Gaussian noise variance proxy of shuffle vector summation & \eqref{eq:sigma_nu_defn}\\
$\widetilde\sigma$ & high-prob bound on perturbation matrix & \cref{ass:sub_Gaussian_summation}\\
$\nu,\, \nu_1,\, \nu_2$ & cumulative-loss optimizer error & \cref{ass:sco_glm_shuffled},\,\cref{ass:sco_loss}\\
$\rho$ & zCDP budget per private optimizer call & \cref{sec:glm_jdp}\\
\midrule
\multicolumn{3}{@{}l}{\textit{Group 4: Algorithm parameters (regularizer + confidence widths)}}\\
\midrule
$\lambda$ & ridge regularization parameter & \eqref{eq:lambda_defn_mone},\,\eqref{eq:lambda_defn_glm_adv}\\
$\lambdau,\, \lambdal$ & high-prob upper / lower bounds on noisy regularizer eigenvalues & \eqref{line:lambda_l_u_defn}\\
$\gamma$ & confidence-width scaling (Criterion I / shuffle batch) & \eqref{eq:gamma_defn_mone},\,\eqref{eq:gamma_defn_glm_adv}\\
$\gamma(\lambdau,\lambdal)$ & parameterized form used in proofs & \eqref{eq:gamma_parametrized_lambda_min_max}\\
$\beta$ & confidence-width scaling (Criterion II, JDP) & \eqref{eq:beta_defn_glm_adv}\\
$\beta(\lambdau,\lambdal)$ & parameterized form used in proofs & \cref{sec:confidence_bounds_non_switching_criterion_I}\\
$\beta(x)$ & per-arm exponential reweighting factor & \eqref{eq:beta_def}\\
$\eta$ & step size of $\gP_{\text{GD}}$ & \cref{sec:shuffle_convex_optimizer}\\
\midrule
\multicolumn{3}{@{}l}{\textit{Group 5: Design matrices}}\\
\midrule
$\mathbf V,\, \tV$ & noisy warm-up (and Criterion-I) design matrix & \cref{alg:shuffled_glm},\cref{alg:jdp_glm_adversarial}\\
$\mathbf H_k,\, \tH_k$ & noisy Hessian-weighted batch design matrix & \cref{alg:shuffled_glm}\\
$\tH_t$ & noisy Hessian-weighted continual-release matrix (JDP) & \cref{alg:jdp_glm_adversarial}\\
$\tH^\ast_k,\, \tH^\ast_o$ & true Hessian-weighted design matrix (used in proofs) & \cref{sec:proof_theorem1_regret_ass},\,\cref{sec:confidence_bounds_non_switching_criterion_I}\\
$\tH^{(u),\ast}_t,\, \tH^{(l),\ast}_t$ & high-prob upper / lower envelope on $\tH_t$ & \cref{sec:confidence_bounds_non_switching_criterion_I}\\
$\gB_V,\, \gB_H$ & binary-tree mechanisms releasing $\tV,\tH_t$ & \cref{alg:jdp_glm_adversarial}\\
$\gR^V_t,\, \gR^H_t$ & binary-tree noise increments at round $t$ & \cref{alg:jdp_glm_adversarial}\\
\midrule
\multicolumn{3}{@{}l}{\textit{Group 6: Estimators and confidence bounds}}\\
\midrule
$\hat\theta_k$ & private MLE at the end of batch $k$ (shuffle) & \cref{alg:shuffled_glm}\\
$\hat\theta_o,\, \hat\theta_\tau$ & private MLE on $\gT_o$ resp.\ on $[t-1]\setminus\gT_o$ (JDP) & \cref{alg:jdp_glm_adversarial}\\
$\widetilde\theta,\, \widetilde\theta_o,\, \widetilde\theta_\tau$ & non-private (unconstrained) MLE counterparts (proofs only) & \cref{sec:proof_theorem1_regret_ass},\,\cref{sec:confidence_bounds_non_switching_criterion_I}\\
$\Theta_I,\, \Theta_{II}$ & constraint sets for the two DP-SGD calls & \cref{sec:glm_jdp}\\
$\text{UCB}_k(x),\, \text{LCB}_k(x)$ & confidence bounds at batch $k$ (shuffle) & \eqref{eq:ucb_def}--\eqref{eq:lcb_def}\\
$\text{UCB}_o,\, \text{LCB}_o$ & exploration-time confidence bounds (JDP) & \eqref{eq:ucb_o}--\eqref{eq:lcb_o}\\
$\text{UCB}(x,\mathbf H,\theta)$ & generic Hessian-weighted UCB & \eqref{eq:ucb_general}\\
$\gP_{\text{GD}}$ & shuffle-private projected gradient descent optimizer & \cref{alg:shuffle_convex_optimizer}\\
$R_2,\, S_2,\, A_2$ & shuffle vector-summation triplet (randomizer / shuffler / analyzer) & \cref{sec:vector_summation_protocol}\\
\bottomrule
\end{tabular}
\end{center}
\renewcommand{\arraystretch}{1.0}

\section{Synthetic Experiments}\label{sec:appendix_simulations}

We evaluate the practicality of \cref{alg:jdp_glm_adversarial} by comparing it against several non-private baselines for generalized linear bandits: \textbf{ECOLog} \cite{faury2022jointlyefficientoptimalalgorithms}, \textbf{GLOC} \cite{scalablegeneralizedlinear}, \textbf{GLM-UCB} \cite{NIPS2010_c2626d85}, and \textbf{RS-GLinUCB} \cite{sawarni2024generalizedlinearbanditslimited}.  

ECOLog achieves computational efficiency with overall time complexity $\tilde{O}(T)$ and avoids $\kappa$ dependence in regret. However, it is designed specifically for logistic bandits and does not extend directly to general GLMs. GLOC is also computationally efficient, but its regret does depend on $\kappa$. GLM-UCB was the first algorithm proposed for generalized linear contextual bandits. Finally, RS-GLinUCB is both computationally efficient and achieves regret bounds without $\kappa$ dependence for GLMs.

For our simulations, we fix the privacy parameter $\delta=2\times 10^{-2}$ and vary $\epsilon \in \{2,4,8\}$. We consider both logistic and probit models with $\zeta = 2\times 10^{-2}$.

\paragraph{Logistic model.}  
We compare Private-GLM (\cref{alg:jdp_glm_adversarial}) against RS-GLinUCB, ECOLog, and GLOC, which are the only algorithms with time complexity $\tilde{O}(T)$ in \cref{fig:logistic}. The setting is $d=5$ dimensions and $K=20$ arms. The parameter $\theta^{\ast}$ is sampled from a $d$-dimensional sphere of radius $S=3.5$, and arms are sampled uniformly from the unit ball. We run $T=15{,}000$ rounds, repeated 10 times. While $\epsilon=2$ results in some performance degradation, the regret of Private-GLM at $\epsilon=8$ is nearly identical to RS-GLinUCB.

\paragraph{Probit model.}  
We compare Private-GLM against RS-GLinUCB, GLM-UCB, and GLOC in \cref{fig:probit}. The setting is again $d=5$ dimensions and $K=20$ arms, with $\theta^{\ast}$ sampled from a sphere of radius $S=3$ and arms drawn from the unit ball. We run $T=5{,}000$ rounds, repeated 10 times. Similar to the logistic case, regret at $\epsilon=2$ is somewhat higher, but at $\epsilon=8$ Private-GLM performs almost on par with RS-GLinUCB.

\begin{figure}[t]
  \centering
  \begin{subfigure}[t]{0.48\textwidth}
    \centering
    \includegraphics[width=\linewidth]{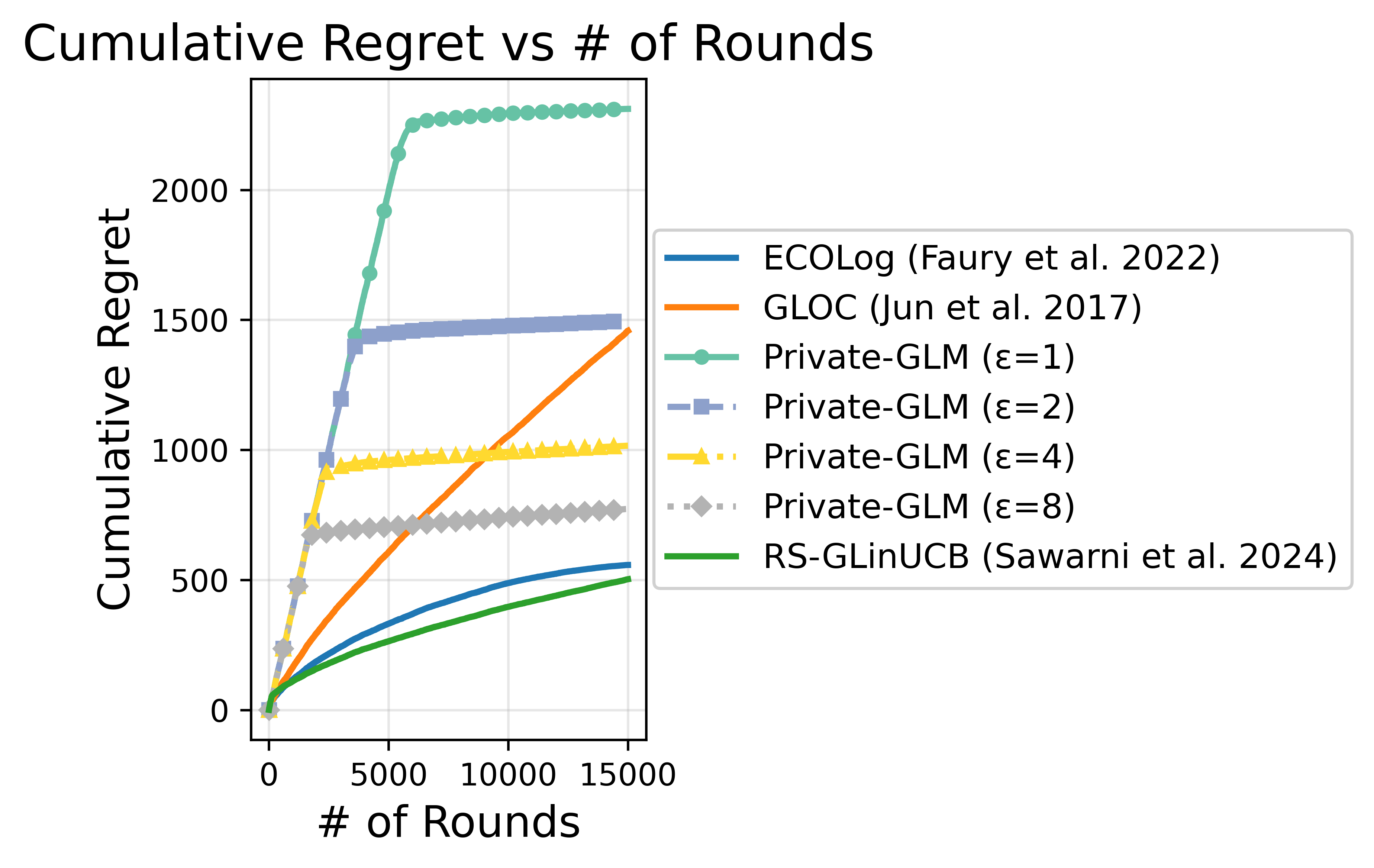}
    \caption{Logistic model}
    \label{fig:logistic}
  \end{subfigure}
  \hfill
  \begin{subfigure}[t]{0.48\textwidth}
    \centering
    \includegraphics[width=\linewidth]{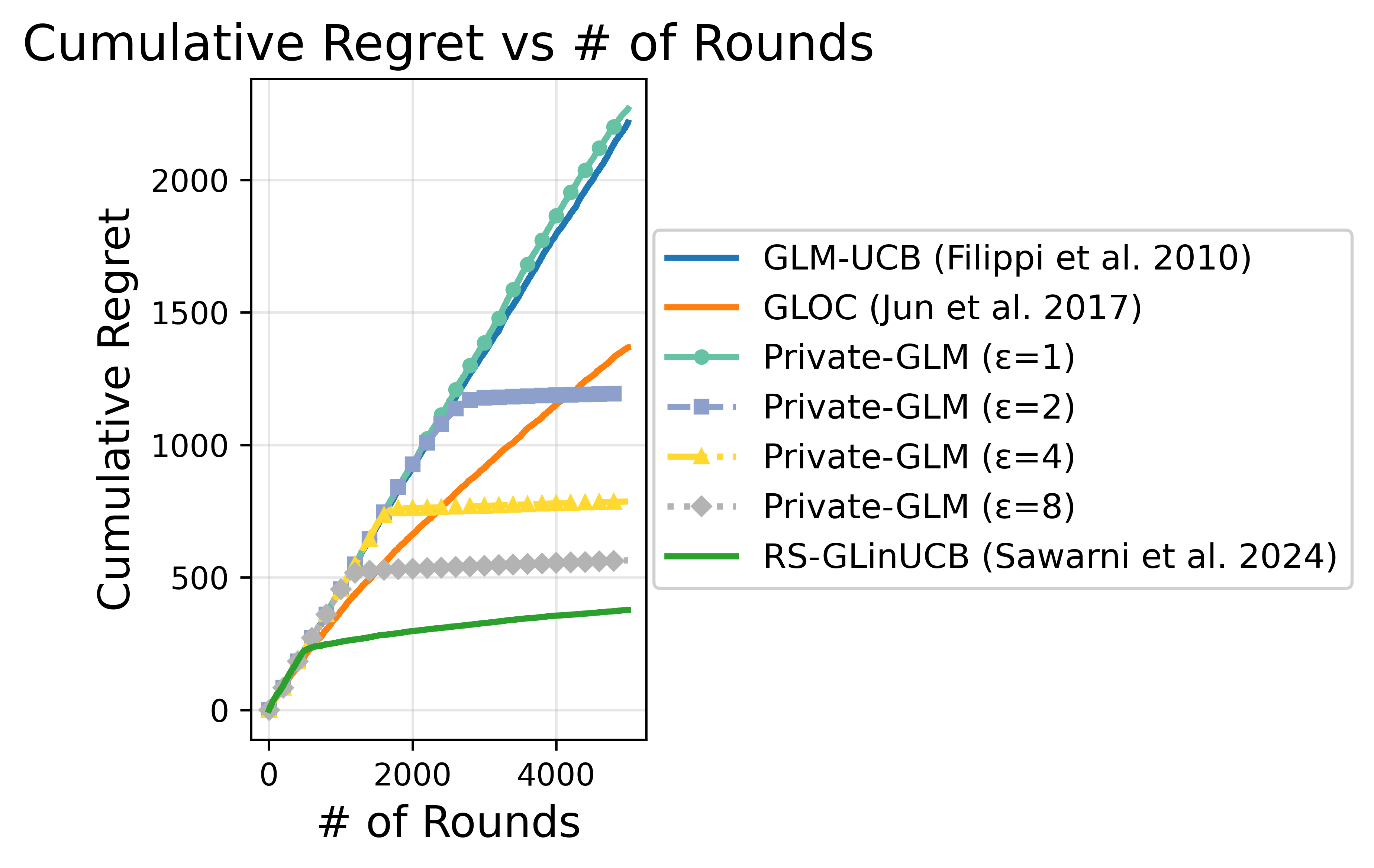}
    \caption{Probit model}
    \label{fig:probit}
  \end{subfigure}
\caption{Cumulative regret under different generalized linear models. Private-GLM variants are plotted with distinct colors/markers.}
  \label{fig:regret}
\end{figure}

\begin{remark}
%For these simulations, the privacy parameters $(\epsilon_o, \delta_o)$ of the optimizer $\mathcal{P}_{GD}$ are not set exactly as in \cref{line:epsilon_0_delta_0_optimizer}. Instead, they are computed by solving the quadratic relation from Theorem 3.20 of \citet{dworkdpbook}. The optimizer is not implemented as the shuffle convex optimizer but instead uses per-step noise addition and step counts calibrated via Rényi-DP analysis \cite{Renyidp}, as is standard in DP-SGD.  

For the simulations, we calibrate the per-step noise using Rényi-DP accounting \cite{Renyidp}, a generalization of zero-CDP \cite{cryptoeprint:concentratedDP} that gives tighter privacy accounting for iterative Gaussian mechanisms such as DP-SGD.

The constants $\lambda$, $\gamma$, and $\beta$ are treated as hyperparameters rather than fixed as in the theoretical construction. The thresholds $\ccut$ and $\ccuttwo$ are scaled by constants in practice. However, the algorithm stops updating $\hat \theta_o$ when $\ccut$ is crossed and $\hat \theta_\tau$ when $\ccuttwo$ is crossed. This ensures that the overall privacy budget is never exceeded.

%Further, one may wonder why Private-GLM sometimes does better than non-private analog in RS-GLinUCB. This is because of the fact that we optimized the regularizer in the convex optimizer of Private-GLM slightly differently from the one in RS-GLinUCB.
\end{remark}

% ====================
% APPENDIX: PRIVACY DEFINITIONS (Moved from main paper)
% All citations verified against refs.bib
% ====================

\section{Privacy Definitions}
\label{app:privacy_defs}

Let $\gD$ denote the data universe and $n \in \sN$ the number of users. Let $D_i \in \gD$ denote the data of user $i$, and $D_{-i} \in \gD^{n-1}$ denote the data of all users except $i$. Let $\varepsilon > 0$ and $\delta \in (0,1]$ be privacy parameters.

\begin{definition}[Differential Privacy \cite{dworkdpbook}]
A mechanism $\gM$ satisfies $(\varepsilon, \delta)$-DP if, for each user $i \in [n]$, each pair of datasets $D, D' \in \gD^n$ differing only in user $i$'s data, and each event $E$ in the range of $\gM$,
\[
\prob{\gM(D_i, D_{-i}) \in E} \leq e^\varepsilon \prob{\gM(D_i', D_{-i}) \in E} + \delta.
\]
\end{definition}

\begin{definition}[Local Differential Privacy]
A mechanism $\gM$ satisfies $(\varepsilon, \delta)$-LDP if for each user $i \in [n]$, each pair of data points $D_i, D_i' \in \gD$, and each event $E$,
\[
\prob{\gM(D_i) \in E} \leq e^\varepsilon \prob{\gM(D_i') \in E} + \delta.
\]
\end{definition}

A shuffle protocol $\gP = (\gR, \gS, \gA)$ consists of: (i) a local randomizer $\gR$, (ii) a shuffler $\gS$, and (iii) an analyzer $\gA$. Each user $i$ applies $\gR$ to their data $D_i$ and sends the resulting messages to the shuffler, which uniformly permutes all messages before forwarding them to the analyzer. We consider a multi-round protocol \cite{Cheu2021ShufflePS,lowy2024privatefederatedlearningtrusted} where users participate across $I$ rounds.

\begin{definition}[Shuffle Differential Privacy]
A multi-round protocol $\gP = (\gR, \gS, \gA)$ for $n$ users satisfies $(\varepsilon, \delta)$-SDP if the mechanism $\gS \circ \gR^n = \left[\gS(\gR^{(i)}(D_1), \ldots, \gR^{(i)}(D_n))\right]_{i=1}^I$ satisfies $(\varepsilon, \delta)$-DP.
\end{definition}

\begin{definition}[Joint Differential Privacy \cite{mechanism_design_large_games,shariff2018differentiallyprivatecontextuallinear}]
A learning algorithm $\gM: \gU^T \to \gA^T$ for contextual bandits is $(\varepsilon, \delta)$-JDP if for all $t \in [T]$, all user sequences $U_T, U_T'$ differing only in user $t$, and all action sets $A_{-t} \subseteq \gA^{T-1}$,
\[
\prob{\gM_{-t}(U_T) \in A_{-t}} \leq e^\varepsilon \prob{\gM_{-t}(U_T') \in A_{-t}} + \delta,
\]
where $\gM_{-t}$ denotes the actions given to all users except user $t$.
\end{definition}

Joint DP requires that each user's action be private with respect to all other users' data. This is weaker than central DP but sufficient to avoid the linear regret lower bound of \cite{shariff2018differentiallyprivatecontextuallinear}. By the Billboard Lemma \cite{mechanism_design_large_games}, shuffle DP implies joint DP.

% ====================
% APPENDIX: EXTENDED RELATED WORK
% ====================

\section{Extended Related Work}
\label{app:related}

\paragraph{Privacy amplification by shuffling.}
The shuffle model bridges local and central DP, showing that randomly permuting locally privatized data achieves central-DP-level protection \cite{bittau2017prochlo,erlingsson2019amplification,Cheu_2019,feldman2021hiding}. Multi-message protocols for binary summation and histograms demonstrate that sending a few short messages per user significantly tightens shuffling guarantees \cite{ghazi2020pure,balle2020multiMessage,pmlr-v139-ghazi21a}.

\paragraph{Limited adaptivity in bandits.}
Several works explore limited adaptivity in linear contextual bandits \cite{ruan2021linearbanditslimitedadaptivity,NIPS2011_e1d5be1c,han2020sequentialbatchlearningfiniteaction,pmlr-v195-hanna23a}, motivated by settings where policies cannot change frequently, such as clinical trials and online advertising. Extensions to logistic bandits \cite{faury2022jointlyefficientoptimalalgorithms} and GLMs \cite{li2017provablyoptimalalgorithmsgeneralized,sawarni2024generalizedlinearbanditslimited} show that polylogarithmic policy switches suffice for optimal regret. Our algorithms build on the batching framework of \citet{sawarni2024generalizedlinearbanditslimited}.

\paragraph{Shuffle DP convex optimization.}
\citet{Cheu2021ShufflePS} developed shuffle-DP algorithms for stochastic convex optimization, achieving near-central-DP accuracy. We adapt their techniques for the GLM estimation subroutine in our stochastic-context algorithm.

\paragraph{Detailed comparison with private linear bandits.}
Shuffle DP for linear contextual bandits has been studied in \cite{pmlr-v167-garcelon22a,pmlr-v162-chowdhury22a}, and for $K$-armed bandits in \cite{mabshuffle}. Joint DP for linear contextual bandits was first introduced in \citet{shariff2018differentiallyprivatecontextuallinear}, using binary tree mechanisms \cite{continualobservation} to update design matrix estimators at each step. \citet{azize2024concentrateddifferentialprivacybandits} improved regret bounds under relaxed assumptions where context sets are public and only rewards require privacy protection. \citet{chakraborty2024fliphatjointdifferentialprivacy} further tightened regret bounds for sparse linear contextual bandits. 

Local DP in contextual bandits has been studied in \cite{zheng2021locallydifferentiallyprivatecontextual}, but their regret scales as $\tilde{O}(T^{3/4}/\varepsilon)$. More recently, \citet{li2024optimalregretlocallyprivate} proposed an alternative approach that achieves $\tilde{O}(\sqrt{T}/\varepsilon)$ regret, though their method incurs exponential dependence on the dimension $d$.

\paragraph{Near-free privacy in linear bandits.}
In linear bandits (without contexts), \citet{osamaalmostfree} show that joint DP can be achieved almost for free, with only an additive logarithmic term involving $\varepsilon$. \citet{azize2024concentrateddifferentialprivacybandits} extend this to linear contextual bandits under relaxed assumptions—specifically, with public contexts and private rewards—and demonstrate that privacy comes at negligible cost. \citet{chakraborty2024fliphatjointdifferentialprivacy} further show that in sparse linear contextual bandits, one can obtain nearly optimal joint DP guarantees when the sparsity parameter satisfies $s \ll d$. \citet{pmlr-v247-azize24a} identifies the characterization of optimal regret for $(\varepsilon, \delta)$-joint DP linear contextual bandits in the general case as an open problem.

% \paragraph{Private GLM bandits.}
% Prior work on differentially private contextual bandits under GLMs is limited. \citet{han2021generalizedlinearbanditslocal} analyze the stronger local DP model and achieve $\tilde{O}(T^{1/2}/\varepsilon)$ regret, but their results rely on eigenvalue assumptions on the context distribution. To the best of our knowledge, no existing work addresses shuffle or joint DP in GLM bandits prior to ours. Our algorithms close this gap: we impose no spectral assumptions on the context distribution and achieve $\tilde{O}(\sqrt{T/\varepsilon})$ regret in both privacy settings.

\section{Regret and Privacy Analysis of \texorpdfstring{\Cref{alg:shuffled_glm}}{Algorithm}}{\label{sec:shuffled_privacy_instantiation_utility_regret_tradeoff}}

\subsection{Parameter settings}\label{sec:params_shuffled_glm}

We collect here the explicit parameter settings for \Cref{alg:shuffled_glm}, referenced from the main body. Throughout this section, $C>0$ denotes a sufficiently large universal constant (whose value may vary line-to-line). The regularization parameter $\lambda$ and confidence-scaling parameter $\gamma$ are set as
\begin{align}
    \lambda &= 20dR \log T
    + \sigma\!\left(2\sqrt{\log (2dMT/\delta)} + 4\sqrt d\right), \label{eq:lambda_defn_mone}\\[4pt]
    \gamma &= C\, RS \sqrt{\lambda}
    + \sqrt{4RS \nu}. \label{eq:gamma_defn_mone}
\end{align}
The choice $C\,RS\sqrt{\lambda}$ in $\gamma$ is conservative: it is taken large enough to dominate the parameterized form $\gamma(\lambdau,\lambdal)$ of \cref{eq:gamma_parametrized_lambda_min_max}, which arises directly from the proof of \Cref{lemma:bounding_est_theta_theta}. For traceability, the proofs below display the exact numerical constants ($24$, $30$) inherited from \citet{sawarni2024generalizedlinearbanditslimited}; any sufficiently large $C$ suffices.
The noise parameters $\sigma$ is the parameter controlling the noise introduced by the shuffled vector summation protocol and $\nu$ is optimization error of the shuffle-private convex optimizer), which are set as
\begin{align}
    \sigma = O\!\left( \frac{R \log(d^2/\delta)}{\kappa^{\ast}\varepsilon} \right), \qquad
    \nu = O\!\left( \frac{RS^2}{\varepsilon}
    \sqrt{d \log^3(Td/\delta)} \right).
    \label{eq:sigma_nu_defn}
\end{align}
These choices explicitly incorporate variance from the shuffled vector summation protocol and the private convex optimizer, ensuring confidence bounds remain valid under differential privacy.

We further define $\lambda_{\max}$ and $\lambda_{\min}$ as follows; these will serve as high-probability upper and lower bounds on the eigenvalues of the perturbed regularization matrices $\{\mathcal{N}_k\}$ introduced in \Cref{sec:additional_notation} (see \Cref{obs:bounds_lambda_min_lambda_max}).
\begin{equation}
\label{line:lambda_l_u_defn}
\begin{aligned}
\lambda_{\max}
&:= \lambda
+ \frac{8 R \log(d^2 / \delta)}{\kappa^{\ast}\varepsilon}
\left( 4\sqrt{d} + 2 \sqrt{\log(2MT/\delta)} \right), \\
\lambda_{\min}
&:= \lambda
- \frac{8 R \log(d^2 / \delta)}{\kappa^{\ast}\varepsilon}
\left( 4\sqrt{d} + 2 \sqrt{\log(2MT/\delta)} \right).
\end{aligned}
\end{equation}

\subsection{Assumptions}{\label{sec:assumptions_shuffled_glm}}

We now present two assumptions that our \Cref{alg:shuffled_glm} satisfies with respect to design matrix $\tV$, $\tH_k$ and the parameter estimate $\hat \theta_k$ at end of each batch. We prove these assumptions hold true in \Cref{sec:regret_analysis_alg_glm_shuffled}.

\begin{assumption}{\label{ass:sub_Gaussian_summation}}
    There exists constant $\tilde{\sigma}$ such that the following property holds.
    \begin{itemize}
        \item For batch $k=1$, $\tV - \sum_{t \in \mathcal{B}_1} x_t x^{\top}_t - \lambda \mI$ is a symmetric, with independent sub-Gaussian upper diagonal elements with zero mean and a variance proxy $\tilde{\sigma}^2$.
        \item For batch $k>1$, $\tH_k - \sum_{t \in \mathcal{B}_k} \frac{\dot{\mu} \left(\langle x_t , \hat \theta_1 \rangle\right) }{\beta(x_t)} x_t x^{\top}_t - \lambda \mI$ is a symmetric, with independent sub-Gaussian upper diagonal elements with zero mean and a variance proxy $\tilde{\sigma}^2$.
    \end{itemize}
\end{assumption}
\begin{assumption}{\label{ass:sco_glm_shuffled}}
    $\exists \nu >0$ such that for every batch $k$,
    \[
    \left|\gL_{\gS}(\hat \theta_k;\lambda) - \min\limits_{\theta: ||\theta||_2 \leq S}\gL_{\gS}(\theta;\lambda)\right| < \nu,
    \]
    where $\gS= \{r_t,x_t\}_{t \in \mathcal{B}_k}$ and
    \[
    \gL_{\gS}(\theta;\lambda) := \sum_{t \in \mathcal{B}_k} \ell(\theta, x_t, r_t) + \frac{\lambda}{2}\|\theta\|_2^2.
    \]
    This holds with probability at least $\left(1- \frac{1}{T^4}\right)$.
\end{assumption}

\paragraph{Convention bridging $\gP_{\text{GD}}$ to $\gL_{\gS}$.}
The shuffle optimizer $\gP_{\text{GD}}$ (\Cref{alg:shuffle_convex_optimizer}) minimizes the \emph{empirical mean} of a per-sample loss, while \Cref{ass:sco_glm_shuffled} measures error against the \emph{cumulative} regularized loss $\gL_{\gS}(\theta;\lambda)$. We bridge the two by invoking $\gP_{\text{GD}}$ on the per-sample loss $\tilde\ell_t(\theta) := \ell(\theta,x_t,r_t) + \tfrac{\lambda}{2 B_k}\|\theta\|^2$ (cf.\ \Cref{alg:shuffled_glm}, lines~\ref{line:warmup_end} and~\ref{line:batch_end}). Multiplying through by $B_k$, the empirical mean $\tfrac{1}{B_k}\sum_{t \in \gB_k} \tilde\ell_t$ has the same minimizer as $\tfrac{1}{B_k}\gL_{\gS}(\theta;\lambda)$, so an empirical-mean error of $\nu_{\textsc{pgd}}$ from $\gP_{\text{GD}}$ corresponds to a cumulative-loss error of $\nu = B_k\,\nu_{\textsc{pgd}}$ in \Cref{ass:sco_glm_shuffled}. The same convention applies on the JDP side (\Cref{ass:sco_loss}), with $B_k$ replaced by the relevant index-set size.

We first show that our algorithm is $(\varepsilon,\delta)$ shuffle differentially private in \Cref{sec:privacy_analysis_alg_glm_shuffled} and then show in \Cref{sec:regret_analysis_alg_glm_shuffled} that our algorithm satisfies the two assumptions mentioned above with appropriate values of $\nu$ and $\tilde \sigma^2$ and then provide a regret bound in \Cref{sec:proof_theorem1_regret_ass} under these two assumptions.

%Achieving privacy via Shuffle Convex Optimization and Vector Summation

% In this section, we define the shuffle protocol identical to \Cref{sec:amplification_vector_summation_sdp} with the exception that the randomizer $R_2$ is just applied on $x_tx^T_t$ for batch $k=1$ and $\frac{\dot{\mu} (\langle x_t, \hat \theta_1\rangle)}{\beta(x_t)}x_t x^T_t$ for batches $k>1$ but not on $r_tx_t$. The parameter $\hat \theta_k$ is computed at each batch $k$ by passing the corresponding data points $\{r_t,x_t\}_{t \in \mathcal{B}_k}$ through a shuffle convex optimizer $\mathcal{P}_{\text{GD}}$ [\Cref{alg:shuffle_convex_optimizer}]. 

In batch $k=1$, the shuffled summation protocol is applied on $x_t x^{\top}_t$ for every batch $k=1$ and for subsequent batches $k>1$ it is applied on $\frac{\dot{\mu} (\langle x_t, \hat \theta_1\rangle)}{\beta(x_t)}x_t x^{\top}_t$. The parameter $\hat \theta_k$ is computed at each batch $k$ by passing the corresponding data points $\{r_t,x_t\}_{t \in \mathcal{B}_k}$ through a shuffle convex optimizer $\mathcal{P}_{\text{GD}}$ [\Cref{alg:shuffle_convex_optimizer}].

\subsection{Privacy analysis}{\label{sec:privacy_analysis_alg_glm_shuffled}}

%Denote the size of $k^{th}$ batch by $B_k = \gT_k - \gT_{k-1}$ and define noise levels $\sigma_1 = \sigma_2 = \frac{4 \sqrt{2 \log (\frac{2.5B_k}{\delta}) \log (2/ \delta)}}{\varepsilon \sqrt{B_k}}$.

The privacy analysis of the shuffle convex optimizer $\mathcal{P}_{\text{GD}}$ directly follows from \Cref{lemma:elimination_first_step} which internally uses the advanced composition theorem \cite{dworkdpbook}. The privacy analysis of the vector summation protocol  follows from Theorem 4.2 of \citet{Cheu2021ShufflePS} which can be shown to be $(\varepsilon/2, \delta/2 - \frac{1}{T^2})$ shuffle differentially private conditioned on the event $\frac{\dot{\mu}(\langle x_t, \hat \theta_1\rangle)}{\beta(x_t)} \leq \frac{1}{\kappa^{\ast}}$. Further, since this event holds with probability at least $1-\frac{1}{T^2}$, the desired $(\varepsilon,\delta)$ shuffle differential privacy holds for estimators $\tV \text{ and } \{\tH_k\}_{k \geq 2}$.

Since, the set of estimators $\{\hat \theta_k\}_{k=1}^{M}, \tV \text{ and } \{\tH_k\}_{k \geq 2}$ are $(\varepsilon,\delta)$ differentially private with the data $\{r_t,\gX_t\}$ in the $k^{th}$ batch, the desired $(\varepsilon,\delta)$ multi-rounded shuffled privacy guarantee naturally follows. Further, we can show that this algorithm naturally satisfied JDP via Billiboard's lemma in \cite{mechanism_design_large_games}.

\subsection{Regret Analysis}\label{sec:regret_analysis_alg_glm_shuffled}

One may observe from \Cref{lemma:utility_privacy_matrix_summation_protocol} and \Cref{lemma:utility_shuffle_convex} that \Cref{ass:sub_Gaussian_summation} and \Cref{ass:sco_glm_shuffled} are satisfied with $\tilde{\sigma}^2 = O\left( \frac{R^2 (\log \left(d^2/\delta)\right)^2}{(\kappa^{\ast})^2\varepsilon^2}\right)$ and $\nu = O\left( \frac{RS^2}{\varepsilon} \sqrt{d \log^3(Td/\delta)} \right)$. The value of $\eta$ follows on applying \Cref{lemma:utility_shuffle_convex} with Lipschitzness of the loss function given by $L = RS + S + \frac{S\lambdau}{B_k}$.

A simple application of \Cref{theorem:main_theoremone_assumptions} yields the following the regret in expectation. Further, \cref{obs:bounds_lambda_min_lambda_max} implies that the bound $\lambdal \geq 20dR \log T$ is satisfied with high probability. 

%\sss{Also comment how value of $\lambda$ fits here.}

% \begin{theorem}{\label{thm:shuffled_regret_glm}}
%     For $\varepsilon<5, \delta<1/2$, the \Cref{alg:shuffled_glm} satisfies $(\varepsilon,\delta)$ shuffle differential privacy and the expected regret may be bounded as 

%     \begin{align}
%         R_T &= O\left(\frac{1}{ \sqrt{\kappa^{\ast}}}\left({d^{3/2}} \sqrt{T \log T} +\frac{d^{9/8} \sqrt{ \log (d^2/\delta)}}{\sqrt{\varepsilon \kappa^{\ast}}} \sqrt{T\log^{3/4} (2dT/\delta)}  \right) \log \log T\right) \text{poly}(R,S)\nonumber\\
%          & \hspace{2 em} + O\left(e^{3S} \kappa^{1/3} T^{1/3}\left({d^{2}} (\log T)^{2/3} +\frac{d^{3/2} (\log (d^2/\delta))^{2/3}}{\left(\varepsilon \kappa^{\ast}\right)^{2/3}} \log^{1/2} (2dT/\delta)  \right) \log \log T\right) \text{poly}(R,S) \nonumber
%     \end{align}
% \end{theorem}

\begin{theorem}{\label{thm:shuffled_regret_glm}}
For $\varepsilon<5$ and $\delta<1/2$, \Cref{alg:shuffled_glm} satisfies
$(\varepsilon,\delta)$ shuffle differential privacy. The expected regret satisfies
\[
\begin{aligned}
R_T
=
\,O\Bigg(
\log\log T
\Bigg[
&
\frac{d^{3/2}}{\sqrt{\kappa^\ast}}
\sqrt{T\log T} +\frac{
    d\sqrt T\,\sqrt{\log(d^2/\delta)}
    \left(d+\log\!\left(\frac{2dT\log\log T}{\delta}\right)\right)^{1/4}
}{
    \kappa^\ast\sqrt{\varepsilon}
}+
\frac{
    d^{9/8}\sqrt T\,
    \log^{3/8}\!\left(\frac{Td}{\delta}\right)
}{
    \sqrt{\kappa^\ast}\,\varepsilon^{1/4}
}
\\
&\hspace{- 1 em}+
e^{3S}\kappa^{1/3}T^{1/3}
\Bigg(
    d^2(\log T)^{2/3} +
\frac{
    d^{4/3}\log^{2/3}(d^2/\delta)
    \left(d+\log\!\left(\frac{2dT\log\log T}{\delta}\right)\right)^{1/3}
}{
    (\varepsilon\kappa^\ast)^{2/3}
} + \frac{
    d^{3/2}\log^{1/2}\!\left(\frac{Td}{\delta}\right)
}{
    \varepsilon^{1/3}
}
\Bigg)
\Bigg]
\Bigg)\operatorname{poly}(R,S).
\end{aligned}
\]
\end{theorem}

\begin{observation}
    Under the assumption of $T= \Omega(e^d)$, the dominant term in the regret simplifies to $\widetilde{O}\left(\frac{d^{3/2} \sqrt {T \log T}}{\sqrt \kappa^\ast} + \frac{d^{9/8} \sqrt T(\log T)^{3/8}}{\sqrt \kappa^\ast \min(\sqrt \varepsilon,\varepsilon^{1/4})}+ \frac{d^2e^{3S} \kappa^{1/3} T^{1/3} (\log T)^{2/3}}{\min (1,\varepsilon)}\right)$. 
\end{observation}

%\section{Privacy and regret guarantee of \Cref{alg:shuffled_glm}}\label{sec:privacy_regret_ass1}

\section{Regret analysis of Algorithm \ref{alg:shuffled_glm} under \texorpdfstring{\Cref{ass:sco_glm_shuffled} and \Cref{ass:sub_Gaussian_summation}}{assumptions}}\label{sec:proof_theorem1_regret_ass}

\newcommand{\maintheoremoneregretassumptions}{
    If \Cref{ass:sub_Gaussian_summation} and \Cref{ass:sco_glm_shuffled} hold true, and $\lambdal \geq 20dR \log T$, the expected regret $R_T$ can be bounded as follows where $M$ denotes the number of rounds
    \begin{align}
    R_T & = O\Biggl(  \left( \sqrt{\kappa} d^2 e^{3S} RS T^{\frac{1}{2\left(1-2^{-M}\right)} } \left(d R\log T + 4\tilde{\sigma}  \sqrt{\log (2dMT/\delta)} + 8 \tilde{\sigma} \sqrt d+ \sqrt{4RS\nu}\right) \right)^{2/3} \nonumber\\
    & \hspace{ 10 em} + \frac{d}{\sqrt {\kappa^{\ast}}} RS T^{ 
 \frac{1}{2\left(1-2^{-M}\right)}} \sqrt{d R\log T + 4\tilde{\sigma}  \sqrt{\log (2dMT/\delta)}+8 \tilde{\sigma} \sqrt d + \sqrt{4RS\nu}}\Biggr) M\nonumber
\end{align}
}
\begin{theorem}{\label{theorem:main_theoremone_assumptions}}
    \maintheoremoneregretassumptions
\end{theorem}

The proof follows a structure similar to \citet{sawarni2024generalizedlinearbanditslimited} however, it requires some key modifications and changes as the estimates $\hat \theta_k$ and $\tH_k$ and $\tV$ are noisy and thus,  require careful handling. To prove this lemma, we define the following notations below.

\subsection{Some additional notations}\label{sec:additional_notation}

Given \Cref{ass:sub_Gaussian_summation} and \Cref{ass:sco_glm_shuffled}, $\tV$ and $\{\tH_k\}_{k=2}^{M}$ can be written using perturbed regularization matrices $\{\mathcal{N}_k\}_{k=1}^{M}$ as follows.
\begin{align}
\tV = \sum_{t\in\mathcal{B}_1} x_tx_t^T + \mathcal{N}_1 \text{ };  \tH_k = \sum_{t\in\mathcal{B}_k}\frac{\dot{\mu}(\langle x_j, \hat{\theta}_1 \rangle)}{\beta(x_t)}x_tx_t^T + \mathcal{N}_k \text{ for } k \geq 2 
\end{align}

Since $\mathcal{N}_k - \lambda \mI$ is symmetric matrix with upper diagonal elements being independent zero mean sub-Gaussian with variance proxy $\tilde \sigma^2 = O\left(\frac{R^2 (\log \left(d^2/\delta)\right)^2}{(\kappa^{\ast})^2\varepsilon^2}\right)$. Then using Corollary 4.4.8 of \citet{vershynin_high-dimensional_2018}, we can bound the $\ell_2$ norm of $\mathcal{N}_k - \lambda \mI$ with high probability leading to the following observation by union bounding over all batches $k \in [M]$.

\begin{observation}\label{obs:bounds_lambda_min_lambda_max}
With probability at least $1-\frac{1}{T^6}$, for all $k \in [M]$, the smallest and largest eigenvalues of $\mathcal{N}_k$ satisfy
\[
\lambda_{\min}(\mathcal{N}_k) > \lambda - \tilde{\sigma}\left(4\sqrt{d} + 2\sqrt{\log(2dMT/\delta)}\right),
\quad
\lambda_{\max}(\mathcal{N}_k) < \lambda + \tilde{\sigma}\left(4\sqrt{d} + 2\sqrt{\log(2dMT/\delta)}\right).
\]
\end{observation}

Further, with probability at least $1-\frac{1}{T^6}$, let ${\lambda}_{\text{min}}$ and ${\lambda}_{\text{max}}$ denote a lower bound and upper bound on the eigen values of matrices $\mathcal{N}_k$ respectively. \footnote{While this definition overloads \Cref{line:lambda_l_u_defn}, \Cref{obs:bounds_lambda_min_lambda_max} shows that the definition in \Cref{line:lambda_l_u_defn} is a valid eigen value bound and hence suffices.} Thus, with probability at least $1-\frac{1}{T^2}$, we have ${\lambda}_{\text{min}} \preceq \mathcal{N}_k \preceq {\lambda}_{\text{max}} \text{ } \forall k \in [M]$.

Following the definition of $H_k$ and using the true vector $\theta^\ast$ we define
\begin{align}
\tH^\ast_k = \sum_{t\in\mathcal{T}_k}\dot{\mu}(\langle x_t, \theta^\ast \rangle)x_tx_t^T + \mathcal{N}_k.
\end{align}

%We shall show that 

% Since $\mathcal{N}_k - \lambda \mI$ is symmetric matrix with upper diagonal elements being independent zero mean sub-Gaussian with variance proxy $\tilde \sigma^2$. Then using Corollary 4.4.8 of \citet{vershynin_high-dimensional_2018}, we can bound the $\ell_2$ norm of $\mathcal{N}_k - \lambda \mI$ with high probability leading to the following observation by union bounding over all batches $k \in [M]$.

% \begin{observation}{\label{obs:bounds_lambda_min_lambda_max}}
%    We have $\lambdal > \lambda -  \tilde{\sigma} \left( \sqrt{d} + 6 \sqrt{\log (2dMT/\delta)}\right)$ and ${\lambda}_{\text{max}} < \lambda +  \tilde{\sigma} \left( \sqrt{d} + 6 \sqrt{\log (2dMT/\delta)}\right)$
% \end{observation}

%One should note that since this probability is over all the batches, we use union bound.

%We write $c$ to denote absolute constant(s) that appears throughout our analysis. 

Note that $\gamma(\lambdau, \lambdal)$ is a `parameterized' version of $\gamma$ (which was defined in \Cref{sec:glm_alg}) in terms of $\lambdau$ and $\lambdal$
\begin{align}{\label{eq:gamma_parametrized_lambda_min_max}}
\gamma(\lambdau, \lambdal) = 24RS\sqrt{\log(T) + d}  + \frac{R(\log(T) + d)}{\sqrt{\lambdal}} + 2S \sqrt{\lambdau} +\sqrt{4RS \nu}.
\end{align}

 In our proof, we present the arguments using this parameterized version. For certain matrix concentration inequalities to hold, we need $\lambdal$ to be above a certain quantity and thus choose $\lambda$ accordingly to $cRd\log T + \tilde{\sigma}\left(\sqrt{d} + 6 \sqrt{ \log (2dkT/\delta)}\right)$ ensuring with high probability a lower bound on $\lambdal$ 
 
We use $\tilde{x}$ to denote the scaled versions of the arms (see \Cref{line:g_opt_sampling_glm} of \Cref{alg:shuffled_glm}); in particular,
\begin{align}{\label{eq:arm_ind_scaling}}
\tilde{x} := \sqrt{\frac{\dot{\mu}(\langle x, \hat{\theta}_w \rangle)}{\beta(x)}}x 
\end{align}

Furthermore, to capture the non-linearity of the problem, we introduce the term $\phi(\lambdau, \lambdal)$:
\begin{align}
\phi(\lambdau, \lambdal) := \sqrt{\kappa}\frac{e^{3S} \gamma(\lambdau, \lambdal)}{2S} .
\end{align}

\subsection{Concentration inequalities}

\begin{lemma}[Bernstein's inequality]\label{lemma:bernstein}
    Let $X_1,X_2,\ldots,X_n$ be a sequence of independent random variables with $|X_t - \mathbb{E}[X_t]| \leq b$. Also, let sum $S := \sum_{t=1}^{n} (X_t - \mathbb{E}[X_t])$ and let $v = \sum_{t=1}^{n}\text{Var} [X_t]$. Then for any $\delta>0$, we have

    \begin{equation}
        \mathbb{P} \left(S \geq \sqrt{2v \log (1/\delta)} + \frac{2b}{3} \log (1/\delta)\right) \leq \delta
    \end{equation}
\end{lemma}

We further present the following concentration bound from \citet{vershynin_high-dimensional_2018} on the sum of squares of zero-mean subGaussian random variables.

\begin{lemma}[Concentration of squares of sub-Gaussian]\label{lemma:square_subGaussian}
    Let $X_1,X_2,\ldots,X_n$ be zero mean iid subGaussian random variables with variance proxy $\sigma^2$. Then for $n> 64 \log (2/ \delta)$, we have

    \begin{equation}
        \mathbb{P} \left(\sum_{i=1}^{n} (X^2_i - \mathbb{E}[X^2_i]) > \sqrt{\frac{8 \sigma^4 \log (2/ \delta)}{n}}\right) \leq \delta
    \end{equation}
\end{lemma}

We now state Lemma A.2 of \citet{sawarni2024generalizedlinearbanditslimited} adapted for our setup.

% \begin{lemma}[\cite{sawarni2024generalizedlinearbanditslimited}]{\label{lemma:bounding_hat_theta_theta}}
%     Let $\mathcal{X} = \{x_1,\ldots,x_s\} \in \mathbb{R}^d$ be a set of vectors with $||x_t||<1$ for all $t \in [s]$ and two positive scalers $\lambdau$ and $\lambdal$. Also, let $r_1, r_2, \ldots, r_s \in [0, R]$ be independent random variables distributed by the canonical
% exponential family; in particular, $\mathbb{E}[r_s] = \mu(\langle x_s, \theta^\ast \rangle)$ for $\theta^\ast \in \mathbb{R}^d$.
% Further, let $\tilde{\theta} = \arg\min_{\theta} \sum_{s=1}^t \ell(\theta, x_s, r_s)$ be the maximum likelihood
% estimator of $\theta^\ast$ and let matrix $\tH^\ast = \sum_{j=1}^{T} \dot{\mu}(\langle x_j, \theta^\ast \rangle) x_j x_j^T + \mathcal{N}$ where $\lambda_1 \mI\preceq \mathcal{N} \preceq \lambda_2 \mI$
% Then, with probability at least $1 - \frac{1}{T^2}$, the following inequality holds\\

% \begin{equation}
%     || \theta^{\ast} - \tilde{\theta}||_{\tH^{\ast}} \leq 24RS \left(\sqrt{\log T +d} + \frac{R(\log T +d)}{\lambdal}\right) + 2S \sqrt{\lambdau}
% \end{equation}
% \end{lemma}

%Now, consider the setup where $\hat{\theta}$ denotes the output when the dataset $\gD = \{x_s,r_s\}_{t=1}^{s}$ is passed through a shuffled convex optimizer $\mathcal{P}_{GD}$. We have the following lemma.

\newcommand{\lemmaoutputscotrue}{
    Let $\mathcal{X} = \{x_1,\ldots,x_s\} \in \mathbb{R}^d$ be a set of vectors with $||x_t||<1$ for all $t \in [s]$ and two positive scalers $\lambdau$ and $\lambdal$. Also, let $r_1, r_2, \ldots, r_s \in [0, R]$ be independent random variables distributed by the canonical
exponential family; in particular, $\mathbb{E}[r_s] = \mu(\langle x_s, \theta^\ast \rangle)$ for $\theta^\ast \in \mathbb{R}^d$.
Further, let $\hat{\theta}$ denote the output of $\mathcal{P}_{\text{GD}}$ on passing the dataset $S= \{r_s,x_s\}_{s=1}^{t}$ with loss function $\ell(.) + \frac{\lambdau}{2t}||\theta||^2_2$. Set the number of iteration $T= \frac{\varepsilon^2s^2}{d \log (sd/\delta)}$ and parameter $\eta = \frac{2S}{L\sqrt{T}}$. Let matrix $\tH^\ast = \sum_{j=1}^{t} \dot{\mu}(\langle x_j, \theta^\ast \rangle) x_j x_j^T + \mathcal{N}$ where $\lambdal \mI\preceq \mathcal{N} \preceq \lambdau \mI$ with probability at least $\left(1 - \frac{1}{T^6}\right)$
Then, if $\lambdal \geq 20 dR \log T$ with probability at least $1 - \frac{3}{T^2}$, the following inequalities hold\\

\begin{equation}
    || \theta^{\ast} - \hat{\theta}||_{\tH^{\ast}} \leq 24RS \left(\sqrt{\log T +d} + \frac{R(\log T +d)}{\sqrt{\lambdal}}\right) + 2S \sqrt{\lambdau} + \sqrt{4RS\nu} 
\end{equation}

}

\begin{lemma}{\label{lemma:bounding_est_theta_theta}}
    \lemmaoutputscotrue
\end{lemma}

To prove this lemma, we define $\tilde \theta = \argmin_{\theta} \sum_{s=1}^{t} \ell(\theta,x_s,r_s) + \lambda \left|\left| \theta\right|\right|^2$ and state a slight variant of a result in Lemma A.2 of \citet{sawarni2024generalizedlinearbanditslimited}. We consider a setup identical to \Cref{lemma:bounding_est_theta_theta} and define the regularized loss function on data set $\gS=\{(x_s,r_s)\}_{s=1}^{t}$ as $\gL_{\gS} (\theta) = \sum\limits_{s=1}^{t} \ell(\theta,x_s,r_s) + \lambda \left|\left|\theta\right|\right|_2^2$

\begin{lemma}{\label{lemma:bounding_est_theta_theta_inter}}
    With probability at least $\left(1- \frac{1}{T^2}\right)$, we have the following.

    \begin{equation}
    ||g({\tilde\theta}) - g (\theta^{\ast})||_{(\tH^{\ast})^{-1}} \leq \left(\sqrt{\log T +d} + \frac{R(\log T +d)}{\sqrt{\lambdal}}\right) + 2S \sqrt{\lambdau} 
\end{equation}

where, $g(\theta) = \sum_{s=1}^{t} x_s{\mu}(\langle x_s ,\theta \rangle) + \lambdau \langle x_s, \theta\rangle$
\end{lemma}

While this lemma~\ref{lemma:bounding_est_theta_theta_inter} is slightly distinct from Lemma A.2 of \citet{sawarni2024generalizedlinearbanditslimited} in the way that $\tH^\ast - \sum_{j=1}^{t} \dot{\mu}(\langle x_j, \theta^\ast \rangle) x_j x_j^T$ is upper and lower bounded by $\lambdau\mI$ and $\lambdal\mI$ respectively unlike \citet{sawarni2024generalizedlinearbanditslimited} where this difference is a constant $\lambda I$, the proof would go along an identical line. We therefore refrain from reproving here.

We now prove the \Cref{lemma:bounding_est_theta_theta} similar to
Appendix E.3 of
\citet{sawarni2024generalizedlinearbanditslimited}. However, our optimizer $\gP_{GD}$ does not exactly converge to the optimal solution (due to privacy constraint/budget) but the cumulative empirical loss of the the obtained $\hat \theta$ differs from the optimal by at most $\nu$ (\Cref{ass:sub_Gaussian_summation} and \Cref{ass:sco_glm_shuffled}). The proof naturally follows from \Cref{lemma:convex_relaxation_error} and appropriate simplification.

Using this lemma, we can now prove the following corollary.

\begin{corollary}
\label{cor:estimator_warmup_bound}
Let $x_1, x_2, \ldots, x_{\tau}$ be the sequence of arms pulled during the warm-up batch and let $\theta_1$ be the estimator of $\theta^\ast$ computed at the end of the batch. Then, for any vector $x$ and $\lambda \geq 0$ the following bound holds with probability greater than $1-\frac{1}{T^2}$:
\begin{equation}
|\langle x, \theta^\ast - \theta_1 \rangle| \leq \sqrt{\kappa} \|x\|_{\tV^{-1}} \gamma(\lambda)
\end{equation}
\end{corollary}

\begin{proof}
    This corollary is derived from \Cref{lemma:bounding_est_theta_theta} and \eqref{eq:gamma_parametrized_lambda_min_max}. By applying the lemma, we have\footnote{Define $\tH^{\ast}$ as defined in \Cref{lemma:bounding_est_theta_theta}.}

    \begin{equation*}
        |\langle x, \theta^{\ast} - \hat \theta_1 \rangle| \leq ||x||_{(\tH^{\ast})^{-1}} ||\theta^{\ast} - \hat \theta_1||_{(\tH^{\ast})} \leq \sqrt{\kappa} ||x||_{(\tV){-1}} \gamma(\lambdau, \lambdal)
    \end{equation*}

    Considering the definition of $\kappa$, we have $\dot{\mu}(x_s, \theta^{\ast}) \geq \frac{1}{\kappa}$. This further implies that $\tH^{\ast} \succeq \frac{1}{\kappa} \tV$ which in turn implies $||x||_{(\tH^{\ast})^{-1}} \leq \sqrt{\kappa} ||x||_{(\tV)^{-1}}$.
\end{proof}

The following preliminary lemmas from Section A.3 of \citet{sawarni2024generalizedlinearbanditslimited} also naturally hold in our setup as well. We restate them below for convenience.

\begin{lemma}{\label{lemma:H_bound_H_star}}
For each batch \( k \geq 2 \) and the scaled data matrix \( \tH_k \) computed at the end of the batch, the following bound holds with probability at least \( 1 - \frac{1}{T^2} \):

\[
\tH_k \preceq \tH^\ast_k \text{ and } \dot{\mu} (\langle x, \hat \theta_1\rangle) \leq \beta(x) \dot{\mu} (\langle x, \theta^{\ast}\rangle)
\] 
\end{lemma}

\begin{lemma}\label{lemma:batch_number_bound}
    The \Cref{alg:shuffled_glm} runs for at most $\log \log T$ batches.
\end{lemma}

\begin{lemma}
    Let \( \hat{\theta}_k \) be the estimator of \( \theta^\ast \) calculated at the end of the \( k \)-th batch. Then, for any vector \( x \), the following holds with probability greater than \( 1 - \frac{\log \log T}{T^2} \) for every batch \( k \geq 2 \):

\[
\left| \langle x, \theta^\ast - \hat{\theta}_k \rangle \right| \leq \|x\|_{\tH^{-1}_k} \gamma(\lambdau, \lambdal).
\]
\end{lemma}

Using these lemmas, we can now prove the following lemma.

\begin{lemma}{\label{lemma:bouding_diff_optimal}}
    Let \( x \in \gX \) be the selected arm in any round of batch \( k \geq 2 \) in the algorithm, and let \( x^\ast \) be the optimal arm in the arm set \( \gX \), i.e., 
\[
x^\ast = \arg\max_{x \in \gX} \mu(\langle x, \theta^\ast \rangle).
\]
With probability greater than \( 1 - \frac{\log \log T}{T^2} \), the following inequality holds:
\[
\mu(\langle x^\ast, \theta^\ast \rangle) - \mu(\langle x, \theta^\ast \rangle) \leq 6\phi(\lambdau, \lambdal) \sum_{\substack{y \in \{x, x^\ast\} \\ \tilde{y} \in \{x, \tilde{x}^\ast\}}} \|y\|_{\tV^{-1}} \|\tilde{y}\|_{\tH_{k-1}^{-1}} + 2\gamma(\lambdau, \lambdal) \sqrt{\dot{\mu}(\langle x^\ast, \theta^\ast \rangle)} \left( \|\tilde{x}^\ast\|_{\tH_{k-1}^{-1}} + \|\tilde{x}\|_{\tH_{k-1}^{-1}} \right).
\]
\end{lemma}

We now prove the following lemma which is a variant of Lemma A.10 of \citet{sawarni2024generalizedlinearbanditslimited}. However, the upper bound in this lemma in larger by a factor of $\sqrt{d}$ as for every round $k>2$, we use G-optimal design instead of distributional optimal design as we are currently unable to adapt distributional optimal design to our setup. Also let $\gD_k$ denotes the distribution of the survived arms post elimination at \Cref{line:elimination_glm_stochastic} in \Cref{alg:shuffled_glm} in batch $k$.

\begin{lemma}{\label{lemma:bounding_diff_optimal_expectation}}
Consider the setup of the previous \Cref{lemma:bouding_diff_optimal}. During any round in batch \( k \) of Algorithm 1, and for an absolute constant \( c \), we have the following. % with probability at least $1-\frac{4\log \log T}{T^2}$
\[
\underset{}{\mathbb{E}} \left[ \left| \mu \left( \langle x^\ast, \theta^\ast \rangle \right) - \mu \left( \langle x, \theta^\ast \rangle \right) \right| \right] \leq c \left( \frac{\phi(\lambdau, \lambdal) d^2}{\sqrt{B_1 B_{k-1}}} + \frac{\gamma(\lambdau, \lambdal)}{\sqrt{B_{k-1}}} \sqrt{\frac{d^2}{\kappa^\ast}}\right).
\]

%One should note that the randomness is 
\end{lemma}

%We prove the lemma below.

\begin{proof}
The proof invokes \Cref{lemma:bouding_diff_optimal} and we thus have

\begin{align}
\underset{\mathcal{X} \sim \mathcal{D}_k}{\mathbb{E}} \left[ \sum_{\substack{y \in \{x, x^\ast\} \\ \tilde{y} \in \{\tilde{x}, \tilde{x}^\ast\}}} ||y||_{\mathbf{V}^{-1}} ||\tilde{y}||_{\mathbf{H}_{k-1}^{-1}} \right]
&\leq 4 \underset{\mathcal{X} \sim \mathcal{D}_k}{\mathbb{E}} \left[ \max_{x \in \mathcal{X}} \|x\|_{\mathbf{V}^{-1}} \max_{x \in \mathcal{X}} ||\tilde{x}||_{\mathbf{H}_{k-1}^{-1}} \right] \\
&\leq 4 \sqrt{\underset{\mathcal{X} \sim \mathcal{D}_k}{\mathbb{E}} \left[ \max_{x \in \mathcal{X}} \|x\|_{\mathbf{V}^{-1}}^2 \right] \underset{\mathcal{X} \sim \mathcal{D}_k}{\mathbb{E}} \left[ \max_{x \in \mathcal{X}} ||\tilde{x}||_{\mathbf{H}_{k-1}^{-1}}^2 \right]} \quad \text{(via Cauchy-Schwarz inequality)} \\
&\leq 4 \sqrt{\underset{\mathcal{X} \sim \mathcal{D}}{\mathbb{E}} \left[ \max_{x \in \mathcal{X}} \|x\|_{\mathbf{V}^{-1}}^2 \right] \underset{\mathcal{X} \sim \mathcal{D}_{k-1}}{\mathbb{E}} \left[ \max_{x \in \mathcal{X}} \|\tilde{x}\|_{\mathbf{H}_{k-1}^{-1}}^2 \right]} \quad \text{(\Cref{claim:bounding_batch_previous})} \\
&\leq c \sqrt{\frac{d^2}{B_1} \cdot \frac{d^2} {B_{k-1}}} \quad \text{(using via \Cref{lemma:bound_norm_squared_optimal_desing})}
\end{align}
We also have
% \begin{align}
% \mathbb{E}{\mathcal{X} \sim \mathcal{D}_k} \left[ \sqrt{\dot{\mu}(\langle x^, \theta^ \rangle)} \left( |x^{e*}|{H{k-1}^{-1}} + |x^e|{H{k-1}^{-1}} \right) \right] &\leq 2 \mathbb{E}{\mathcal{X} \sim \mathcal{D}k} \left[ \sqrt{\dot{\mu}(\langle x^, \theta^ \rangle)} \max{x \in \mathcal{X}} |x^e|{H_{k-1}^{-1}} \right] \
% &\leq 2 \min \left( \sqrt{\frac{1}{\kappa^} \mathbb{E}{\mathcal{X} \sim \mathcal{D}k} \left[ \max{x \in \mathcal{X}} |x^e|{H_{k-1}^{-1}} \right]}, \sqrt{\mathbb{E}_{\mathcal{X} \sim \mathcal{D}_k}[\dot{\mu}(\langle x^, \theta^\ast \rangle)] \mathbb{E}{\mathcal{X} \sim \mathcal{D}k} \left[ |x^e|{H{k-1}^{-1}}^2 \right]} \right) \
% &\quad \text{(using the definition of $\kappa^$ for the first bound and Jensen for the second)} \
% &\leq c \sqrt{\frac{d \log d}{\kappa^ \tau_{k-1}} \wedge \frac{d^2}{\kappa \tau_{k-1}}} \quad \text{(using Lemma A.17)}
% \end{align}

\begin{align}
\underset{\mathcal{X} \sim \mathcal{D}_k}{\mathbb{E}} \left[ \sqrt{\dot{\mu}(\langle x^\ast, \theta^\ast \rangle)} \left( \|\tilde x^{\ast}\|_{H_{k-1}^{-1}} + \|\tilde x\|_{H_{k-1}^{-1}} \right) \right] &\leq 2 \underset{\mathcal{X} \sim \mathcal{D}_k}{\mathbb{E}} \left[ \sqrt{\dot{\mu}(\langle x^\ast, \theta^\ast \rangle)} \max_{x \in \mathcal{X}} \|\tilde x\|_{H_{k-1}^{-1}} \right] \\
&\leq 2 \left( \sqrt{\underset{\mathcal{X} \sim \mathcal{D}_k}{\mathbb{E}}\left[\dot{\mu}(\langle x^\ast, \theta^\ast \rangle)\right] \underset{\mathcal{X} \sim \mathcal{D}_k}{\mathbb{E}} \left[ \max_{x \in \gX}\|\tilde x\|_{H_{k-1}^{-1}}^2 \right]} \right) \\
&\quad \text{(using Cauchy-Schwartz inequality)} \\
&\leq c \sqrt{\frac{d^2}{\hat \kappa B_{k-1}}} \quad \text{(using \Cref{lemma:bound_norm_squared_optimal_desing})}
\end{align}
Substituting the above bounds in \Cref{lemma:bouding_diff_optimal} we obtained the stated inequality. 

Further, observe that with low probability at most $\frac{1}{T^2}$ (union bounding over all bad events) the bound above could be at most $R$. However it is order-wise insignificant compared to other terms. 
\end{proof}

\section{Proof of Main Theorem \ref{theorem:main_theoremone_assumptions}}{\label{sec:proof_main_theorem_glm_shuffled}}

We now prove the main theorem below and denote the arm selected at a given round by $x_t$ and the optimal arm at given round by $x^{\ast}_t$. 

For the first batch, we trivially bound the regret by $R B_1$ where denotes the bound on the rewards.

For each batch $k \geq 2$,\Cref{lemma:bounding_diff_optimal_expectation} implies 

\begin{align}
    {\mathbb{E}} \left[ \sum_{t \in \mathcal{B}_k} \mu \left( \langle x_t^\ast, \theta^\ast \rangle \right) - \mu \left( \langle x_t, \theta^\ast \rangle \right) \right] \leq & c B_k\left( \frac{\phi(\lambdau, \lambdal) d^2}{\sqrt{B_1 B_{k-1}}} + \frac{\gamma(\lambdau, \lambdal)}{\sqrt{B_{k-1}}} \sqrt{\frac{d^2}{\kappa^\ast}}\right)\\
    \leq & c \alpha\left( \frac{\phi(\lambdau, \lambdal) d^2}{\sqrt{B_1}} + \gamma(\lambdau, \lambdal) \sqrt{\frac{d^2}{\kappa^\ast}}\right)
\end{align}

Since there are $M$ batches, the expected regret can be bounded as 

\begin{equation}
    R_T \leq O\left( B_1R + \frac{\phi(\lambdau, \lambdal)}{\sqrt{B_1}} \alpha + \frac{d\alpha }{\sqrt{\kappa^{\ast}}} \gamma(\lambdau, \lambdal)\right) M
\end{equation}

%The last term follows since with probability at-most $\frac{4(\log \log T)^2}{T^2}\cdot T$ (simple union bound over all batches), there is no bound on the regret and thus we bound it by $cR$.

Now setting $B_1 = \left(\frac{\phi(\lambdal,\lambdau) d^2 \alpha}{R}\right)^{2/3}$, one can bound the regret as follows

\begin{equation}
    R_T \leq O\left(M\left( \left( \frac{\phi(\lambdau, \lambdal) d^2 \alpha}{R}\right)^{2/3} + \frac{d}{\sqrt {\kappa^{\ast}}} \gamma(\lambdau, \lambdal) \alpha\right) \right)
\end{equation}

Now choosing $\lambda = 20dR \log T + \tilde{\sigma} \left( 2 \sqrt{\log (2dMT/\delta)} + 4\sqrt d\right)$, we get 

$$\gamma (\lambdal,\lambdau) \leq \gamma \text{ where } \gamma = 30 RS \sqrt{d R\log T + 2\tilde{\sigma}  (2 \sqrt{\log (2dMT/\delta)}  + 4\sqrt d)}+ \sqrt{4RS \nu} $$ 

Substituting $\alpha = T^{ \frac{1}{2\left(1-2^{-M}\right) }}$ and $\phi (\lambdal, \lambdau) = \frac{\sqrt{\kappa}e^{3S} \gamma^2}{S}$, we get 

\begin{align}
    R_T & = O\Biggl(  \left( \sqrt{\kappa} d^2 e^{3S} RS T^{\frac{1}{2\left(1-2^{-M}\right)} } \left(d R\log T + 4\tilde{\sigma}  \sqrt{\log (2dMT/\delta)} + 8 \tilde{\sigma} \sqrt d+ \sqrt{4RS\nu}\right) \right)^{2/3} \nonumber\\
    & \hspace{ 15 em} + \frac{d}{\sqrt {\kappa^{\ast}}} RS T^{ 
 \frac{1}{2\left(1-2^{-M}\right)}} \sqrt{d R\log T + 4\tilde{\sigma}  \sqrt{\log (2dMT/\delta)} + 8\tilde \sigma \sqrt d + \sqrt{4RS\nu}}\Biggr) M
\end{align}

\section{Properties of Generalized Linear Models}\label{sec:appendix_glm_properties}

% \noindent\textbf{Lemma~\ref{lemma:bounding_est_theta_theta} (Restated):}
% \lemmaoutputscotrue

% % Before, we prove the main result, we firt make the following observation which follows from a slight modification of \cite{}

% % We now prove a very similar corollary to \cite[Corollary A.4]{sawarni2024generalizedlinearbanditslimited}.

% % \begin{corollary}\label{corollary:warmup_batch}
% % Let $x_1, x_2, \ldots, x_{\tau}$ be the sequence of arms pulled during the warm-up batch and let $\hat\theta_1$ be the estimator of $\theta^\ast$ computed at the end of the batch. Then, for any vector $x$ and $\lambda \geq 0$ the following bound holds with probability greater than $1-\frac{1}{T^2}$:
% % \begin{equation}
% % |\langle x, \theta^\ast - \hat\theta_1 \rangle| \leq \sqrt{\kappa} \|x\|_{\tV^{-1}} \gamma(\lambda_{\text{min}}, \lambda_{\text{max}})
% % \end{equation}
% % \end{corollary}

% \section{Useful properties of Generalized Linear Models}

% Recall that a Generalized Linear Model is characterized by a canonical exponential family, i.e., the random variable $r$
% has density function $\rho(r) = \exp (rz - b(z) + c (r))$, with parameter $z$, log-partition function $b(·)$, and a function $c$. Further, $\dot{b}(z) = \mu(z)$ is also called the link function. 

% We will also assume that the random variable $r$ has a bounded non-negative support, i.e., $r \in [0, R]$ almost surely.

We now state the following key lemmas on GLMs from Appendix C of \citet{sawarni2024generalizedlinearbanditslimited}.

\begin{lemma}[Self Concordance for GLMs]\label{lemma:self_concordance_multiplicative}
    For distributions in the exponential family supported on $[0,R]$, the function $\mu(.)$ satisfies $|\ddot{\mu}(z)|\leq R \dot{\mu}(z) \text{ } \forall z \in \mathbb{R}.$
\end{lemma}

% \begin{proposition}[Self-Concordance]{\label{lemma:self_concordance_multiplicative}}
%     For any GLM supported on $[0,R]$ almost surely, the link function $\mu(.)$ satisfies $||\ddot{\mu}(z)|| \leq R\dot{\mu}(z)$ for all $z \in \mathbb{R}$ and $\frac{\dot{\mu}(z_2)}{\exp(R|z_1-z_2|)} \leq \dot{\mu}(z_1) \leq \dot{\mu}(z_2)\exp(R|z_1-z_2|)$ for any $z_1,z_2 \in \mathbb{R}$. 
% \end{proposition}

The following lemma from \citet{sawarni2024generalizedlinearbanditslimited} is also a modification of the self-concordance result in \citet{faury2020improvedoptimisticalgorithmslogistic}.

\begin{lemma}{\label{lemma:bounds_alpha_alpha_tilde}}
For an exponential distribution with log-partition function $b(\cdot)$, for all $z_1, z_2 \in \mathbb{R}$, letting $\mu(z) := \dot{b}(z)$, the following holds:
\begin{align}
\alpha(z_1, z_2) &:= \int^1_{v=0} \dot{\mu}(z_1 + v(z_2 - z_1)) \geq \frac{\dot{\mu}(z)}{1 + R|z_1 - z_2|} \text{ for } z \in \{z_1, z_2\} \tag{25} \\
\dot{\mu}(z_2)e^{-R|z_2-z_1|} &\leq \dot{\mu}(z_1) \leq e^{R|z_2-z_1|}\dot{\mu}(z_2) \tag{26} \\
\tilde{\alpha}(z_1, z_2) &:= \int^1_{v=0} (1-v)\dot{\mu}(z_1 + v(z_2 - z_1))dv \geq \frac{\dot{\mu}(z_1)}{2 + R|z_1 - z_2|}
\end{align}

\end{lemma}

We also state some nice properties of GLMs.

\begin{lemma}[Properties of GLMs~\citealp{sawarni2024generalizedlinearbanditslimited}]
\label{lemma:glm_properties}
For any random variable $r$ that is distributed by a canonical exponential family, we have
\begin{enumerate}
    \item $\mathbb{E}[r] = \mu(z) = \dot{b}(z)$
    \item $\mathbb{V}[r] = \mathbb{E}[(r - \mathbb{E}[r])^2] = \dot{\mu}(z) = \ddot{b}(z)$
    \item $\mathbb{E}[(r - \mathbb{E}[r])^3] = \dddot{b}(z)$
\end{enumerate}
\end{lemma}

The following observation follows that the fact $\dot{b}(.)$ is monotone since $\mathbb{V}[r] = \ddot{b}(z) \geq 0$
\begin{observation}[Convexity of $b(.)$]
    The function $b(.)$ is a convex function for every exponential family of distributions.
\end{observation}

\section{Optimal Design Guarantees}\label{sec:optimal_design_guarantees}

% {\color{red} Major typos check rebuttal very confusing now.}

In this section, we study the optimal design guarantees at different stages of the algorithm. Recall that $\pi_G$ denotes the $G$-optimal design that is used in every batch to select an arm after the elimination process in line \ref{line:elimination_glm_stochastic} of \cref{alg:shuffled_glm}.

Further, recall that the distribution of the survived arms after the elimination process in round $k$ is denoted by $\gD_k$. Further, from the second batch onwards we scale the remaining arm sets $\gX$ to $\tilde{\gX}$ and then apply $G-$ optimal design on it. Now define the expected design matrices as follows.
\begin{equation}\label{eq:W_def}
    \tW_1 := 
    \underset{\mathcal{X} \sim \mathcal{D}}{\mathbb{E}}
        \Bigl[\,
            \underset{x \sim \pi_G(\mathcal{X})}{\mathbb{E}}
            \bigl[ xx^{\top} \mid \mathcal{X} \bigr]
        \Bigr], 
    \quad
    \tW_k := 
    \underset{\mathcal{X} \sim \mathcal{D}_k}{\mathbb{E}}
        \Bigl[\,
            \underset{{x} \sim \pi_G(\tilde{\mathcal{X}})}{\mathbb{E}}
            \bigl[ \tilde{x}\tilde{x}^{\top} \mid \mathcal{X} \bigr]
        \Bigr],\;
    \forall k \geq 2.
\end{equation}

We now state the following claims from Section A.5 of \citet{sawarni2024generalizedlinearbanditslimited}.
\begin{claim}{\label{claim:bounding_batch_previous}}
    The following holds true for any semi-definite matrix $\tA$ and batch $k$.
    $$ \underset{\gX \sim \gD_k} {\mathbb{E}} \max_{x \sim \gX} ||x||^2_{\tA} \leq \underset{\gX \sim \gD_j} {\mathbb{E}} \max_{x \sim \gX} ||x||^2_{\tA} \text { } \forall j \in [k-1]$$
\end{claim}

This claim follows from the fact that the set of surviving arms in batch $k$ is always 
a subset of those that survived in the earlier batches. 
Formally, for any $j < k$, there exists a coupling between $\gD_j$ and $\gD_k$ 
such that if $\gY \sim \gD_j$ and $\gX \sim \gD_k$, then $\gX \subseteq \gY$.

The following claim follows from \citep[lemma 4]{ruan2021linearbanditslimitedadaptivity} since in every batch $k$, we only use $G-$optimal design.

\begin{claim}{\label{claim:data_matrix_bound}}
    $$ \underset{\gX \sim \gD}{\mathbb{E}}\left[\max_{x \in \gX} ||x||^2_{\tW^{-1}_1}\right] \leq d^2 \text{ and } \underset{\gX \sim \gD_k}{\mathbb{E}}\left[\max_{x \in \gX} ||\tilde x||^2_{\tW_k^{-1}}\right] \leq d^2
 $$
\end{claim}

We now state and prove the following bounds in expectation  on $\max_{x \in \gX}||x||^2_{W^{-1}_1}$ and $\max_{x \in \gX}||\tilde{ x} ||^2_{W^{-1}_k}$ for $k \geq 2$.

\begin{lemma}{\label{lemma:bound_norm_squared_optimal_desing}}
    The following bounds simultaneously hold with at least probability $1-\frac{1}{T^3}$.

    \begin{equation}
        \underset{\gX \sim \gD}{ \mathbb{E}}\left[\max_{x \in \gX} || x||^2_{\tV^{-1}}\right] \leq O\left(\frac{d^2}{B_1}\right) \text{ and } \underset{\gX \sim \gD_k}{ \mathbb{E}}\left[\max_{x \in \gX} || \tilde x||^2_{\tH^{-1}_k}\right] \leq O\left(\frac{d^2}{B_k}\right)
    \end{equation}
\end{lemma}

To prove this lemma, we first state the matrix Chernoff bound from \citet{tropp2015introductionmatrixconcentrationinequalities}

\begin{lemma}[Matrix Chernoff bound]{\label{lemma:matrix_chernoff}}
    Let $x_1,x_2,\ldots,x_n \sim \gA$ be vectors in $\mathbb{R}^d$ with $||x_t||_2 < 1$. Then we have
    $$ \mathbb{P} \left[ 3 \varepsilon n\mI + \sum_{i=1}^{n} x_i x^{\top}_i \preceq \frac{n}{8} \underset{x \sim \gA}{\mathbb{E}} \left[xx^{\top}\right]\right] \leq 2d\exp\left(-\frac{\varepsilon n}{8}\right)$$
\end{lemma}

\begin{proof}[Proof of Lemma \ref{lemma:bound_norm_squared_optimal_desing}]
    Since 
    \[
        \lambda - \tilde{\sigma}\!\left(4\sqrt{d} + 2 \sqrt{\log(2dkT/\delta)}\right) > 16 \log (Td).
    \]
    Then, by \cref{obs:bounds_lambda_min_lambda_max}, we obtain
    \[
        \tV \succeq 16 \log (Td)\,\mI + \sum_{i \in \gB_1} x_i x_i^{\top},
    \]
    which provides both lower and upper bounds on the regularizer $\gN_k$. 
    Note that here $x_i \sim \pi_G(\gD)$, i.e., each arm is chosen according to the 
    G-optimal design applied to the distribution $\gD$ over the set of arms. 

    Similarly, for any batch $k$, we have
    \[
        \tH_k \succeq 16 \log (Td)\,\mI + \sum_{i \in \gB_k} \tilde{x}_i \tilde{x}_i^{\top},
    \]
    where $x_i \sim \pi_G(\tilde{\gX})$, with $\tilde{\gX}$ denoting the scaled version 
    of $\gX \sim \gD_k$ as defined in \cref{eq:arm_scaling}, 
    and $\tilde{x}_i$ the corresponding scaled arm (see \cref{eq:arm_ind_scaling}). 

    For batch $k=1$, let $\gA := \pi_G(\gX)$ with $\gX \sim \gD$. 
    For any batch $k>1$, set $\gA := \pi_G(\tilde{\gX})$ with $\gX \sim \gD_k$. 
    Applying \cref{lemma:matrix_chernoff} then yields
    \[
        \tV \succeq \tfrac{1}{8} B_1 \tW_1 
        \quad \text{and} \quad 
        \tH_k \succeq \tfrac{1}{8} B_k \tW_k
    \]
    with probability at least $1 - \tfrac{1}{T^5}$. 

    Finally, for a given sample $\gX_k$ drawn from the distribution of surviving arms $\gD_k$, 
    setting $\gA := \pi_G(\gX_k)$ in \cref{lemma:matrix_chernoff} ensures that 
    $\tV \succeq \tfrac{1}{8} B_1 \tW_1$ holds with high probability. 
    The claim then follows by applying \cref{claim:data_matrix_bound}.
\end{proof}

\section{Regret and Privacy Analysis of Algorithm \ref{alg:jdp_glm_adversarial}}\label{sec:jdp_instantiation_utility_regret_tradeoff}

\subsection{Parameter settings}\label{sec:params_jdp_glm_adversarial}

We collect here the explicit parameter settings for \Cref{alg:jdp_glm_adversarial}. Define
\[
    \Lambda_T := \log(T/\zeta),
    \qquad
    \Lambda_\delta := \log(1/\delta).
\]The regularization parameter $\lambda$, the confidence-width scaling $\beta$ (for Criterion II), and the confidence-width scaling $\gamma$ (for Criterion I) are set as

% We define where $\Lambda_T := \log(T/\zeta) \text{ and } \Lambda_\delta := \log(1/\delta)$. 
\begin{align}
\lambda &:= \frac{5}{4} d \Lambda_T + \frac{100 \sqrt d {\Lambda_\delta \log T} \left(\sqrt{\Lambda_T} + \sqrt{d}\right)}{\varepsilon \kappa^{\ast}}, \label{eq:lambda_defn_glm_adv} \\
\beta &=C\left[R\sqrt{d\Lambda_T} +\frac{RS\,d^{1/4}\Lambda_T^{1/2}\Lambda_\delta^{1/2}
        \bigl(d+\Lambda_T\bigr)^{1/4}
    }{
        \sqrt{\varepsilon\min(\kappa^\ast,1)}
    }
\right] 
\label{eq:beta_defn_glm_adv}\\
\gamma
&=
C\left[
    R^2S\sqrt{d\Lambda_T}
    +
    \frac{
        R^2S\,d^{1/4}\Lambda_T^{1/2}
        \bigl(d+\Lambda_T\bigr)^{1/4}
        \Lambda_\delta^{1/2}
    }{
        \sqrt{\varepsilon\kappa^\ast}
    }
    +
    \frac{
        R^4S^4 d\sqrt{\kappa}\,\Lambda_T\Lambda_\delta^{1/2}
    }{
        \varepsilon
    }
\right]. \label{eq:gamma_defn_glm_adv}
\end{align}

 for some sufficiently large constant $C$. 
These choices account for the binary-tree noise, the privacy cost of switching, and the DP-SGD optimizer error; they ensure the confidence bounds remain valid under joint differential privacy.

%\beta &= CR \sqrt{d\log (16T/\zeta)} + \frac{R^{2.5}S \sqrt {d \log (1/\delta) \log T}}{\sqrt{\varepsilon\kappa^{\ast}}}+ \frac{CR^{2.5}S^2 d^{0.25} (\log (T/\zeta))^{3/4}}{\sqrt{\varepsilon}}, \label{eq:beta_defn_glm_adv} \\

\subsection{Assumptions}

We now present two assumptions that our \Cref{alg:jdp_glm_adversarial} satisfies with respect to design matrices $\tV$ and $\{\tH_k\}_{k=2}^{M}$ and the parameter estimate $\hat \theta_\tau$ at the end of each policy switch which would be necessary to prove the regret bounds.
\begin{assumption}{\label{ass:regularizer_noise}}
        With probability at least $(1- \zeta/8)$, the following holds true for every time step $t\in [T]$.
        $$\lambdal \mI \preceq \tV - \sum_{s \in \gT_o} x_s x^{\top}_s \preceq \lambdau \mI \text{ and } \lambdal \mI \preceq \tH_t - \sum\limits_{s \in [t-1]\setminus \gT_o} \frac{\dot{\mu}\left(\langle x_s, \hat \theta_o\rangle\right)}{e}x_s x^{\top}_s \preceq \lambdau\mI$$  
    \end{assumption}

    \begin{assumption}{\label{ass:sco_loss}}
        With probability at least $(1-\zeta/8)$, there exist constants $\nu_1,\nu_2$ such that the DP-SGD optimizer minimizes the criterion-specific regularized losses up to errors $\nu_1$ and $\nu_2$ under the first and second policy switches, respectively. Formally,

        \begin{equation}
            \left|\left| \gL^{(1)}_{\gT_o}(\hat \theta;\lambda) - \min_{\theta \in \Theta} \gL^{(1)}_{\gT_o}(\theta;\lambda) \right|\right| \leq \nu_1 \text{ and } \left|\left| \gL^{(2)}_{[t-1]\setminus\gT_o}(\hat \theta;\lambda) - \min_{\theta \in \Theta} \gL^{(2)}_{[t-1]\setminus\gT_o}(\theta;\lambda) \right|\right| \leq \nu_2
        \end{equation}

        where for any index set $\gT \subseteq [T]$,
        \[
            \gL^{(i)}_{\gT}(\theta;\lambda) := \sum_{s \in \gT} \ell(\theta,x_s,r_s) + \frac{\lambda}{2}\|\theta\|_2^2, \quad i \in \{1,2\}.
        \]

    \end{assumption}

    We first verify that our algorithm satisfies the assumptions stated above 
with suitable values of $\nu_1$ and $\nu_2$, while ensuring 
$(\varepsilon,\delta)$-JDP as shown in 
\Cref{sec:privacy_analysis_regret_glm_adversarial}. 
A key step in this analysis is proving that the switching steps 
$\{\gT_o \setminus \{t\}\}$ remain $(\varepsilon,\delta)$-indistinguishable 
between any two context–action datasets differing only at time step $t$ 
(\Cref{lemma:utility_privacy_glm_adv_2}); a detailed proof is provided in 
\Cref{sec:utility_privacy_switching_I_II}. 
Next, we establish that the design matrices $\{\tH_t\}$ and $\tV$ are 
differentially private with respect to the action vectors $x_t$ on the 
chosen indices $\gT_o$ and $[T-1]\setminus\gT_o$ respectively, via the binary tree 
mechanism \cref{lemma:utility_privacy_glm_adv_1}. A similar analysis is also done for estimators $\hat \theta_o, \hat \theta_\tau$ in \cref{lemma:utility_privacy_glm_adv_3}. The final privacy bounds follows from a simple composition theorem. Finally, in 
\Cref{sec:glm_jdp_regret_analysis_assumption}, we bound the regret of 
\Cref{alg:jdp_glm_adversarial} under these assumptions, thereby completing 
the joint privacy and regret analysis of our algorithm.

\begin{remark}[Privatization of policy switches]
   One might ask why we must privatize policy~I switches but not those of policy~II. A policy~I switch reads the fresh context set $\gX_t$ at time $t$ (\cref{line:policy_I_switch}) and thus touches raw data, whereas a policy~II switch is computed solely from already privatized summaries—e.g., the noisy design matrices $\tH_t$ (\cref{line:policy_II_switch})—and therefore enjoys privacy by post-processing. This distinction necessitates a separate privacy analysis for policy~I switches.

   Another natural question is why we do not use the sparse vector technique Section~3.6 of \citet{dworkdpbook} to privatize policy~I switches. We avoided this approach because the resulting utility tradeoffs are worse: in particular, the error tolerance in Theorem~3.26 of \citet{dworkdpbook} grows as $\sqrt{c}$, where $c$ is the number of queries above the threshold, which adversely affects the regret bounds. 
\end{remark}

\subsection{Privacy Analysis of Algorithm \ref{alg:jdp_glm_adversarial} }{\label{sec:privacy_analysis_regret_glm_adversarial}}

    In this section, we aim to show that our algorithm is $(\varepsilon,\delta)$ joint differentially private and satisfy the assumptions mentioned above.

   % For the privacy proofs, we first show that conditional on the switching times $\gT_o$, the design matrices $\tV, \tH_t$ and the estimators $\hat \theta_o, \hat \theta_\tau$ are differentially private in \cref{lemma:utility_privacy_glm_adv_1} and \cref{lemma:utility_privacy_glm_adv_3}. We finally show that for any datasets over reward, context pair differing at index $i$, the timestep $\gT_o\setminus \{i\}$ is $(\varepsilon/3,\delta/3)$ indistinguishable. Applying, a simple composition theorem we complete the proof. 

    % \newcommand{\lemmadputilityone}{Let $\gT_{o,t}$ denote the value of $\gT_o$ in \Cref{alg:jdp_glm_adversarial} at time step $t$. Now set $\lambdal \geq \frac{8\gamma^2\kappa R^2}{\varepsilon}\log (32T/\zeta)) \sqrt{128\ccut \log (8/\delta)}$ and $\ccut =  8dR^2 \kappa \gamma^2 \log T$. If \Cref{ass:regularizer_noise} holds true, then the sequence of sets $\{\gT_{o,t}\}_{t=1}^{T}$ is $(\varepsilon/4,\delta/4)$ differentially private w.r.t $\{r_t,\gX_t\}_{t=1}^{T}$ and satisfies \Cref{ass:indices_selected}. }

    \newcommand{\lemmadputilityone}{Let $\tV_t$ denote the value of $\tV$ at time step $t$ and let $\gT_{o,t}$ denote the value of $\gT_o$ in \Cref{alg:jdp_glm_adversarial} at time step $t$. The estimators $\{\tH_t, \tV_t\}_{t=1}^{T}$ are $(\varepsilon/3,\delta/6 + \zeta/6)$ differentially private with respect to $\{r_t,\gX_t\}_{t \in \gT_{o,T}}$ and $\{r_t,\gX_t\}_{t \in [T-1]\setminus \gT_{o,T}}$ respectively and it satisfies \Cref{ass:regularizer_noise}.}

    \newcommand{\lemmadptwo}{Consider two datasets $\gD$ and $\gD'$ over reward, context set pairs which differ only at index $i$. Let $\gT_{o,t}$ denote the value of the set of $\gT_o$ at time step $t$. Then the sets $\{\gT_{o,t}\setminus \{i\}\}_{t=1}^{T}$ when computed over datasets $\gD$ and $\gD'$ are $(\varepsilon/3, \delta/6 + \zeta/6)$ indistinguishable. Moreover, Criterion~I and Criterion~II are triggered at most $\ccut$ and $\ccuttwo$ times, respectively.
    }

\newcommand{\lemmadputilitythree}{
Let $\hat{\theta}_{o,t}$ denote the value of estimator $\hat \theta_o$ at time step $t$ and $\hat{\theta}_{\tau,t}$ denote the value of estimator $\hat \theta_\tau$ at time step $t$. Assume that $R$ and $S$ are bounded below by positive absolute constants. Then the set of distinct estimators in $\{\hat{\theta}_{o,t}\}_{t=1}^{T}$ and $\{\hat{\theta}_{\tau,t}\}_{t=1}^{T}$ are each $(\varepsilon/3,\delta/3)$-differentially private with respect to $\{r_t,\gX_t\}_{t\in \gT_o}$ and $\{r_t,\gX_t\}_{t\in [T]\setminus \gT_o}$ respectively, when $\delta \geq \zeta$. Moreover, the estimators $\hat \theta_o$ and $\hat \theta_{\tau}$ satisfy \Cref{ass:sco_loss} with
\begin{align}
    \nu_1
    &=
    O\left(
        \frac{R^2S^2\gamma}{\varepsilon}
        \sqrt{d\kappa\log T\log(8/\delta)}
        \sqrt{d\log T}
    \right), \\
    \nu_2
    &=
    O\left(
        \frac{RS^2}{\varepsilon}
        \sqrt{
            \log(8/\delta)
            \log_2\left(1+\frac{TR^3}{d}\right)
        }
        \sqrt{d+\log T}
    \right).
\end{align}
}

%Suppose the number of distinct optimizer calls for $\hat\theta_o$ is bounded by $\ccut = 8dR^2\kappa\gamma^2\log T$, and the number of distinct optimizer calls for $\hat\theta_\tau$ is bounded by $\ccuttwo = \log_2\!\left(\frac{\lambdau}{\lambdal} + \frac{T/(\kappa^{\ast})^2}{\lambdal d}\right)$, where $\lambdau = \frac{4\sqrt d + 2}{4\sqrt d + 1}\lambda$ and $\lambdal = \frac{4\sqrt d}{4\sqrt d + 1}\lambda$. 

    \begin{lemma}{\label{lemma:utility_privacy_glm_adv_1}}
        \lemmadputilityone        
    \end{lemma}

  \textbf{Tree-based mechanism \cite{continualstatsrelease,continualobservation}}:  It is used to compute differentially private prefix sums of a sequence (data-set) $\gD = \{x_1,x_2,\ldots,x_n\}$. To do this, we first construct a balanced binary tree $T$ with $\{x_1,x_2.\ldots,x_n\}$ as the leaf nodes and each internal node $y$ in tree $T$ stores the sum of all leaf nodes rooted at $y$. First notice that one can compute any partial sum $\sum_{j=1}^{i} x_j$ using at most $m := \lceil\log(n) + 1\rceil$ nodes of $T$. Second, notice that for any two neighboring data sequences $D$ and $D'$ the partial sums stored at no more than $m$ nodes in $T$ are different. Thus, if the count in each node preserves $(\varepsilon_0, \delta_0)$-differential privacy, using the advanced composition of \citet{dworkdpbook} we get that the entire algorithm is $(O(m\varepsilon_0^2 + \varepsilon_0\sqrt{2m\ln(1/\delta')}), m\delta_0 + \delta')$-differentially private. Alternatively, one can ensure that the entire tree is $(\varepsilon,\delta)$-differentially private via standard advanced composition by setting
$\varepsilon_0 = \varepsilon/\sqrt{8m\ln(2/\delta)}$ and $\delta_0 = \delta/(2m)$, with the composition slack parameter $\delta'=\delta/2$.

However, this choice is somewhat conservative and is best suited to the small-$\varepsilon$ regime, where advanced composition gives its sharpest guarantees. Since we do not assume $\varepsilon \leq 1$, we instead use privacy accounting via zero-concentrated differential privacy (zCDP) \cite{cryptoeprint:concentratedDP}, which is particularly convenient for Gaussian mechanisms. Specifically, $\rho$-zCDP implies $(\rho + 2\sqrt{\rho \log(1/\delta)}, \delta)$-DP for every $\delta > 0$. Therefore a sufficient condition is that the full tree be $\varepsilon^2/(144 \log(3/\delta))$-zCDP. Because zCDP composes additively, it is enough to ensure that each node is $\varepsilon^2/(144 m \log(3/\delta))$-zCDP.

  %  Before, we prove this lemma 

  We now present how we instantiate the binary tree to compute the noise term $\gR_t$ at every step $t$ (\Cref{line:noise_addition_R_t_step_I} and \Cref{line:noise_addition_R_t_step_II}). The construction is almost identical to Section 4.2 of \citet{shariff2018differentiallyprivatecontextuallinear}.

  %\footnote{While we do not apriori know the set $\gT_{o}$ at later time steps, a bound on $\gT_o$ from \Cref{lemma:utility_privacy_glm_adv_2} should suffice.} 

    \begin{proof}
        We instantiate the above tree-based algorithm with symmetric Gaussian noise at every node: sample $Z' \in \mathbb{R}^{d\times d}$ with each $Z'_{i,j} \sim N(0, \sigma^2_{\text{noise}})$ i.i.d. and symmetrize to get $Z = (Z' + Z'^{\top})/2$ which we add to every node. 
        Further, we construct two separate trees for computing the summation $\sum_{i \in \gT_{o,t}} x_i x^{\top}_i$ and $\sum_{i \in [t] \setminus \gT_{o,t}} \frac{\dot{\mu} (x_i, \theta_o)}{e}x_i x^{\top}_i$ for every time step $t\in [T]$ and $\gT_{o,t}$ denotes the value of set $\gT_o$ at a time-step $t$. 
        The $j^{th}$ leaf node at the first tree stores $x_jx^{\top}_j$ if $j \in \gT_o$ and 0 otherwise, while the $j^{th}$ leaf node at the second tree stores $\frac{\dot{\mu}(x_t, \hat \theta_o)}{e}x_jx^{\top}_j$ if $j \in \gT_o$ and 0 otherwise. 
        Further, \Cref{lemma:anytime_bound_theta_o_modulus} implies that a single data point modification can affect a term with the difference's operator norm bounded by $\frac{1}{\kappa^{\ast}}$ under event $\gE_o$ (event in \Cref{lemma:anytime_bound_theta_o}). It follows that in order to make sure each node in the tree-based algorithm preserves $\frac{\varepsilon^2}{144m\log (3/\delta)}$-zero-Concentrated DP \footnote{Here we use better bounds for the composition of Gaussian noise based on zero-Concentrated DP \cite{cryptoeprint:concentratedDP}.}, the variance in each coordinate must be $\sigma^2_{\text{noise}} = \frac{72 m (\log(32/\delta))^2}{(\kappa^{\ast})^2\varepsilon^2}$.
        When all entries on $Z$ are sampled from $\mathcal{N}(0, 1)$, known concentration results \cite{vershynin_high-dimensional_2018} on the top singular value of $Z$ give that $\mathbb{P}[\|Z\| > (4\sqrt{d} + 2\ln(2T/\alpha))] < \alpha/2T$. For every time $t \in \gT_{o}$, the cumulative noise matrix $\sum_{i \in \gT_{o}} \gR_i$ is a sum of at most $m$ such matrices, thus the variance of each coordinate is $m \sigma^2_{\text{noise}}$. A similar argument can also be presented for time steps $t \notin \gT_{o,t}$. 

        %Thus, we can argue that for every

        For every time step $t \in [T]$, we can argue with probability at least $1- \zeta/8$.

        \begin{equation}
            \left|\left|\tV_t - \lambda \mI - \sum_{s \in \gT_o} x_s x^{\top}_s\right|\right| \leq \sigma_{\text{noise}} \sqrt{2m} \left(4\sqrt{d} + \sqrt{2\ln(2n/\alpha)}\right) \leq \sqrt{72} \frac{\log(32/\delta) \log T}{\varepsilon \kappa^{\ast}}\left(4\sqrt{d} + 2 \sqrt{\log(16T/\zeta)}\right)
        \end{equation}

        Similarly, 

        \begin{equation}
             \left|\left|\tH_t - \lambda \mI - \sum_{s \in [t-1]\setminus \gT_o} \frac{\dot{\mu}\left(\langle x_s, \hat \theta_o\rangle\right)}{e}  x_s x^{\top}_s\right|\right| \leq \sqrt{72} \frac{\log(32/\delta) \log T}{\varepsilon \kappa^{\ast}}\left(4\sqrt{d} + 2\sqrt{\log(16T/\zeta)}\right)
        \end{equation}

        %Further, since $T > e^d$, we can bound both the expressions by $\frac{64 \log(4/\delta) \left(\log(16 T/\zeta)\right)^2}{\varepsilon (\kappa^{\ast})}$.

        Using the definition of $\lambda$ from \cref{eq:lambda_defn_glm_adv}, the equation above proves that \Cref{ass:regularizer_noise} is satisfied with $\lambda_{\text{min}} = \frac{4 \sqrt d}{4 \sqrt d+1} \lambda$
        and $\lambda_{\text{max}} = \frac{4 \sqrt d+2}{4 \sqrt d+1} \lambda$.
        
        We argued above that conditioning under event $\gE_o$, $\{\tH_t\}_{t=1}^{T}$ and $\{\tV_t\}_{t=1}^{T}$ is $(\varepsilon/4, \delta/8)$ differentially with respect to $\{r_t, \gX_t\}_{t\in \gT_{o,T}}$ and $\{r_t, \gX_t\}_{t\in [T]\setminus \gT_{o,T}}$. Thus, the estimators are $(\varepsilon/3, \delta/6 + \zeta/6)$ differentially private as the probability of the event $\gE_o$ is at least $1-\zeta/8$ from \Cref{lemma:anytime_bound_theta_o}.
    \end{proof}

    % \sss{Potential major issue: $\lambdau$ scales as $\gamma^{2.5}$ where as $\gamma$ only scales as $\sqrt{\lambdau}$. This raises unsatisfiability issue. Possible fix use running summation of $\sum_{t \in \gT_o}||x_t||_{\tV^{-1}}^2 + \max_{x \in \gX_t} ||x ||^2_{\tV^{-1}} \geq \frac{(|\gT_o|+1)}{\gamma^2 \kappa R^2}$ but not sure. won't work  tried some methods like using different $\varepsilon$ for every count region and then using advanced compoistion from https://proceedings.mlr.press/v37/kairouz15.pdf

    % but again $\lambdau > \gamma^2 \log T$ which is contradicting defintion of $\gamma$ in \eqref{eq:gamma_parametrized_lambda_min_max}.
    % }

    % \sss{Other approach pre-multiply by $\det(\tV)$. Though sparse vector theorem, directly won't apply here we can still reprove it for our case. }
    
    \begin{lemma}{\label{lemma:utility_privacy_glm_adv_2}}        \lemmadptwo
    \end{lemma}

        We prove this lemma in \Cref{sec:utility_privacy_switching_I_II}. The proof involves expanding the log likelihood ratios as a summation with each term measuring the log likehood conditioned on the actions at previous steps conditioning on the action at previous steps similar to the techniques in \cite{kairouz2015compositiontheoremdifferentialprivacy,fullyadaptiveprivacy}. The $\ccuttwo$ bound on the number of invocations of Criterion~II follows from \Cref{lemma:policy_switch_bound_general}, while the $\ccut$ bound for Criterion~I follows from \Cref{lemma:utility_privacy_glm_adv_2_policyswitch} by applying the elliptic potential lemma, \Cref{lemma:elliptic_potential}.

    The following lemma establishes privacy and utility guarantees for the DP-SGD optimizer used under Criterion~I and Criterion~II in \Cref{alg:jdp_glm_adversarial}. %While our proof uses the shuffle convex optimizer $\mathcal{P}_{\text{GD}}$, analogous rates can be obtained for generic DP-SGD optimizers~\citep{Bassily2014PrivateERM}.

    \begin{lemma}{\label{lemma:utility_privacy_glm_adv_3}}
        \lemmadputilitythree
    \end{lemma}

\begin{proof}
We prove the claim for $\{\hat{\theta}_{o,t}\}_{t=1}^{T}$; the proof for
$\{\hat{\theta}_{\tau,t}\}_{t=1}^{T}$ is identical with $\ccut$ replaced by
$\ccuttwo$ and the corresponding loss parameters.

Let
\[
    \rho_\star=\frac{\varepsilon^2}{144\log(8/\delta)}.
\]
By the standard zCDP-to-DP conversion, a $\rho$-zCDP mechanism is
\[
    \left(
        \rho+2\sqrt{\rho\log(1/\delta')},
        \delta'
    \right)\text{-DP}.
\]
Taking $\delta'=\delta/3$, the choice of $\rho_\star$ is a conservative
sufficient choice for $(\varepsilon/3,\delta/3)$-DP. Indeed, in the usual
privacy regime,
\[
    \rho_\star
    \leq
    \left(
        \sqrt{\log(3/\delta)+\varepsilon/3}
        -
        \sqrt{\log(3/\delta)}
    \right)^2.
\]
Hence any $\rho_\star$-zCDP mechanism is
$(\varepsilon/3,\delta/3)$-differentially private.

Now consider $\{\hat{\theta}_{o,t}\}_{t=1}^{T}$. By the switching lemma
\Cref{lemma:utility_privacy_glm_adv_2_policyswitch}, the number of distinct
optimizer calls used to construct the distinct values in
$\{\hat{\theta}_{o,t}\}_{t=1}^{T}$ is at most
\[
    \ccut = 8dR^2\kappa\gamma^2\log T .
\]
We allocate total zCDP budget $\rho_\star$ to this collection, so each optimizer
call is run with zCDP budget $\rho_\star/\ccut$. By adaptive composition of
zCDP, the collection of all distinct optimizer calls used to construct
$\{\hat{\theta}_{o,t}\}_{t=1}^{T}$ is $\rho_\star$-zCDP. Therefore, by the
zCDP-to-DP conversion above, the set of distinct estimators in
$\{\hat{\theta}_{o,t}\}_{t=1}^{T}$ is $(\varepsilon/3,\delta/3)$-differentially
private with respect to $\{r_t,\gX_t\}_{t\in \gT_o}$.

For utility, we invoke \Cref{lemma:noisy_full_batch_pgd_zcdp} with zCDP
parameter $\rho_\star/\ccut$ for each optimizer call and set the confidence
parameter to $\zeta=1/T^6$. Here the SCO error is measured for the cumulative
objective appearing in \Cref{ass:sco_loss}, rather than the averaged empirical
objective in \Cref{lemma:noisy_full_batch_pgd_zcdp}; consequently, the
normalizing $1/n$ factor from the empirical-average statement is not present in
the value of $\nu_1$ below. The loss defining $\hat\theta_o$ is convex and
$(R+RS)$-Lipschitz over $\{\theta:\|\theta\|\leq S\}$, whose diameter is
$O(S)$. Since $R$ and $S$ are bounded below by positive absolute constants,
\[
    S(R+RS)=O(RS^2).
\]
Thus \Cref{lemma:noisy_full_batch_pgd_zcdp} gives
\[
    \nu_1
    =
    O\left(
        RS^2
        \sqrt{\frac{\ccut}{\rho_\star}}
        \sqrt{d+\log T}
    \right).
\]
Substituting $\rho_\star=\varepsilon^2/(144\log(8/\delta))$ and
$\ccut=8dR^2\kappa\gamma^2\log T$ yields
\[
    \nu_1
    =
    O\left(
        \frac{R^2S^2\gamma}{\varepsilon}
        \sqrt{d\kappa\log T\log(8/\delta)}
        \sqrt{d+\log T}
    \right).
\]
Finally, using $d+\log T \leq O(d\log T)$ for $d\geq 1$ and $T\geq e$, we get
\[
    \nu_1
    =
    O\left(
        \frac{R^2S^2\gamma}{\varepsilon}
        \sqrt{d\kappa\log T\log(8/\delta)}
        \sqrt{d\log T}
    \right).
\]

We now apply the same argument to
$\{\hat{\theta}_{\tau,t}\}_{t=1}^{T}$. The number of distinct optimizer calls is
at most
\[
    \ccuttwo
    =
    4\log_2\left(
        1+\frac{TR^3}{d}
    \right).
\]
We again allocate total zCDP budget $\rho_\star$ to this collection, so each
optimizer call is run with zCDP budget $\rho_\star/\ccuttwo$. By adaptive
composition of zCDP, the collection of all distinct optimizer calls used to
construct $\{\hat{\theta}_{\tau,t}\}_{t=1}^{T}$ is $\rho_\star$-zCDP. Thus, by
the same zCDP-to-DP conversion, the set of distinct estimators in
$\{\hat{\theta}_{\tau,t}\}_{t=1}^{T}$ is $(\varepsilon/3,\delta/3)$-differentially
private with respect to $\{r_t,\gX_t\}_{t\in [T]\setminus \gT_o}$.

For the $\hat\theta_\tau$ optimizer, the corresponding diameter-times-Lipschitz
quantity in the SCO error bound is $O(RS^2)$. Invoking
\Cref{lemma:noisy_full_batch_pgd_zcdp} with zCDP parameter
$\rho_\star/\ccuttwo$ and confidence parameter $\zeta=1/T^6$ gives
\[
    \nu_2
    =
    O\left(
        RS^2
        \sqrt{\frac{\ccuttwo}{\rho_\star}}
        \sqrt{d+\log T}
    \right).
\]
Substituting $\rho_\star=\varepsilon^2/(144\log(8/\delta))$ and
$\ccuttwo=4\log_2(1+TR^3/d)$ gives
\[
    \nu_2
    =
    O\left(
        \frac{RS^2}{\varepsilon}
        \sqrt{
            \log(8/\delta)
            \log_2\left(1+\frac{TR^3}{d}\right)
        }
        \sqrt{d+\log T}
    \right).
\]

Combining the privacy and utility guarantees for the two estimator sequences
proves the lemma.
\end{proof}

       % By applying \Cref{lemma:utility_shuffle_convex} with the parameters of the optimizer $\gP_{\text{GD}}$ specialized to our setting, we obtain that the estimator $\hat \theta_{o,t}$ is $\Bigl(\tfrac{\varepsilon}{12 \sqrt{2\ccut \log (4/\delta)}}, \tfrac{\delta}{12 \ccut}\Bigr)$-differentially private at every time step $t$, while the estimator $\hat \theta_{\tau,t}$ is $\Bigl(\tfrac{\varepsilon}{12 \sqrt{2\ccuttwo \log (4/\delta)}}, \tfrac{\delta}{12 \ccuttwo}\Bigr)$-differentially private at every time step $t$.

       %  However, there can only be at most $\ccut$ distinct estimators $\hat \theta_{o,t}$ and $\ccuttwo$ distinct estimators $\hat \theta_{\tau,t}$ due to \cref{lemma:utility_privacy_glm_adv_2_policyswitch}. Thus, the desired privacy guarantee on $\{\theta_{o,t},\theta_{\tau,t}\}_{t \in [T]}$ follows on application of advanced composition theorem \citet[Theorem 3.20]{dworkdpbook} as $\varepsilon/3 \leq 3$

       %  Furthermore, \Cref{ass:sco_loss} can be satisfied with $\nu_1 = O\left(\frac{RS^2 \sqrt{\ccut \log (4/\delta)}}{\varepsilon} \sqrt{ d \log^2 (Td/\delta)}\right)$ and $\nu_2 = O\left(\frac{RS^2 \sqrt{\ccuttwo \log (4/\zeta)}}{\varepsilon}\sqrt{ d \log^2 (Td/\zeta)}\right)$ on application of \Cref{lemma:utility_shuffle_convex} assuming $\delta \geq \zeta$.

    We now argue below that \Cref{alg:jdp_glm_adversarial} is $(\varepsilon,\delta)$ jointly differentially private.
    \begin{lemma}
        The \Cref{alg:jdp_glm_adversarial} is $(\varepsilon,\delta)$ joint differentially private if $\delta>\zeta$
    \end{lemma}

    \begin{proof}
        Consider two datasets $\gD,\gD'$ over user–context pairs that differ at the $i^{\text{th}}$ time step. Let $\tV_t$ denote the value of $\tV$ in \Cref{alg:jdp_glm_adversarial} at time step $t$. Conditioning on the event that $\gT_o \setminus \{i\}$ is identical across the two datasets, we can argue that $\{\tV_t, \tH_t, \hat \theta_{o,t}, \hat \theta_{\tau,t}\}_{t=1}^T$ is $(\varepsilon/3 + \varepsilon/3,; \delta/3 + \delta/3) = (2\varepsilon/3, 2\delta/3)$-indistinguishable when computed on $\gD$ and $\gD'$ from \cref{lemma:utility_privacy_glm_adv_3} and \cref{lemma:utility_privacy_glm_adv_1}. \footnote{While \Cref{lemma:utility_privacy_glm_adv_3} guarantees privacy only with respect to the distinct values of the estimators $\theta_{o,t}$ and $\hat \theta_{\tau,t}$, the update points of $\hat \theta_{\tau,t}$ also depend on $\tH_t$ (via policy II switches), thereby extending the guarantee to joint indistinguishability of the entire tuple from simple composition.} This follows because the estimators ${\tV_t}$ and $\tH_t$ are based on disjoint data, while ${\hat \theta_{o,t}, \hat \theta_{\tau,t}}$ are also computed on disjoint subsets. Moreover, from \Cref{lemma:utility_privacy_glm_adv_2}, the sets $\gT_o \setminus \{i\}$ computed on $\gD$ and $\gD'$ are $(\varepsilon/3,\delta/3)$-indistinguishable. By basic composition, this implies that the entire collection of estimators along with policy I switches is jointly $(\varepsilon,\delta)$-indistinguishable. Since the action taken at each step is a function of these estimators, the joint differential privacy guarantees follows.

        %From \Cref{lemma:utility_privacy_glm_adv_1} and \Cref{lemma:utility_privacy_glm_adv_3}, $\{\tV_t,\tH_t,\hat \theta_{o,t}, \hat \theta_{\tau,t}\}_{t=1}^{T}$ is $(3\varepsilon/8,3\delta/8)$ differentially private w.r.t $\{r_t, \gX_t\}_{t \in \gT_{o,T}}$. This implies that $\{\tV_t,\tH_t,\hat \theta_{o,t}, \hat \theta_{\tau,t}\}_{t=1}^{T}$ is $(3\varepsilon/8*2 + \varepsilon/4,3\delta/8*2 +\delta/4) = (\varepsilon,\delta)$ differentially private w.r.t $\{r_t,\gX_t\}_{t=1}^{T}$. Now applying \cite[Proposition 2.1]{dworkdpbook}, we get the desired privacy bound.
%     
    \end{proof}

    Now an application of \Cref{theorem:glm_joint_dp_regret_assumptions}, we prove the following regret and privacy bound.

    \newcommand{\maintheoremjointprivacy}{ 
    
    Assume $\delta \leq \min(1/e, d/R^3)$ and $d \leq T$. \Cref{alg:jdp_glm_adversarial} satisfies $(\varepsilon,\delta)$ joint differential privacy and with probability at least $1-\zeta$ (for $\delta > \zeta$), the regret $R^T$ may be bounded as 

    \begin{align}
        R^T & = O\left(\left[Rd\sqrt{\Lambda_T} \sqrt{\frac{T}{\kappa^\ast}} +\frac{RS\,d^{3/4}\Lambda_T^{1/2}\sqrt T\Lambda_\delta^{1/2}
        \bigl(d+\Lambda_T\bigr)^{1/4}
    }{
        \sqrt{\varepsilon\kappa^\ast\min(\kappa^\ast,1)}
    }
\right] \sqrt{\log (RT)}\right)  \nonumber\\
        & \hspace{ 5 em} + O\left(\kappa d R^3\left[ R^2S\sqrt{d\Lambda_T} +\frac{R^2S\,d^{1/4}\Lambda_T^{1/2}
        \bigl(d+\Lambda_T\bigr)^{1/4}
        \Lambda_\delta^{1/2}}{\sqrt{\varepsilon\kappa^\ast}} + \frac{R^4S^4 d\sqrt{\kappa}\,\Lambda_T\Lambda_\delta^{1/2}}{\varepsilon} \right]^2 \log T\right)
    \end{align}

    where $\Lambda_T = \log (T/\zeta)$ and $\Lambda_\delta = \log (1/\delta)$
    
    }

    % \sss{Setting $\gamma=cDR \log T....$ ensures that $\gamma(\lambdau) \leq \gamma$ and thus desired property holds.}
    \begin{theorem}{\label{thm:jdp_regret_glm}}
        \maintheoremjointprivacy
    \end{theorem}

    \begin{observation}
        When $T= \Omega(e^d)$, the dominant term in the regret scales as $\widetilde{O}\left(d \sqrt{T/\kappa^{\ast}} \log T+\frac{d^{3/4} \sqrt{T}}{\sqrt{\varepsilon \kappa^\ast \min(\kappa^\ast,1)}} (\log T)^{5/4}\right)$. The latter term in the regret can be thought of as the additive cost for privacy. While we do not have a matching lower bound, even for the simpler case of linear contextual bandits we do have a similar additive term for privacy see \citet{shariff2018differentiallyprivatecontextuallinear,pmlr-v247-azize24a}.

        When $T= O(e^d)$, the dominant term in the regret scales as $\widetilde{O}\left(d \sqrt{T/\kappa^{\ast}} \log T+\frac{d \sqrt{T}}{\sqrt{\varepsilon \kappa^\ast \min(\kappa^\ast,1)}} (\log T)\right)$. 
    \end{observation}

    \begin{proof}

% For notational convenience, define
% \[
%     \Lambda_T := \log(T/\zeta),
%     \qquad
%     \Lambda_\delta := \log(1/\delta).
% \]
% Assume that $R$ and $S$ are bounded below by positive absolute constants,
% $T\geq d$, $\zeta\leq\delta$, and
% \[
%     \delta \leq \min\left\{\frac{d}{R^3},\frac1e\right\}.
% \]
Since $\delta \leq \min\left\{\frac{d}{R^3},\frac1e\right\}$, we have  $\Lambda_\delta\geq 1$, and $\log\left(1+\frac{TR^3}{d}\right) \leq O(\Lambda_T)$.
Indeed, $\delta\leq d/R^3$ implies $TR^3/d \leq T/\delta$, and since
$\zeta\leq\delta$, we have $\log(T/\delta)\leq \log(T/\zeta)=\Lambda_T$.

% We choose
% \begin{align}
% \lambda
% &:=
% \frac{5}{4}d\log(T/\zeta)
% +
% \frac{
% 64(4\sqrt d+1)\log(32/\delta)\log T
% \left(\sqrt{\log(16T/\zeta)}+\sqrt d\right)
% }{
% \varepsilon\kappa^\ast
% }, \label{eq:lambda_defn_glm_adv} \\
% \beta
% &:=
% C\left[
%     R\sqrt{d\Lambda_T}
%     +
%     \frac{
%         RS\,d^{1/4}\Lambda_T^{1/2}\Lambda_\delta^{1/2}
%         (d+\Lambda_T)^{1/4}
%     }{
%         \sqrt{\varepsilon\min(\kappa^\ast,1)}
%     }
% \right], \label{eq:beta_defn_glm_adv} \\
% \gamma
% &:=
% C\left[
%     R^2S\sqrt{d\Lambda_T}
%     +
%     \frac{
%         R^2S\,d^{1/4}\Lambda_T^{1/2}
%         (d+\Lambda_T)^{1/4}
%         \Lambda_\delta^{1/2}
%     }{
%         \sqrt{\varepsilon\kappa^\ast}
%     }
%     +
%     \frac{
%         R^4S^4d\sqrt{\kappa}\,\Lambda_T\Lambda_\delta^{1/2}
%     }{
%         \varepsilon
%     }
% \right], \label{eq:gamma_defn_glm_adv}
% \end{align}
% for a sufficiently large universal constant $C$.

Recall that
\[
    \lambdau
    =
    \frac{4\sqrt d+2}{4\sqrt d+1}\lambda,
    \qquad
    \lambdal
    =
    \frac{4\sqrt d}{4\sqrt d+1}\lambda.
\]
Thus $\lambdau/\lambdal=O(1)$. From the definition of $\lambda$, and absorbing
absolute constants inside the $O(\cdot)$ notation, we have
\[
    \sqrt{\lambdau}
    \leq
    C\left[
        \sqrt{d\Lambda_T}
        +
        \frac{
            d^{1/4}\Lambda_T^{1/2}(d+\Lambda_T)^{1/4}
            \Lambda_\delta^{1/2}
        }{
            \sqrt{\varepsilon\kappa^\ast}
        }
    \right].
\]
Here we used $\log T\leq \Lambda_T$, $\log(32/\delta)=O(\Lambda_\delta)$,
$\log(16T/\zeta)=O(\Lambda_T)$, $4\sqrt d+1=O(\sqrt d)$, and
\[
    \sqrt{\Lambda_T}+\sqrt d
    =
    O\left(\sqrt{d+\Lambda_T}\right).
\]
Moreover, the first term in the definition of $\lambda$ gives
\[
    \lambdal \geq c\,d\Lambda_T
\]
for a universal constant $c>0$.

We first verify $\gamma(\lambdau)\leq \gamma$. By
\Cref{eq:parametric_form_gamma_glm_adv},
\[
\begin{aligned}
    \gamma(\lambdau)
    \leq\;&
    4R^2S\sqrt{\lambdau}
    +
    \frac{7Sd}{\sqrt{\lambdau}}
    \log\left(
        \frac{16T}{\zeta R^2d\lambdau}
    \right)
    +
    4RS\sqrt{\nu_1}.
\end{aligned}
\]
Using $\lambdau/\lambdal=O(1)$, $\lambdal\geq c d\Lambda_T$, and the lower
bounds on $R,S$, the first two terms are bounded by
\[
    C R^2S\sqrt{d\Lambda_T}
    +
    C
    \frac{
        R^2S\,d^{1/4}\Lambda_T^{1/2}(d+\Lambda_T)^{1/4}
        \Lambda_\delta^{1/2}
    }{
        \sqrt{\varepsilon\kappa^\ast}
    }.
\]
From the updated optimizer guarantee in \Cref{lemma:utility_privacy_glm_adv_3},
\[
    \nu_1
    =
    O\left(
        \frac{R^2S^2\gamma}{\varepsilon}
        \sqrt{d\kappa\log T\log(8/\delta)}
        \sqrt{d\log T}
    \right).
\]
Hence, using $\log T\leq \Lambda_T$ and $\log(8/\delta)=O(\Lambda_\delta)$,
\[
    4RS\sqrt{\nu_1}
    \leq
    C
    \frac{
        R^2S^2
        d^{1/2}\kappa^{1/4}\Lambda_T^{1/2}
        \Lambda_\delta^{1/4}
    }{
        \sqrt{\varepsilon}
    }
    \sqrt{\gamma}.
\]
Therefore,
\[
    \gamma(\lambdau)
    \leq
    A_\gamma+B_\gamma\sqrt{\gamma},
\]
where
\[
    A_\gamma
    =
    C\left[
        R^2S\sqrt{d\Lambda_T}
        +
        \frac{
            R^2S\,d^{1/4}\Lambda_T^{1/2}
            (d+\Lambda_T)^{1/4}
            \Lambda_\delta^{1/2}
        }{
            \sqrt{\varepsilon\kappa^\ast}
        }
    \right]
\]
and
\[
    B_\gamma
    =
    C
    \frac{
        R^2S^2
        d^{1/2}\kappa^{1/4}\Lambda_T^{1/2}
        \Lambda_\delta^{1/4}
    }{
        \sqrt{\varepsilon}
    }.
\]
Choosing $\gamma\geq C(A_\gamma+B_\gamma^2)$ for a sufficiently large constant
$C$ ensures $\gamma(\lambdau)\leq\gamma$. This is exactly the choice in
\eqref{eq:gamma_defn_glm_adv}, since
\[
    B_\gamma^2
    =
    O\left(
        \frac{
            R^4S^4d\sqrt{\kappa}\,\Lambda_T\Lambda_\delta^{1/2}
        }{
            \varepsilon
        }
    \right).
\]

We now verify $\beta(\lambdau,\lambdal)\leq \beta$. By
\Cref{eq:beta_parametrized_glm_adv},
\[
\begin{aligned}
    \beta(\lambdau,\lambdal)
    =
    4R\frac{\lambdau}{\sqrt{\lambdal}}
    +
    \frac{7d}{R\sqrt{\lambdal}}
    \log\left(
        \frac{16T}{\zeta R^2d\lambdal}
    \right)
    +
    10\sqrt{\nu_2}.
\end{aligned}
\]
Using $\lambdau/\lambdal=O(1)$, $\lambdal\geq c d\Lambda_T$, and the bound on
$\sqrt{\lambdau}$, the first two terms are at most
\[
    C R\sqrt{d\Lambda_T}
    +
    C
    \frac{
        R\,d^{1/4}\Lambda_T^{1/2}(d+\Lambda_T)^{1/4}
        \Lambda_\delta^{1/2}
    }{
        \sqrt{\varepsilon\kappa^\ast}
    }.
\]
The optimizer error gives
\[
    \nu_2
    =
    O\left(
        \frac{RS^2}{\varepsilon}
        \sqrt{
            \log(8/\delta)
            \log\left(1+\frac{TR^3}{d}\right)
        }
        \sqrt{d+\log T}
    \right).
\]
Using $\log(1+TR^3/d)\leq O(\Lambda_T)$ and $\log(8/\delta)=O(\Lambda_\delta)$,
we get
\[
    10\sqrt{\nu_2}
    \leq
    C
    S\sqrt{\frac{R}{\varepsilon}}
    \Lambda_\delta^{1/4}
    \Lambda_T^{1/4}
    (d+\Lambda_T)^{1/4}.
\]
Since $R,S$ are bounded below by positive absolute constants,
$\Lambda_\delta\geq 1$, and $\min(\kappa^\ast,1)\leq 1$, both the second term
above and the contribution from $\sqrt{\nu_2}$ are bounded by
\[
    C
    \frac{
        RS\,d^{1/4}\Lambda_T^{1/2}\Lambda_\delta^{1/2}
        (d+\Lambda_T)^{1/4}
    }{
        \sqrt{\varepsilon\min(\kappa^\ast,1)}
    }.
\]
Therefore
\[
    \beta(\lambdau,\lambdal)
    \leq
    C\left[
        R\sqrt{d\Lambda_T}
        +
        \frac{
            RS\,d^{1/4}\Lambda_T^{1/2}\Lambda_\delta^{1/2}
            (d+\Lambda_T)^{1/4}
        }{
            \sqrt{\varepsilon\min(\kappa^\ast,1)}
        }
    \right],
\]
which is the stated choice of $\beta$.

%+ O\left(\frac{RS d^{0.25} \log (T/\zeta) \log (1/\delta)}{\sqrt{\varepsilon \kappa^{\ast}}}\right)

    % The first inequality follows from the fact $\lambdau/\lambdal \leq 6/4$ 
    % Thus, the choice of $\beta$ in \Cref{alg:jdp_glm_adversarial} yields $\beta(\lambdau,\lambdal) \leq \beta$.

    Further, we show in \Cref{lemma:policy_switch_bound_general} that Criterion~I and Criterion~II in \Cref{alg:jdp_glm_adversarial} are invoked at most $\ccut$ and $\ccuttwo$ times, respectively. We can now directly instantiate \Cref{theorem:glm_joint_dp_regret_assumptions} for our case to prove the desired regret bound.
    \end{proof}

    \section{Regret analysis of Algorithm \ref{alg:jdp_glm_adversarial} under  assumptions \ref{ass:regularizer_noise} and \ref{ass:sco_loss}}{\label{sec:glm_jdp_regret_analysis_assumption}}

    \newcommand{\maintheoremjointprivacyassumptions}{

    When the events in \Cref{ass:regularizer_noise} and \Cref{ass:sco_loss} hold, Criterion~I and Criterion~II are invoked at most $\ccut$ and $\ccuttwo$ times, respectively, and the parametric conditions $\gamma(\lambdau) \leq \gamma$ and $\beta(\lambdau,\lambdal) \leq \beta$ (as defined in \Cref{eq:beta_parametrized_glm_adv} and \Cref{eq:parametric_form_gamma_glm_adv}) are satisfied, the regret $R_T$ can be bounded by some universal constant $C$ with probability at least $1-\zeta$.

\begin{align}
R_T \leq & C\sqrt {d} \beta\sqrt{\log_2\left(\frac{RT}{d\lambdal}\right) } \sqrt{ \sum_{t \in [T]\setminus \gT_o}\dot{\mu} \left( \langle x^\ast_t, \theta^\ast\rangle\right)}  + 8dR^3 \kappa \gamma^2 \log T
\end{align}
    %$$$$
    
    }

    % \noindent\textbf{\Cref{theorem:glm_joint_dp_regret_assumptions} (Restated):} \maintheoremjointprivacyassumptions

    \begin{theorem}{\label{theorem:glm_joint_dp_regret_assumptions}}
        \maintheoremjointprivacyassumptions
    \end{theorem}

    The proof proceeds in several steps. 
We first establish confidence bounds under both switching and 
non-switching cases in 
\Cref{sec:confidence_bounds_switching_criterion_I,sec:confidence_bounds_non_switching_criterion_I}. 
We then use these bounds to control the instantaneous regret 
(\Cref{sec:confidence_bounds_switching_criterion_I}), and finally 
combine all components to prove the main theorem in 
\Cref{sec:proof_main_theorem_jdp_regret}. 

A key technical ingredient is the policy switching criterion, 
which looks at the largest determinant among the noisy design 
matrices observed so far. To analyze this, we prove 
\Cref{lemma:policy_switch_bound_general}, which bounds the total 
number of switching steps. This result is subsequently applied in the instantaneous regret 
analysis (\Cref{lemma:diff_selected_optimal_H}).

    We first define the (scaled) data matrix. 

    \begin{equation}
        \tH^{\ast}_o = \sum_{s \in \gT_o} \left(\dot{\mu}\left(\langle x_s, \theta^{\ast}\rangle\right) x_s x^{\top}_s \right) + \lambdau\mI
    \end{equation}

    and observe that $\kappa \tH^{\ast}_o \succeq \tV$.
    
    We further define a parametrized form of $\gamma$ below and will show in \Cref{alg:jdp_glm_adversarial} that $\gamma(\lambdau) \leq \gamma$ with $\gamma$.

    \begin{align}{\label{eq:parametric_form_gamma_glm_adv}}
        \gamma(\lambdau) = 4R^2S \sqrt{\lambdau} + \frac{7Sd}{\sqrt{\lambdau}} \log \left(\frac{16T}{\zeta R^2 d\lambdau} \right) + 4RS\sqrt{\nu_1}
    \end{align}

    In \Cref{thm:jdp_regret_glm}, we shall show that for the chosen value of $\gamma$ and $\lambda$ in \Cref{alg:jdp_glm_adversarial}, this parametrized form is upper bounded by $\gamma$. 

    We now state the main concentration lemma from Theorem 1 of \citet{faury2020improvedoptimisticalgorithmslogistic}.

    \newcommand{\lemmaanytimebound}{
    Let $\{\gF_t\}_{t=1}^{\infty}$ be a filtration. Let $\{x_t\}_{t=1}^{\infty}$ be a stochastic process in $\gB_2(d)$ such that $x_t$ is $\gF_t$ measurable. Let $\{\eta_t\}_{t=1}^{\infty}$ be a martingale difference sequence such that $\eta_t$ is $\gF_t$ measurable. Furthermore, assume that $|\eta_t|<1$ almost surely and denote $\sigma^2_t = \text{Var}(\eta_t\mid \gF_{t-1})$. Let $\lambda>0$ and for any $t \geq 1$ define:

        \begin{equation*}
            S_t = \sum_{s=1}^{t-1} \eta_s x_s \text{    } \tH_t = \sum_{s=1}^{t-1} \sigma_s^2 x_s x^{\top}_s + \lambda\mI 
        \end{equation*}

        Then for any $\delta \in (0,1]$,

        $$\mathbb{P}\left(\exists t \geq 1: ||S_t||_{\tH^{-1}_t} \geq \frac{\sqrt{\lambda}}{2} + \frac{2}{\sqrt{\lambda}} \log \left(\frac{\text{det}(\tH_t)^{1/2} } {\lambda^{d/2} \delta}\right) + \frac{2}{\sqrt{\lambda}} d \log(2)\right) \leq \delta$$

}

    \begin{lemma}{\label{lemma:anytime_bound}}
        \lemmaanytimebound
    \end{lemma}

    We now state the following concentration bound which is a variant of Lemma B.2 of \citet{sawarni2024generalizedlinearbanditslimited}. %A proof is provided in Appendix ?? 

    \section{Confidence Bounds under switching criterion I bounds}\label{sec:confidence_bounds_switching_criterion_I}
    
    %which further invokes \cite[Theorem 1]{faury2020improvedoptimisticalgorithmslogistic} in its proof.

    \newcommand{\anytimeboundthetanot}{

    In any round $t$, let $\hat \theta_o$ be calculated using the set of rewards in rounds $\gT_o$. Then with probability $1- \frac{\zeta}{8}$, we have 

        $$||\hat \theta_o - \theta^{\ast}||_{\tH^{\ast}_o} \leq \gamma(\lambdau)$$

    }

    \begin{lemma}{\label{lemma:anytime_bound_theta_o}}
        \anytimeboundthetanot
    \end{lemma}

   We now define $\tilde \theta_o$ as follows.

\begin{equation}
    \tilde \theta_o = \argmin_{\theta} \sum_{t \in \gT_o}\ell(\theta,r_t,x_t) + \frac{\lambdau}{2} || \theta||^2_2 \text{ and } g_o(\theta) = \sum_{s \in \gT_o}{\mu} (\langle \theta, x_s\rangle)x_s + \lambdau \theta
\end{equation}

To prove this lemma, we first state the following lemma which is a variant of Lemma B.2 of \citet{sawarni2024generalizedlinearbanditslimited}. We refrain from proving this lemma again as it follows the proof of Lemma B.2 of \citet{sawarni2024generalizedlinearbanditslimited}.

% Let $$g(\theta) = \sum_{s \in \gT_o}\dot{\mu} (\langle \theta, x_s\rangle)x_s + \lambda_{max} \hat{\theta}_o$$

\begin{lemma}{\label{lemma:anytime_bound_theta_o_inter}}
    At any round $t$, let $\tilde \theta_o$ denote the minimiser computed above. Then with probability at least $(1-\zeta/16)$, we have 
    
    % $$ || \tilde{\theta}_o - \theta^{\ast}||_{\tH^{\ast}_o} \leq  4R^2S \sqrt{\lambdau} + \frac{7Sd}{\sqrt{\lambdau}} \log \left(4/\zeta\left(1+ \frac{T}{R^2 d\lambdau}\right) \right)$$ and 

    $$ || g_o(\tilde{\theta}_o) - g_o(\theta^{\ast})||_{\left(\tH^{\ast}_o\right)^{-1}} \leq  4R \sqrt{\lambdau} + \frac{7d}{R\sqrt{\lambdau}} \log \left(16/\zeta\left(1+ \frac{T}{R^2 d\lambdau}\right) \right)$$
\end{lemma}

Using \Cref{ass:sco_loss} and \Cref{lemma:anytime_bound_theta_o_inter} and instantiating \Cref{lemma:convex_relaxation_error} (with $c=2RS$), we can prove \Cref{lemma:anytime_bound_theta_o}.

 % we can prove lemma \ref{lemma:anytime_bound_theta_o}. The approach is identical to the proof in lemma \ref{lemma:bounding_est_theta_theta} and thus, we refrain from proving it here.

    \section{Confidence Bounds under non-Switching Criterion I rounds}{\label{sec:confidence_bounds_non_switching_criterion_I}}

    Let $\gE_o$ denote the event in \Cref{lemma:anytime_bound_theta_o}. We now prove the following lemma from Lemma B.3 of \citet{sawarni2024generalizedlinearbanditslimited} in our setup.

    \begin{lemma}{\label{lemma:anytime_bound_theta_o_modulus}}
        If in round $t$, the switching criterion I is not satisfied and event $\gE_o$ holds, we have 
        $\left|\langle x, \hat \theta_o - \theta^{\ast}\rangle\right| \leq \frac{1}{R}$ for all $x \in \gX_t$ with probability at least $1-\zeta/8$.
    \end{lemma}

    \begin{proof}
        \begin{align}
            \left|\langle x, \hat \theta_o - \theta^{\ast}\rangle\right| & \leq \left|\left|x\right|\right|_{\left(\tH^{\ast}_o\right)^{-1}} \left|\left|\hat \theta_o - \theta^{\ast}\right|\right|_{\tH^{\ast}_o} \quad\text{(Cauchy-Schwartz)}\\
            & \leq \left|\left|x\right|\right|_{\left(\tH^{\ast}_o\right)^{-1}} \gamma(\lambdau) \quad\text{(via \Cref{lemma:anytime_bound_theta_o})}\\
            & \leq \left|\left|x\right|\right|_{\tV^{-1}} \gamma(\lambdau) \sqrt{\kappa}\quad\text{(as $\tV \preceq \kappa \tH^{\ast}_o$)}\\
             & \leq \frac{1}{\sqrt{\kappa R^2(\gamma)^2}} \gamma(\lambdau) \sqrt{\kappa}\quad\text{(switching criterion I not satisfied}\\
             & \leq \frac{1}{R} \quad \text{ as $\gamma(\lambdau) \leq \gamma$.}
        \end{align}
    \end{proof}

    Further, define $$\hstarut = \sum\limits_{s \in [t-1]\setminus \gT_o} \dot{\mu}\left(\langle x_s, \theta^{\ast}\rangle\right) x_s x^{\top}_s +\lambdau \mI\text{ and } \hstarlt = \sum\limits_{s \in [t-1]\setminus \gT_o} \dot{\mu}\left(\langle x_s, \theta^{\ast}\rangle\right) x_s x^{\top}_s +\lambdal\mI$$

    We can prove the following corollary using self-concordance property and the lemma above. A similar proof is presented in Corollary B.4 of \citet{sawarni2024generalizedlinearbanditslimited}
    \begin{corollary}{\label{corollary:bounded_H_H_star_constant}}
        Under event $\gE_o$, $\tH_t \preceq \hstarut$ and  $\hstarlt \preceq e^2 \tH_t$.
    \end{corollary}

    \begin{proof}
        For a given $s \in [t-1]$, let $\hat \theta_o^s$ denote the value of $\hat \theta_o$ at that round. Then applying \Cref{lemma:anytime_bound_theta_o_modulus} and \Cref{lemma:self_concordance_multiplicative}, we have the following for every $x_s \in \gX_s$

        \begin{equation}{\label{eq:temp_eqn_bound_mu_dot}}
            e^{-1} \dot{\mu}(\langle x_s, \hat \theta^s_o\rangle) \leq \dot{\mu}(\langle x_s,  \theta^\ast\rangle) \leq e \dot{\mu}(\langle x_s, \hat \theta^s_o\rangle)
        \end{equation}

        Now, conclude the proof from \Cref{ass:regularizer_noise}.
    \end{proof}

    Recall that $\tau$ is the round when switching criterion II (\Cref{line:policy_II_switch}) is satisfied. Now define the following equations.

    \begin{align}
        g_{\tau} (\theta) & = \sum_{s \in [\tau -1]\setminus \gT_o} \mu\left(\langle x_s, \theta\rangle\right)x_s + \lambdau \theta\\
        \Theta & = \{\theta: || \theta- \hat \theta_o|| \leq \gamma\sqrt{\kappa} \}\\
        \tilde{\theta}_\tau &= \argmin_{\theta} \sum_{s \in [t-1]\setminus \gT_o} \ell(\theta,x_s,r_s) + \frac{\lambda}{2} ||\theta||^2 \\
        \beta(\lambdau, \lambdal) & = 4R \frac{\lambdau}{\sqrt{\lambdal}} + \frac{7d}{R\sqrt{\lambdal}} \log \left(\frac{16 T}{\zeta R^2 d\lambdal}\right) + 10\sqrt{\nu_2} \label{eq:beta_parametrized_glm_adv}
    \end{align}

    The last equation denotes the parametrized form of $\beta$ and we shall show in \Cref{thm:jdp_regret_glm} that $\beta(\lambdau,\lambdal) \leq \beta$ where $\beta$ is chosen from \Cref{alg:jdp_glm_adversarial}.

    One can now prove the following lemma identically to Lemma B.5 of \citet{sawarni2024generalizedlinearbanditslimited}.

    \begin{lemma}{\label{lemma:anytime_bound_theta_tau_V}}
        Under event $\gE_o$, $$ \left|\left| \hat \theta_\tau - \theta^{\ast}\right|\right|_{\tV} \leq 2\gamma \sqrt{\kappa} \text{ and } \left|\left| \hat \theta_o - \theta^{\ast}\right|\right|_{\tV} \leq \gamma(\lambdau) \sqrt{\kappa}$$
    \end{lemma}

    \begin{proof}

    First, observe that 
    
    \begin{equation}{\label{eq:lemma:anytime_bound_theta_tau_V_local}}
        \left|\left| \hat \theta_o - \theta^{\ast}\right|\right|_{\tV} \leq \sqrt \kappa\left|\left| \hat \theta_o - \theta^{\ast}\right|\right|_{\tH^{\ast}_o} \leq \sqrt \kappa \gamma(\lambdau)
    \end{equation}

    The first inequality follows from the fact that $\tV \preceq \kappa \tH^{\ast}_o$. And second inequality follows from \Cref{lemma:anytime_bound_theta_o}.

        \begin{align}
            \left|\left| \hat \theta_\tau - \theta^{\ast}\right|\right|_{\tV} \leq & \left|\left| \hat \theta_\tau - \theta_o\right|\right|_{\tV} + \left|\left| \theta^{\ast} - \theta_o\right|\right|_{\tV} \\
            \leq & 2\sqrt{\kappa} \gamma \quad \text{(from \eqref{eq:lemma:anytime_bound_theta_tau_V_local} and $\hat\theta_\tau \in \Theta$)}
        \end{align}

    \end{proof}

    We now prove the following lemma.

    \begin{lemma}{\label{lemma:anytime_bound_theta_tau}}
        Under event $\gE_o$ with probability at least $(1-\zeta/8)$, we have $\left|\left|\hat \theta_\tau - \theta^{\ast}\right|\right|_{\hstarutau} \leq \beta (\lambdau, \lambdal)$ and $\left|\left|\hat \theta_\tau - \theta^{\ast}\right|\right|_{\hstarltau} \leq \beta (\lambdau, \lambdal)$
    \end{lemma}

    We first define the following equations that globally minimize the regualarized loss over $\{r_s,x_s\}_{s\in [t-1]\setminus \gT_o}$.

\begin{equation}
    \tilde \theta_\tau = \argmin_{\theta} \sum_{s \in [t-1]\setminus \gT_o}\ell(\theta,r_s,x_s) + \frac{\lambdau}{2} || \theta||^2_2 \text{ and } g_\tau(\theta) = \sum_{s \in [t-1]\setminus \gT_o}{\mu} (\langle \theta, x_s\rangle)x_s + \lambda_{max} \theta
\end{equation}

Before we prove this lemma, we state the following lemma. The proof of this lemma is identical to Lemma B.2 of \citet{sawarni2024generalizedlinearbanditslimited} and we refrain from reproving it here.

\begin{lemma}{\label{lemma:anytime_bound_theta_o_inter_tau}}
    At any round $t$, consider the minimisers computed above. Then under event $\gE_o$ with probability at least $(1-\zeta/4)$, we have 
    
    % $$ || \tilde{\theta}_o - \theta^{\ast}||_{\hstarutau} \leq  4R^2S \sqrt{\lambdau} + \frac{7Sd}{\sqrt{\lambdau}} \log \left(16/\zeta\left(1+ \frac{T}{R^2 d\lambdau}\right) \right)$$ and 

    $$ \left|\left| g_\tau(\tilde{\theta}_\tau) - g_\tau(\theta^{\ast})\right|\right|_{\left(\hstarutau\right)^{-1}} \leq  4R \sqrt{\lambdau} + \frac{7d}{R\sqrt{\lambdal}} \log \left(16/\zeta\left(1+ \frac{T}{R^2 d\lambdal}\right) \right)$$

    % $$ || \tilde{\theta}_o - \theta^{\ast}||_{\hstarltau} \leq  4R^2S \frac{\lambdau}{\sqrt{\lambda_{\text{min}}}} + \frac{7Sd}{\sqrt{\lambda_{\text{min}}}} \log \left(16/\zeta\left(1+ \frac{T}{R^2 d\lambda_{\text{min}}}\right) \right)$$ and 

    $$ \left|\left| g_\tau(\tilde{\theta}_\tau) - g_\tau(\theta^{\ast})\right|\right|_{\left(\hstarltau\right)^{-1}} \leq  4R \frac{\lambdau}{\sqrt{\lambdal}} + \frac{7d}{R\sqrt{\lambdal}} \log \left(16/\zeta\left(1+ \frac{T}{R^2 d\lambdal}\right) \right)$$

\end{lemma}

    \begin{proof}[Proof of \cref{lemma:anytime_bound_theta_tau}] For all rounds $s \in [\tau-1]\setminus\gT_o$, we have

        \begin{align}
            \left|\langle x, \hat \theta_\tau - \theta^{\ast}\rangle\right| & \leq \left|\left|x\right|\right|_{\tV^{-1}} \left|\left|\hat \theta_\tau - \theta^{\ast}\right|\right|_{\tV} \quad\text{(Cauchy-Schwartz)}\\
            & \leq 2\left|\left|x\right|\right|_{\tV^{-1}} \gamma(\lambdau) \sqrt{\kappa}\quad\text{(via \Cref{lemma:anytime_bound_theta_tau_V})}\\
            & \leq \frac{2}{\sqrt{R^2\kappa \gamma^2}} \gamma(\lambdau) \sqrt{\kappa}\quad\text{(switching criterion I not satisfied)}\\
             & \leq \frac{2}{R}
        \end{align}

Using \Cref{lemma:anytime_bound_theta_tau_V}—which guarantees that $\theta^{\ast}$ lies in the convex set $\Theta := \{\theta: ||\theta - \hat \theta_o||_{\tV} \leq \gamma\sqrt \kappa$—and using \cref{lemma:anytime_bound_theta_o_inter_tau}, we instantiate \Cref{lemma:convex_relaxation_error} with $c=2$ for our setting, thereby establishing \Cref{lemma:anytime_bound_theta_tau}.
        
    \end{proof}

\section{Bounding the instantaneous regret at each round $t$}{\label{sec:bounding_instantaneous_regret}}

Let the event in \Cref{lemma:anytime_bound_theta_tau} be denoted by $\gE_\tau$ and now only consider rounds $t \in [T]$ which does not satisfy Switching Criterion $I$. Let $x_t$ denote the arm played at each step in \Cref{line:arm_selected_regular}. Let $x^{\ast}_t$ denote the best arm with the highest reward available at that round.

\begin{corollary}{\label{corollary:theta_tau_diff_x_dot}}
    Under the event $\gE_{\tau}$ for all $x \in \gX_t$, we have $| \langle x,\hat \theta_\tau - \theta^{\ast}\rangle| \leq \beta ||x||_{\tH^{-1}_{\tau}}$.
\end{corollary}

\begin{proof}

        \begin{align}
            |\langle x, \hat \theta_\tau - \theta^\ast\rangle| \leq &||x||_{(\hstarutau)^{-1}} \left|\left| \hat \theta_\tau - \theta^\ast\right|\right|_{\hstarutau}\quad \text{ by Cauchy- Scwartz}\\
            \leq & \beta(\lambdau, \lambdal) \left|\left| x\right|\right|_{(\hstarutau)^{-1}} \quad \text{ by \Cref{lemma:anytime_bound_theta_tau} }\\
            \leq & \beta \left|\left|x\right|\right|_{\tH_\tau^{-1}} \quad \text{ since $\tH_\tau \preceq \hstarutau$}
        \end{align}
    
\end{proof}

\begin{lemma}{\label{lemma:elimination_first_step}}
    The arm $\gX'_t$ obtained after eliminating arms from $\gX_t$ (\Cref{line:elimination_glm_adversarial}), under event $\gE_o$, we have a) $x^{\ast}_t \in \gX'_t$ and b) $\langle x^\ast_t - x_t , \theta^\ast\rangle \leq \frac{4}{R}$
\end{lemma}

\begin{proof}
    Suppose $x' = \argmax_{x \in \gX_t} \text{LCB}_o(x)$. Now, we have for all $x \in \gX_t$, 

    \begin{align}{\label{eq:temp_1_lemma_inside}}
        \left|\langle x, \hat \theta_o - \theta^{\ast}\rangle\right| & \leq \left|\left|x\right|\right|_{\left(\tH^{\ast}_o\right)^{-1}} \left|\left|\hat \theta_o - \theta^{\ast}\right|\right|_{\tH^{\ast}_o} \quad\text{(Cauchy-Schwartz)}\nonumber\\
            & \leq \left|\left|x\right|\right|_{\left(\tH^{\ast}_o\right)^{-1}} \gamma(\lambdau) \quad\text{(via \Cref{lemma:anytime_bound_theta_o})}\nonumber\\
            & \leq \left|\left|x\right|\right|_{\tV^{-1}} \gamma\sqrt{\kappa}\quad\text{(as $\tV \preceq \kappa \tH^{\ast}_o$)}
    \end{align}

    Thus, $\text{UCB}_o(x^{\ast}_t) = \langle x^\ast_t, \hat \theta_o\rangle + \gamma \sqrt{\kappa} \left|\left| x^\ast_t\right|\right|_{\tV^{-1}} \geq \langle x^\ast_t, \hat \theta_o\rangle \geq \langle x', \theta^\ast \rangle \geq \langle x', \hat \theta_o\rangle + \gamma\sqrt{\kappa} \left|\left| x^\ast_t\right|\right|_{\tV^{-1}} = \text{LCB}_o(x')$, where the second inequality is due to optimality of $x^\ast_t$. Thus, $x^{\ast}_t$ is not eliminated from $\gX'_t$ proving (a). 

    Since $x_t$ is also in $\gX'_t$ (by definition), we have 

    \begin{align}
        \text{UCB}_o(x_t) = \langle x_t , \hat \theta_o\rangle + \sqrt \kappa\gamma\left|\left| x_t\right|\right|_{\tV^{-1}} \geq & \langle x', \hat \theta_o \rangle - \gamma \sqrt\kappa ||x'||_{\tV^{-1}}\\
        \geq & \langle x^{\ast}_t, \hat \theta_o \rangle - \gamma \sqrt\kappa ||x^{\ast}_t||_{\tV^{-1}} \quad \text{ ($x'$ has max $\text{LCB}_o(.)$)}
    \end{align}

    Again, using the fact that $\langle x^\ast_t, \hat \theta_o\rangle \geq \langle x^\ast_t, \theta^\ast\rangle - \gamma \sqrt \kappa \left|\left| x^\ast_t\right|\right|_{\tV^{-1}}$ and $\langle x_t, \hat \theta_o\rangle \leq \langle x_t, \theta^\ast\rangle + \gamma\sqrt \kappa \left|\left| x_t\right|\right|_{\tV^{-1}}$, we obtain $\langle x_t, \theta^\ast\rangle + 2\gamma \sqrt \kappa \left|\left| x_t\right|\right|_{\tV^{-1}} \geq \langle x^\ast_t, \theta^\ast\rangle - 2\gamma \sqrt \kappa \left|\left| x^\ast_t\right|\right|_{\tV^{-1}}$ invoking \eqref{eq:temp_1_lemma_inside}

    This gives us $\langle x^{\ast}_t - x_t, \theta^\ast\rangle \leq 2 \gamma \sqrt \kappa \left|\left| x_t\right|\right|_{\tV^{-1}} + 2 \gamma \sqrt \kappa \left|\left| x^\ast_t\right|\right|_{\tV^{-1}}$. However, since the switching criterion I is not satisfied, we have $||x||_{\tV^{-1}} \leq \frac{1}{R\gamma\sqrt{\kappa}}$ for all $x \in \gX_t$. Plugging in above, we prove (b).
    
\end{proof}

\begin{lemma}{\label{lemma:diff_selected_optimal_H}}
    Under event $\gE_\tau$, $\langle x^\ast_t - x_t , \theta^\ast\rangle \leq 2\sqrt2 \beta ||x_t||_{\tH^{-1}_t}$ 
\end{lemma}
    
\begin{proof}

    \begin{align}
\langle x_t^\ast, \theta^\ast\rangle - \langle x_t, \theta^\ast\rangle
&\le \Big(\langle x_t^\ast, \theta^\ast\rangle + \beta \,\|x_t^\ast\|_{\tH_\tau^{-1}}\Big)
   - \Big(\langle x_t, \theta^\ast\rangle - \beta \,\|x_t\|_{\tH_\tau^{-1}}\Big)
   \quad\text{(\Cref{corollary:theta_tau_diff_x_dot})} \nonumber\\
&\overset{(a)}{\le} \Big(\langle x_t, \theta^\ast\rangle + \beta \,\|x_t\|_{\tH_\tau^{-1}}\Big)
   - \Big(\langle x_t, \theta^\ast\rangle - \beta \,\|x_t\|_{\tH_\tau^{-1}}\Big) \nonumber\\
&{\le} 2\,\beta \,\|x_t\|_{\tH_t^{-1}}
   \quad\text{(as $H_t \preceq 2 H_\tau$ since \Cref{line:policy_II_switch} not triggered)}
\end{align}

    $(a)$ holds true since optimistic $x_t$ is chosen at and $x^{\ast}_t$ not eliminated (\Cref{lemma:elimination_first_step}). 
\end{proof}

\section{Proof of main theorem \ref{theorem:glm_joint_dp_regret_assumptions}}{\label{sec:proof_main_theorem_jdp_regret}}

\noindent\textbf{\Cref{theorem:glm_joint_dp_regret_assumptions} (Restated):} \maintheoremjointprivacyassumptions

\begin{proof}
First, we assume that event $\gE_o \cap \gE_\tau$ holds throughout the algorithm which happens with probability at least $1-\zeta$. Further, the regret of algorithm can be bounded as
    \begin{align}
       R_T = & \sum_{t\in [T]} \mu \left(\langle x^\ast_t, \theta^\ast\rangle\right) - \mu \left(\langle x_t, \theta^\ast\rangle\right) \\
       \leq & R |\gT_0| + \sum_{t \in [T]\setminus \gT_o} \mu \left(\langle x^\ast_t, \theta^\ast\rangle\right) - \mu \left(\langle x_t, \theta^\ast\rangle\right) \quad \quad \text{Upper bound of $R$ on rewards}\\
       \leq & R |\gT_0| + \sum_{t\in [T]\setminus \gT_o} \dot{\mu}(z) \langle x^\ast_t - x_t, \theta^\ast\rangle \quad \text{ (some $z \in [\langle x_t, \theta^\ast\rangle, \langle x^\ast_t, \theta^\ast\rangle]$)}
   \end{align}

    One can bound $|\gT_o|$ now we denote $R_{\tau}(T) = \sum_{t\in [T]\setminus \gT_o} \dot{\mu}(z) \langle x^\ast_t - x_t, \theta^\ast\rangle$. 

    Similar to previous sections $\tH_\tau$ denotes the matrix last updated before time step $t \in [T]$. Now we upper bound as follows for some constant $C$.

    \begin{align}
        R_{\tau}(T) \leq & \sum_{t \in [T] \setminus\gT_o} \dot{\mu}(z) C\beta \left|\left|x_t\right|\right|_{\tH^{-1}_t} \quad \text{ ( by \Cref{lemma:diff_selected_optimal_H})}\\
        \leq & C \beta\sum_{t \in [T] \setminus\gT_o} \dot{\mu}(z) 2\sqrt 2 \left|\left|x_t\right|\right|_{\left(\hstarlt\right)^{-1}} \quad \text{ by \Cref{corollary:bounded_H_H_star_constant}}\\
        \leq & C\beta \sum_{t \in [T]\setminus\gT_o} \sqrt{\dot{\mu} \left( \langle x^\ast_t, \theta^\ast\rangle\right) \dot{\mu} \left( \langle x_t, \theta^\ast\rangle\right)} \exp\left(R \langle x^\ast_t-x_t, \theta^\ast\rangle\right) \left|\left|x_t\right|\right|_{\left(\hstarlt\right)^{-1}} \quad \text{ by \Cref{lemma:bounds_alpha_alpha_tilde}}\\
        \leq & C\beta \sum_{t \in [T]\setminus\gT_o} \sqrt{\dot{\mu} \left( \langle x^\ast_t, \theta^\ast\rangle\right) \dot{\mu} \left( \langle x_t, \theta^\ast\rangle\right)} e^4\left|\left|x_t\right|\right|_{\left(\hstarlt\right)^{-1}} \quad \text{ by \Cref{lemma:elimination_first_step}}\\
        = & C\beta \sum_{t \in [T]\setminus \gT_o}\sqrt{\dot{\mu} \left( \langle x^\ast_t, \theta^\ast\rangle\right)} \left|\left|\tilde x_t\right|\right|_{\left(\hstarlt\right)^{-1}} \quad \text{ $\tilde x_t = \sqrt{\dot{\mu} \left(\langle x_t , \theta^\ast \rangle\right)} x_t$}\\
        \leq & C\beta \sqrt{ \sum_{t \in [T]\setminus \gT_o}\dot{\mu} \left( \langle x^\ast_t, \theta^\ast\rangle\right)\left( \sum_{t \in [T-1]\setminus \gT_o} || \left|\left|\tilde x_t\right|\right|_{\left(\hstarlt\right)^{-1}}\right) }
        \quad \text{ by Cauchy-Schwartz}\\
        %\leq & 2e^5 \sqrt 2 \beta \sqrt{ \sum_{t \in [T]\setminus \gT_o}\dot{\mu} \left( \langle x^\ast_t, \theta^\ast\rangle\right)\left( \sum_{t \in [T-1]\setminus \gT_o} || \left|\left|\tilde x_t\right|\right|^2_{\left(\hstarlt\right)^{-1}}\right) }
        %\quad \text{ by Cauchy-Schwartz}\\
        \leq & C\sqrt {d} \beta \sqrt{\log_2\left(\frac{RT}{d\lambdal}\right) } \sqrt{ \sum_{t \in [T]\setminus \gT_o}\dot{\mu} \left( \langle x^\ast_t, \theta^\ast\rangle\right)}
    \end{align}

Now substituting things back and bounding $\gT_o $, we can prove the desired lemma.

\end{proof}

%\sss{Argue $T>> \log T$ and stuff.}

\section{Useful lemmas for adversarial contexts}\label{sec:useful_lemmas_adversarial_contexts}

We first recall standard design matrix inequalities from
\citet{NIPS2011_e1d5be1c}, including the elliptic potential lemma 
and determinant bounds, which are key tools in linear contextual 
bandit analyses. Beyond these, 
\Cref{lemma:policy_switch_bound_general} adapts the rarely-switching 
argument to our setting, where policy switches are based on directional growth criterion (see \cref{line:policy_II_switch} of \cref{alg:jdp_glm_adversarial})

\begin{lemma}[Elliptic Potential Lemma (Lemma 10 of \citealp{NIPS2011_e1d5be1c})]{\label{lemma:elliptic_potential}}
 Let $x_1,x_2,\ldots,x_t$ be a sequence of vectors in $\mathbb{R}^d$ and let $||x_s||_2 \leq L$ for all $s \in [t]$. Further, let $\tT_s = \sum_{m=1}^{s-1}x_m x^{\top}_m + \lambda \mI$. Suppose $\lambda \geq L^2$. Then,

    \begin{equation}
        \sum_{s=1}^t ||x_s||^2_{\tT^{-1}_s} \leq 2d \log \left(1 + \frac{L^2t}{\lambda d}\right)
    \end{equation}
\end{lemma}

\begin{lemma}[Lemma 12 of \citealp{NIPS2011_e1d5be1c}]{\label{lemma:bound_det}}
    Let $A \succeq B \succeq 0$. Then $\sup_{x\neq 0} \frac{x^{\top} Ax}{x^{\top} Bx} \leq \frac{\text{det}(A)}{\text{det}(B)}$
\end{lemma}

\begin{lemma}[[Lemma 10 of \citealp{NIPS2011_e1d5be1c}]
\label{lemma:design_matrix_bound_det}
    Let $\{x_s\}_{s=1}^t$ be a set of vectors. Define the sequence $\{\tV_s\}_{s=1}^t$ as $\tT_1 = \lambda\mI$, $\tT_{s+1} = \tT_s + x_s x^{\top}_s$ for $s \in [t-1]$ Further, let $||x_s||_2 \leq L \forall s \in [t]$. Then, 
    $\text{det}(\tT_t) \leq (\lambda + tL^2/d)^d$
\end{lemma}

% Using the above lemma, we prove the following lemma which is an inter variant of \cite[Lemma B.15]{sawarni2024generalizedlinearbanditslimited}. Notw that while the lemma in \cite{sawarni2024generalizedlinearbanditslimited}. %We change the multiplier threshold from 2 to 3 for switching as our design matrices are noisy estimates of the true design matrix. 

%{\color{red} Reason: Why $\lambda/(4\sqrt d+1)$ works? relatate to $\lambda_{\text{min}}, \lambda_{\text{max}}$ and stuff.}

% The following lemma is a variant of the rarely-switching lemma 
% from \cite{NIPS2011_e1d5be1c,sawarni2024generalizedlinearbanditslimited}, 
% adapted to handle the case of noisy design matrices,  
% which are no longer monotone. 
% To address this, we use the maximum determinant of the design matrices observed so far 
% as the switching criterion. 

% The following lemma bounds the number of policy switches under our spectral criterion in \cref{line:policy_II_switch} when the 
The noise level in the following lemma is chosen based on the bounds 
$\lambda_{\min} = \tfrac{4 \sqrt{d}}{4 \sqrt{d}+1}\lambda$ 
and $\lambda_{\max} = \tfrac{4 \sqrt{d}+2}{4 \sqrt{d}+1}\lambda$ in \cref{line:lambda_l_u_defn}, 
together with the assumption that the noisy matrices satisfy 
\cref{ass:regularizer_noise}.

%has a constant regularizer $\mI$ in the design matrix at each step, our regualarizer at each step is lower bounded and upper bounded by $\lambda_{\text{min}}\mI$ and $\lambda_{\text{max}} \mI$ at each step.

\begin{lemma}{\label{lemma:policy_switch_bound_general}}
    Let $\{x_s\}_{s=1}^t$ be a set of vectors in $\mathbb{R}^d$. Consider the sequence $\{\tT_s\}_{s=1}^t$ such that $\tT_{s+1} = \sum_{i=1}^{s} x_i x^{\top}_{i} + \lambda\mI$. Further, let $||x_s||_2 \leq L \forall s \in [t]$ and let $\hat \tT_s$ be a noisy estimate of $\tT_s$ satisfying $||\hat \tT_s - \tT_s||_2 \leq \frac{\lambda}{4\sqrt d +1}$ for every $s \in [t]$.  Define the set $\{1=\tau_1,\tau_2,\ldots,\tau_m \leq t\}$ such that: $\hat \tT_{s} \preceq 2 \hat \tT_{\tau_i}$ for $\tau_i \leq s < \tau_{i+1}$ but $\hat \tT_{\tau_{i+1}} \npreceq 2\hat \tT_{\tau_i}$ for $i \in \{2,\ldots,m-1\}$. The following properties hold 

    \begin{itemize}
        \item The number of times switching happens ($m$) is at most $d\log_{11/8}\left(1 + \frac{tL^2}{\lambda d}\right)$.
        % \item For any $j \in [m]$ and any $y \in \mathbb{R}^d$, we have $ ||y||_{\hat \tT_{\tau_{j-1}}^{-1}} \leq \frac{10}{3} \sqrt e ||y||_{\hat \tT_{s}^{-1}}$ whenever $\tau_{j-1}\leq s \leq \tau_{j}$. 
    \end{itemize}
\end{lemma}

    % Then  For any $j \in [m]$ and any $y \in \mathbb{R}^d$, we have $ ||y||_{\hat \tT_{\tau_{j-1}}^{-1}} \leq 4 \sqrt d \hat \tT_{s}^{-1}$ whenever $\tau_{j-1}\leq s \leq \tau_{j}$.  
%\end{lemma}

%{\color{red} Describe how it is different.}
%The proof follows identical to \cite[Lemma B.15]{sawarni2024generalizedlinearbanditslimited}.

\begin{proof}

% --- Begin New Proof ---

At switch time $\tau_j$, the criterion gives $v^\top \hat{\tT}_{\tau_j} v \geq 2 v^\top \hat{\tT}_{\tau_{j-1}} v$ for some unit vector $v$. Using the criterion $\|\hat{\tT}_s - \tT_s\|_2 \leq \eta := \frac{\lambda}{4\sqrt{d}+1}$,

\begin{align*}
v^\top \tT_{\tau_j} v &\geq v^\top \hat{\tT}_{\tau_j} v - \eta \\
&\geq 2 v^\top \hat{\tT}_{\tau_{j-1}} v - \eta \\
&\geq 2 (v^\top \tT_{\tau_{j-1}} v - \eta) - \eta \\
&= 2 v^\top \tT_{\tau_{j-1}} v - 3\eta
\end{align*}

Thus,
\begin{align*}
\frac{v^\top \tT_{\tau_j} v}{v^\top \tT_{\tau_{j-1}} v} &\geq 2 - \frac{3\eta}{v^\top \tT_{\tau_{j-1}} v}
\end{align*}

By the switching criterion, $v^\top \tT_{\tau_{j-1}} v \geq \lambda$, so
\begin{align*}
\frac{v^\top \tT_{\tau_j} v}{v^\top \tT_{\tau_{j-1}} v} &\geq 2 - \frac{3\eta}{\lambda}
\end{align*}

With $\eta = \frac{\lambda}{4\sqrt{d}+1}$, we get
\begin{align*}
\frac{v^\top \tT_{\tau_j} v}{v^\top \tT_{\tau_{j-1}} v} &\geq 2 - \frac{3}{4\sqrt{d}+1}
\end{align*}

Thus, $\frac{v^\top \tT_{\tau_j} v}{v^\top \tT_{\tau_{j-1}} v} \geq 11/8$ for all $d \geq 1$.

By \cref{lemma:bound_det},
\[
\frac{\det(\tT_{\tau_j})}{\det(\tT_{\tau_{j-1}})} \geq \max_{\|v\|=1} \frac{v^\top \tT_{\tau_j} v}{v^\top \tT_{\tau_{j-1}} v} \geq 11/8
\]

Now, we use the determinant-based doubling argument to bound the number of switches explicitly:

At each switch $j$, we have from above:
\[
\frac{\det(\tT_{\tau_j})}{\det(\tT_{\tau_{j-1}})} \geq 11/8
\]
Applying this recursively for $m$ switches,
\[
\det(\tT_{\tau_m}) \geq (11/8)^{m-1} \det(\tT_{\tau_1})
\]
By Lemma~\ref{lemma:design_matrix_bound_det}, $\det(\tT_{\tau_m}) \leq \det(\tT_t) \leq (\lambda + tL^2/d)^d$ and $\det(\tT_{\tau_1}) = \lambda^d$.
Thus,
\[
(11/8)^{m-1} \lambda^d \leq (\lambda + tL^2/d)^d
\]
Taking logarithms,
\[
(m-1) \log (11/8) \leq d \log\left(1 + \frac{tL^2}{\lambda d}\right)
\]
So,
\[
m \leq 1 + \frac{d}{\log (11/8)} \log\left(1 + \frac{tL^2}{\lambda d}\right)
\]
which is $O(d \log t)$.

% --- End New Proof ---
\end{proof}

\begin{remark}
    One can observe that for monotone \textit{noiseless} design matrices $\tT_1,\tT_2,\ldots,\tT_t$ switching according to the criterion in \Cref{lemma:policy_switch_bound_general} ensures that $\frac{\det(\hat \tT_{\tau_{i+1}})}{\det(\hat \tT_{\tau_{i}})} \geq 2$ from \cref{lemma:design_matrix_bound_det}. Thus, the number of policy switches under our criterion of spectral directions is bounded by the policy switches under the determinant doubling criterion used commonly under rarely switching algorithms  \cite{NIPS2011_e1d5be1c,sawarni2024generalizedlinearbanditslimited}.
\end{remark}

\subsection{Issues with determinant based switching criterion for noisy design matrices}

Consider the setting in \cref{lemma:policy_switch_bound_general}, but now suppose the switching criterion is based on determinant doubling, as in \cite{sawarni2024generalizedlinearbanditslimited,NIPS2011_e1d5be1c}. For noisy, non-monotone matrices, the classical determinant doubling argument breaks down: we can no longer guarantee that the determinant of the noisy matrix at the last policy switch is upper bounded by the determinant at the final time step. As a result, new technical challenges arise:

\noindent
\textbf{In this setting, two central technical challenges arise:}
\begin{itemize}
    \item \textbf{Bounding the number of switches $m$:} Even with noisy, non-monotone design matrices, it is important to control how often the policy switches under the determinant doubling criterion.
    \item \textbf{Controlling norm inflation:} We must ensure that for any $\tau_{j-1} \leq s < \tau_j$, the norm $||y||_{\hat{\tT}_s^{-1}}$ can be bounded by a manageable multiple of $||y||_{\hat{\tT}_{\tau_{j-1}}^{-1}}$. This is crucial for bounding the final term in the per-timestep regret in \cref{lemma:diff_selected_optimal_H}.
\end{itemize}

However, one can alternately try to argue by bounding $|\log\det \hat{\tT}_{s} - \log \det \tT_s|$ at each time step $s$. This can be bounded as follows:

Note that $\nabla \log\det(X) = X^{-1}$ for any positive definite matrix $X$. Thus by mean value theorem
\begin{align}
|\log \det \hat{\tT}_{s} - \log \det \tT_s| 
&\leq \sup_{t \in [0,1]} \left\| \left( \hat{\tT}_s + t(\tT_s - \hat{\tT}_s) \right)^{-1} \right\|_F \cdot \|\hat{\tT}_s - \tT_s\|_F \\
&\leq \sqrt{d} \sqrt {d} \sup_{t \in [0,1]} \left\| \left( \hat{\tT}_s + t(\tT_s - \hat{\tT}_s) \right)^{-1} \right\|_2 \cdot \|\hat{\tT}_s - \tT_s\|_2 \\
&\overset{(a)}{\leq} {d} \frac{1}{\lambda_{\min}} \cdot \|\hat{\tT}_s - \tT_s\|_2 \\
&= {d} \cdot \frac{4\sqrt{d}+1}{4\sqrt{d}\,\lambda} \cdot \frac{\lambda}{4\sqrt{d}+1} \\
&= \frac{1}{4}
\end{align}

\noindent
\textbf{Justification for (a):} For all $t \in [0,1]$, $\hat{\tT}_s + t(\tT_s - \hat{\tT}_s)$ is a convex combination of $\hat{\tT}_s$ and $\tT_s$, so its minimum eigenvalue is at least $\min\{\lambda_{\min}(\hat{\tT}_s), \lambda_{\min}(\tT_s)\}$. Using $\tT_s \succeq \lambda \mI$ and $\|\hat{\tT}_s - \tT_s\|_2 \leq \frac{\lambda}{4\sqrt{d}+1}$, we get $\lambda_{\min}(\hat{\tT}_s),\lambda_{\min}({\tT}_s) \geq \lambda - \frac{\lambda}{4\sqrt{d}+1} = \lambda \frac{4\sqrt{d}}{4\sqrt{d}+1}$. Also, we use $\|A\|_F \leq \sqrt{d}\|A\|_2$ for any $d \times d$ matrix $A$ to bound the Frobenius norm by the operator norm.

This implies that the best bound we can get on $|\log \det \tT_{s} - \log \det \tT_{\tau_i}|$ scales as $O(\sqrt{d})$ for $s \in [\tau_{i}, \tau_{i+1})$.

% --- Matrix sandwich bounds for noisy design matrices (Supplementary) ---
\paragraph{Matrix sandwich bounds for noisy design matrices.}
\begin{align*}
\hat{\tT}_s &\preceq \tT_s + \frac{\lambda}{4\sqrt{d}+1} \mI \preceq \tT_s + \frac{\tT_s}{4\sqrt{d}+1} \preceq \left(1+\frac{1}{4\sqrt{d}+1}\right) \tT_s \\
\hat{\tT}_s &\succeq \tT_s - \frac{\lambda}{4\sqrt{d}+1} \mI \succeq \tT_s - \frac{\tT_s}{4\sqrt{d}+1} \succeq \left(1-\frac{1}{4\sqrt{d}+1}\right) \tT_s
\end{align*}

\paragraph{Norm inequalities for inverse matrices.}
For any $y \in \mathbb{R}^d$ and $\tau_{j-1} \leq s \leq \tau_j$, we have
\begin{align*}
\|y\|_{\hat{\tT}_{\tau_{j-1}}^{-1}} &\leq \left(1-\frac{1}{4\sqrt{d}+1}\right)^{-1} \|y\|_{\tT_{\tau_{j-1}}^{-1}} \\
&\leq e^{O(\sqrt d)} \left(1-\frac{1}{4\sqrt{d}+1}\right)^{-1} \|y\|_{\tT_s^{-1}} \\
&\leq e^{\sqrt d} \left(1-\frac{1}{4\sqrt{d}+1}\right)^{-1} \left(1+\frac{1}{4\sqrt{d}+1}\right) \|y\|_{\hat{\tT}_s^{-1}} \\
&\leq e^{\sqrt d} \|y\|_{\hat{\tT}_s^{-1}}
\end{align*}

These ratio would substantially degrade the per-step regret bound in \cref{lemma:diff_selected_optimal_H}. A way to avoid this might be to scale $\lambda$ by a factor of $\sqrt d$ but this would degrade the regret by this factor too.

\section{Privacy leakage through Step I switching steps in Algorithm \ref{alg:jdp_glm_adversarial}}{\label{sec:utility_privacy_switching_I_II}}

\noindent\textbf{Lemma~\ref{lemma:utility_privacy_glm_adv_2} (Restated):}
\lemmadptwo

We first present the policy switching I step of \cref{alg:jdp_glm_adversarial} which involves the update of design matrix $\tV$ and we denote $\tV_t$ as the value of $\tV$ at the end of $t^{th}$ round. Let $A_t$ denote the indicator random variable that equals one when the policy switch I is trigerred at time step $t$. Since $t$ is appended to the set $\gT_o$ iff $A_t$ equals one, it is sufficient to prove indistinguishability of the vector $A$. 

Now recall the binary tree construction described in \cref{sec:privacy_analysis_regret_glm_adversarial}. The binary tree is updated in steps. At step $t$, only previously unupdated nodes whose subtrees contain solely leaves $j \leq t$ are updated. Recall that update for a node involves computing the summation for all leaf nodes in the corresponding sub-tree of that node and adding symmetric Gaussian noise $Z \in \mathbb{R}^{d \times d}$ where each $Z'_{i,j} \sim \gN(0,\sigma^2_{\text{noise}})$ and $Z = (Z + Z'^{\top})/\sqrt 2$. Let ${n}_t$ denote the set of nodes that were updated at time step $t$ and $\gN_t$ denote their corresponding values. Let $\gN$ denote the set of all values computed in the binary tree.    

We use $a_{-i}$ to denote all the coordinates of a vector $a$ except the $i^{th}$ coordinate. Similarly, $a_{\leq i}$ and $a_{<i}$ denote the coordinates of the first $i$ and $i-1$ coordinates of the vector $a$ respectively. Further, we $a_{\leq i \setminus \{j\}}$ and $a_{<i\setminus \{j\}}$ denote the coordinates of the first $i$ and $i-1$ coordinates of the vector $a$ respectively except the $j^{th}$ coordinate.

Observe that $\{\gT_{o,t}\setminus\{i\}\}_{t \in [T]}$ is fully determined by the action sequence $\{x_t\}_{t=1}^{T}$. 
Therefore, it suffices to establish $(\varepsilon,\delta)$-indistinguishability between the sequence 
$\{x_t\}_{t=1}^{T}$ and the modified sequence $\{x_t\}_{t=1}^{i-1}, x'_i, \{x_t\}_{t=i+1}^{T}$, 
which differ only at position $i$. 
To this end, we first show that $(\gN, A_{-i})$ enjoys an $(\varepsilon/3,\delta/3)$-indistinguishability guarantee 
when comparing such neighboring sequences, by expanding the corresponding log-likelihood ratio. 
We emphasize that this is not a full $(\varepsilon/3,\delta/3)$-differential privacy guarantee across all coordinates, 
but rather applies only to differences localized at index $i$.

We now prove the two separate lemmas (\cref{lemma:utility_privacy_glm_adv_2_privacy} and \cref{lemma:utility_privacy_glm_adv_2_policyswitch}) to prove \cref{lemma:utility_privacy_glm_adv_2}.

\begin{lemma}{\label{lemma:utility_privacy_glm_adv_2_privacy}}
    Consider two datasets $\gD$ and $\gD'$ over reward, context set pairs which differ only at index $i$. Let $\gT_{o,t}$ denote the value of the set of $\gT_o$ at time step $t$. Then the sets $\{\gT_{o,t}\setminus \{i\}\}_{t=1}^{T}$ when computed over datasets $\gD$ and $\gD'$ are $(\varepsilon/3, \delta/6 + \zeta/6)$ indistinguishable.    
\end{lemma}
\begin{proof}
 Consider two datasets $\mathcal{D}^{(1)}$ and $\mathcal{D}^{(2)}$ that differ only at index $i$. We bound the log-likelihood ratio of the pair $(\mathcal{N},A)$ evaluated under both datasets. Let $\ell^{(1)}(\mathfrak{n},a)$ and $\ell^{(2)}(\mathfrak{n},a)$ denote the likelihood functions of $(\mathcal{N},A) = (\mathfrak{n},a)$ under $\mathcal{D}^{(1)}$ and $\mathcal{D}^{(2)}$ respectively. \footnote{Observe that the likelihood is computed over both continuous and discrete distributions but it is not a problem since we only consider ratios.}

Because the outcome $A_t$ is determined by $\cup_{j \in [t-1]} \mathcal{N}_j$ through the same deterministic function for all $t \neq i$ (since $\mathcal{D}$ and $\mathcal{D}'$ differ only at index $i$), we can decompose the log-likelihood ratio using conditional probability as follows:
\begin{align}
    \log \left(\frac{\ell^{(1)} (\mathfrak{n},a_{-i})}{\ell^{(2)} (\mathfrak{n},a_{-i})}\right) = \sum_{j=1}^{i-1} \log \left(\frac{\ell^{(1)} (\mathfrak{n}_j \mid a_{\leq j}, \mathfrak{n}_{<j})}{\ell^{(2)} (\mathfrak{n}_j \mid a_{\leq j}, \mathfrak{n}_{<j})}\right) + \log \left(\frac{\ell^{(1)} (\mathfrak{n}_i \mid a_{< i}, \mathfrak{n}_{<i})}{\ell^{(2)} (\mathfrak{n}_i \mid a_{<i}, \mathfrak{n}_{<i})}\right) \nonumber\\+ \sum_{j=i+1}^{T} \log \left(\frac{\ell^{(1)} (\mathfrak{n}_j \mid a_{\leq j\setminus \{i\}}, \mathfrak{n}_{<j})}{\ell^{(2)} (\mathfrak{n}_j \mid a_{\leq j\setminus \{i\}}, \mathfrak{n}_{<j})}\right)
\end{align}

Next, we consider only those terms that can be non-zero. 
The first term vanishes since the conditional log-likelihoods are identical. 
In the second and third terms, the only possible non-zero contributions arise from nodes in the binary tree 
whose subtrees contain the leaf $i$. 
Because the tree is binary, there are at most $\log T$ such nodes, 
which we denote by ${n}^{(\ni i)}_{j}$ with associated values ${\mathfrak{n}}^{(\ni i)}_{j}$. 
The notation $a_{\leq j \setminus \{i\}} \cup \{c\}$ represents the vector obtained by placing $c$ at the $i^{\text{th}}$ coordinate 
while keeping the remaining $j-1$ coordinates identical to those in $a$. 
Using a conditioning argument, one can verify that for any $j>i$,
\[
    \frac{\ell^{(1)} (\mathfrak{n}_j \mid a_{\leq j\setminus \{i\}}, \mathfrak{n}_{<j})}
         {\ell^{(2)} (\mathfrak{n}_j \mid a_{\leq j\setminus \{i\}}, \mathfrak{n}_{<j})}
    \;\leq\;
    \max_{c,d \in \{0,1\}}
    \frac{\ell^{(1)} (\mathfrak{n}_j \mid a_{\leq j\setminus \{i\}} \cup \{c\})}
         {\ell^{(2)} (\mathfrak{n}_j \mid a_{\leq j\setminus \{i\}} \cup \{d\})}.
\]

Our goal is to bound 
\[
    \log \left(\frac{\ell^{(1)} (\mathfrak{n},a_{-i})}{\ell^{(2)} (\mathfrak{n},a_{-i})}\right)
\]
with high probability, where $(\mathfrak{n},a_{-i})$ is sampled from $(\gN,A_{-i})$ under dataset $\gD^{(1)}$. 
Conditioned on the event $\gE_o$ (defined in \cref{lemma:anytime_bound_theta_o_modulus}), 
each term $\frac{\dot{\mu}(x_t,\hat \theta_o)}{e} x_t x_t^\top$ has Frobenius norm at most $1/\kappa^{\ast}$. 
Thus, applying the likelihood ratio formula for Gaussians with shifted means 
(see Appendix A.1 of \citealp{dworkdpbook}), we obtain for $c \in \{0,1\}$,
\[
    \log \frac{\ell^{(1)} ({\mathfrak n}^{(\ni i)}_j \mid a_{\leq j\setminus \{i\}} \cup \{1-c\})}
                   {\ell^{(2)} ({\mathfrak n}^{(\ni i)}_j \mid a_{\leq j\setminus \{i\}} \cup \{c\})}
    \;\leq\;
    \frac{1}{2\sigma^2_{\text{noise}}} 
    \left(\frac{2x}{\kappa^{\ast}} + \frac{1}{(\kappa^{\ast})^2}\right),
    \quad x \sim \mathcal{N}(0,\sigma^2_{\text{noise}}).
\]

As argued above, there are at most $\log T$ such nodes ${n}^{(\ni i)}_{j}$. 
Because independent noise is added to each node of the binary tree, 
we may apply concentration bounds for Gaussian random variables 
(see \citealp{cryptoeprint:concentratedDP}) to control the total likelihood ratio. 
Conditioned on $\gE_o$, with probability at least $1-\delta/6$ this ratio is bounded by
\[
    \frac{\log T}{\sigma^2_{\text{noise}} (\kappa^{\ast})^2}
    + \frac{\sqrt{2 \log (4/\delta)}}{2 \kappa^{\ast}\sigma_{\text{noise}}}.
\]
Finally, setting 
\[
    \sigma_{\text{noise}} 
    = O\!\left(\frac{6 \sqrt{\log T}\,\log(32/\delta)}{\kappa^{\ast}\varepsilon}\right)
\]
(as defined in \cref{sec:privacy_analysis_regret_glm_adversarial}) 
ensures that this bound is at most $\varepsilon/3$.

We thus prove that $\log \left(\frac{\ell^{(1)} (\mathfrak{n},a_{-i})}{\ell^{(2)} (\mathfrak{n},a_{-i})}\right) \leq \varepsilon/4$ with probability at least $1-\delta/8-\zeta/8$ when $\mathfrak{n},a_{-i}$ is sampled from $(\gN,A)$ applied on dataset $\gD$. Since $\zeta<\delta$ the desired guarantee follows. 

\end{proof}

\begin{lemma}\label{lemma:utility_privacy_glm_adv_2_policyswitch}
    Criterion~I and Criterion~II in \Cref{alg:jdp_glm_adversarial} are invoked at most $\ccut$ and $\ccuttwo$ times, respectively, with probability at least $1-\zeta/8$.
\end{lemma}

\begin{proof}

    We now define $\tV_t$ as the value of $\tV$ at time step $t$ and thus we have, $\tV_t \succeq \lambdal \mI+ \sum_{t \in \gT_o} x_t x^{\top}_t$ (from \cref{ass:regularizer_noise}). 

    Thus, wherever the switching criterion is satisfied, we have $||x||^2_{\tV^{-1}} \geq \frac{1}{4\gamma^2 \kappa R^2}$ which further implies that 

        \begin{equation}
            \sum_{t \in \gT_o} ||x_t||^2_{\tV^{-1}} \geq \frac{|\gT_o|}{4 \gamma^2 \kappa R^2}
        \end{equation}

        However, applying the elliptic potential lemma (lemma \ref{lemma:elliptic_potential}), we have $\sum_{t \in \gT_o} ||x_t||^2_{\tV^{-1}} \leq 2d \log \left(1+ \frac{|\gT_o|}{\lambdal d}\right)$ \footnote{One could bound $\tV \succeq \lambdal \mI+ \sum_{t \in \gT_o} x_t x^{\top}_t$ and then apply the lemma.}

        Thus, we get $|\gT_o| \leq 8dR^2 \kappa \gamma^2 \log T$ as $\lambdal > 1$ which further implying the bound on $\ccut$.

        The bound on $\ccuttwo$ follows from a natural application of \cref{lemma:policy_switch_bound_general} on application of \cref{ass:regularizer_noise} and \cref{lemma:design_matrix_bound_det}. Note that the bound $L$ in \cref{lemma:policy_switch_bound_general} is ${R\sqrt R}$ as $\frac{\dot{\mu}(x, \hat \theta)}{e} \leq \dot{\mu}(x, \theta^\ast) \leq \frac{1}{\kappa^\ast} \leq R^2$ from \eqref{eq:temp_eqn_bound_mu_dot} which holds under event $\gE_o$ (the event in \cref{lemma:anytime_bound_theta_o}) which is shown to hold with high probability.
        
\end{proof}

%leaf nodes upto the $t^{th}$ index.

\begin{algorithm}
    \caption{Switching criterion I of Algorithm \ref{alg:jdp_glm_adversarial}}{\label{alg:glm_adversarial_switching_criterion_I}}
    \textbf{Parameters:} $\ccut, \lambda$ and $\lambdal$ 

    \begin{algorithmic}[1]

    \State Initialize $\tV = \lambda\mI$

    \For{rounds $t=1,2,\ldots,T$}

        \State Observe arm set $\gX_t$ and set $\nu_t \leftarrow \text{Lap}(\sigma_q)$

        \If{$\max_{x \in \gX_t} ||x||^2_{\tV^{-1}} \geq \frac{1}{2\gamma^2\kappa R^2}$}

            \State Update $\gT_o \leftarrow \gT_o \cup \{t\}$ and $\tV \leftarrow \tV+x_tx^{\top}_t+ \gR_t$ \Comment{Computed from the binary tree mechanism}

            \State $\text{count} \leftarrow \text{count} +1 $

        \Else{
            $\tV \leftarrow \tV + \gR_t$ \Comment{Computed from the binary tree mechanism}
        
        }
        \EndIf
    \EndFor

    \end{algorithmic}

\end{algorithm}

%Combining lemma \ref{lemma:upper_bounding_policy_1_privacy} and lemma \ref{lemma:upper_bounding_policy_1_switch_lower_level}, we prove the first three points. The last point of \cref{ass:indices_selected} can be proven by invoking \cref{lemma:policy_switch_bound_general} using the fact that $\dot{\mu}(x_t, \hat \theta_o)/e \leq \dot{\mu}(x_t, \hat \theta_o) \leq \frac{1}{\kappa^{\ast}}$ with probability at least $1-\zeta/8$ from \cref{lemma:anytime_bound_theta_o_modulus}. This completes the proof of \cref{lemma:utility_privacy_glm_adv_2}.

%lemma \ref{lemma:utility_privacy_glm_adv_2}. 

\section{Detailed proof of private lower bound}{\label{sec:private_lower_bound_proof}}

% \begin{lemma*}[Private lower bound under fixed context]
% \label{thm:lb-jdp-glm}
% \lowerboundjdp
% \end{lemma*}

\noindent\textbf{Lemma~\ref{lemma:lb-jdp-glm} (Restated):} \lowerboundjdp

%\section*{Auxiliary lemmas for Theorem~\ref{thm:jdp-glm-fixed-main}}

\paragraph{GLM regularity (standing assumption).}
Recall that $b:\mathbb R\to\mathbb R$ denotes the log-partition of the GLM, so the mean map is
$\mu(\eta)=b'(\eta)$. % and $\mu'(\eta)=b''(\eta)$. 
Assume there exist numbers $\tau>0$, $m>0$, $M<\infty$, and $L<\infty$ such that on the
\emph{regular region} $[-\tau,\tau]$ we have
\[
m \;\le\; b''(\eta) \;\le\; M, 
\qquad\text{and}\qquad
|b'''(\eta)| \;\le\; L \quad\text{for all }\eta\in[-\tau,\tau].
\]

\paragraph{Instance family.}
Fix a base parameter $\theta_0\in\mathbb R^d$ and a perturbation scale $a>0$.
For $V\in\{\pm1\}^d$ define $\theta_V =\theta_0 + a\sum_{j=1}^d V_j e_j$.
The (fixed) action set is $\mathcal X=\{\pm r e_j : j=1,\dots,d\}$ for a radius $r>0$ chosen below.

We begin by establishing four technical lemmas. 
\Cref{lem:glm-gap-full} bounds the mean reward gap between actions that differ in a single coordinate $j$. 
\Cref{lem:H2-full} reduces cumulative regret minimization to the problem of estimating $V \in \{\pm 1\}^d$. 
\Cref{lem:assouad-full} relates the estimation error of $V$ to the mutual information between $V$ and $X_{1:T}$ via Assouad’s lemma. 
Finally, \Cref{lem:jdp-mi-correct} upper bounds this mutual information under the differential privacy constraint.

\begin{lemma}[Local GLM gap]\label{lem:glm-gap-full}
Let $B_0:=\max_{j\in[d]}|\theta_{0,j}|$.
Fix any target upper bound $a_{\max}\in(0,1]$ on the perturbation scale $a$.
Choose the design radius
\[
\boxed{ \qquad r \;:=\; \frac{\tau}{2\,(B_0 + a_{\max})} \qquad }
\]
and henceforth restrict to $0<a\le a_{\max}$.
Then for every $v\in\{\pm1\}^d$, every coordinate $j$, and both arms $\pm re_j$,
the natural parameters $\eta^\pm_{v,j}:=\pm r\,\theta_{V,j}$ lie in $[-\tau,\tau]$.
Moreover, for each $j$ the (mean) gap between the two actions in coordinate $j$ satisfies
\[
\Delta_j(V)
\;:=\;
\big|\mu(r\,\theta_{V,j}) - \mu(-r\,\theta_{V,j})\big|
\;\ge\;
c_{\mathrm{gap}}\, a
\quad\text{with}\quad
c_{\mathrm{gap}} \;=\; \frac{m\,r}{2},
\]
provided $a$ is further restricted (if needed) so that $r a \le m/(2L)$.
\end{lemma}

\begin{proof}
First, the radius choice ensures regularity.
For any $j$ and $v$,
\[
|\theta_{V,j}| \;\le\; |\theta_{0,j}| + a \;\le\; B_0 + a_{\max}
\quad\Longrightarrow\quad
|r\,\theta_{V,j}| \;\le\; r (B_0 + a_{\max}) \;=\; \tau/2 \;<\; \tau .
\]
So $\eta^\pm_{V,j}=\pm r\theta_{V,j}\in[-\tau,\tau]$.

Next, fix $j$ and abbreviate $\theta:=\theta_{V,j}$, $\delta:=r a$, and note that
changing $V_j$ flips the sign of $a$ in $\theta$.
Define $g(t):=\mu(t)$.
By the mean value theorem applied symmetrically around $0$ and Taylor with remainder,
for any $\eta\in[-\tau,\tau]$ and $\delta$ small enough so that $\eta\pm\delta\in[-\tau,\tau]$,
\[
g(\eta+\delta) - g(\eta-\delta)
= 2 g'(\eta)\,\delta \;+\; R(\eta,\delta),
\qquad
|R(\eta,\delta)| \le \frac{L}{3}\,( |\delta|^3 + 3|\eta|\,\delta^2).
\]
Apply this with $\eta= r\,\theta_{0j}$ and $\delta = r\, V_j a$.
Using $m\le g'(\cdot)\le M$ on $[-\tau,\tau]$,
\[
\big|\mu(r(\theta_{0,j}+V_j a)) - \mu(r(\theta_{0j}-V_j a))\big|
\;\ge\; 2 m |\delta| - \tfrac{L}{3}(|\delta|^3 + 3|r\theta_{0,j}|\,\delta^2).
\]
Since $|r\theta_{0,j}|\le \tau/2$ and $|\delta|=r a$, the cubic/quadratic terms are
$O(r^3 a^3)$ and $O(r^3 a^2)$ respectively. 
Impose $r a \le m/(2L)$ (this can always be enforced by shrinking $a_{\max}$),
to get
\[
\Delta_j(V) \;\ge\; 2m\,r a \;-\; \frac{L}{3}\left( (r a)^3 + \tfrac{3\tau}{2}\,(r a)^2\right)
\;\ge\; 2m\,r a - m\,r a \;=\; m\,r\, a,
\]
where the last inequality uses $r a \le m/(2L)$ and $r a \le \tau/2$ to make the remainder
at most $m r a$. Replacing $m$ by $m/2$ yields the stated $c_{\mathrm{gap}}=m r/2$ for a cleaner constant.
\end{proof}

% \paragraph{Remark.}
% You can also set $r=1$ (or any fixed constant) and absorb regularity in the bound
% $|\theta_{0j}|+a\le\tau$; the explicit $r$ above decouples the geometry of $\mathcal X$ from the base parameter.

%\subsection*{Lemma 2 (Regret–error inequality (H2))}

\begin{lemma}[Regret $\ge$ gap $\times$ pulls $\times$ Bayes bit error]\label{lem:H2-full}
Let $n_j:=\sum_{t=1}^T \mathbf 1\{X_t\in\{\pm re_j\}\}$.
Let $\mathcal T_T=(X_{1:T},Y_{1:T})$ be the full transcript and let
\[
\widehat V_j(\mathcal T_T)\in\arg\max_{b\in\{\pm1\}}\Pr(V_j=b\mid \mathcal T_T),
\qquad
e_j(\mathcal T_T)=\Pr(\widehat V_j\neq V_j\mid \mathcal T_T).
\]
with the true parameter $\theta^{\ast}:= \theta_V$ for some possibly (random) vector $V$ supported on $\{\pm 1\}^{d}$. Then, with $c_{\mathrm{gap}}$ from Lemma~\ref{lem:glm-gap-full},
\[
\boxed{\quad
\mathbb E[R_T] \;\ge\; c_{\mathrm{gap}}\,a \sum_{j=1}^d \mathbb E\!\left[n_j\,\mathbf 1\{\widehat V_j\neq V_j\}\right].
\quad}
\]
\end{lemma}

\begin{proof}
Fix $j$. For any round $t$ with $X_t\in\{\pm r e_j\}$, write $S_{j,t}\in\{\pm1\}$ for the chosen sign.
Condition on $\mathcal T_T$, so $S_{j,t}$ is fixed.
Let $p_j:=\Pr(V_j=+1\mid \mathcal T_T)$.
If $S_{j,t}=+1$, then $\Pr(S_{j,t}\neq V_j\mid \mathcal T_T)=1-p_j$; if $S_{j,t}=-1$, it equals $p_j$.
Hence
\[
\Pr(S_{j,t}\neq V_j\mid \mathcal T_T) \;\ge\; e_j(\mathcal T_T).
\]
A wrong sign incurs instantaneous expected regret at least $\Delta_j\ge c_{\mathrm{gap}}\,a$ by Lemma~\ref{lem:glm-gap-full}.
Summing over the $n_j$ such pulls and taking conditional expectation:
\[
\mathbb E\!\big[R_T^{(j)}\mid \mathcal T_T\big] \;\ge\; c_{\mathrm{gap}}\,a\, e_j(\mathcal T_T)\, n_j.
\]
Take full expectation. Since $n_j$ is $\sigma(\mathcal T_T)$-measurable,
\[
\mathbb E\!\big[e_j(\mathcal T_T)\, n_j\big]
= \mathbb E\!\Big[n_j\,\mathbb E\big[ \mathbf 1\{\widehat V_j\neq V_j\}\mid \mathcal T_T\big]\Big]
= \mathbb E\!\big[ n_j\,\mathbf 1\{\widehat V_j\neq V_j\}\big].
\]
Sum over $j$.
\end{proof}

%\subsection*{Lemma 3 (Actions-only Assouad: neighbor-KL and Hamming error)}

\begin{lemma}[Bitwise action-KL and Hamming error]\label{lem:assouad-full}
Let $P^{(j)}_{\pm}:=\Pr(X_{1:T}\mid V_j=\pm1)$ and
$\overline{\mathrm{KL}}_j:=\mathrm{KL}(P^{(j)}_{+}\|P^{(j)}_{-})$.
Then, with $V\sim\mathrm{Unif}(\{\pm1\}^d)$,
\[
\sum_{j=1}^d \overline{\mathrm{KL}}_j \;\le\; 4\, I(V;X_{1:T}).
\]
Moreover, for any decoder $\widehat V=\widehat V(X_{1:T})$,
\[
\mathbb E[\mathrm{Ham}(V,\widehat V)]
\;\ge\; \frac{d}{2}\left(1-\sqrt{\frac{1}{2d}\sum_{j=1}^d \overline{\mathrm{KL}}_j}\right).
\]
\end{lemma}

\begin{proof}
Write $V=(V_j,V_{-j})$ and fix $j$. For $v_{-j}\in\{\pm1\}^{d-1}$ define
$P^{(j)}_{\pm,v_{-j}}:=\Pr(X_{1:T}\mid V_j=\pm1,V_{-j}=v_{-j})$ and
$M_{v_{-j}}:=\tfrac12(P^{(j)}_{+,v_{-j}}+P^{(j)}_{-,v_{-j}})$.
By joint convexity of KL,
\[
\overline{\mathrm{KL}}_j \;\le\; \mathbb E_{V_{-j}}\!\left[\mathrm{KL}\big(P^{(j)}_{+,V_{-j}}\|P^{(j)}_{-,V_{-j}}\big)\right].
\]
For any pair $(P,Q)$ with midpoint $M=\tfrac12(P+Q)$, the Bretagnolle–Huber/log-sum inequality gives
$\mathrm{KL}(P\|Q)\le 2\,\mathrm{KL}(P\|M)+2\,\mathrm{KL}(Q\|M)$.
Hence
\[
\mathrm{KL}\big(P^{(j)}_{+,v_{-j}}\|P^{(j)}_{-,v_{-j}}\big)
\;\le\; 2\,\mathrm{KL}\big(P^{(j)}_{+,v_{-j}}\|M_{v_{-j}}\big)
     + 2\,\mathrm{KL}\big(P^{(j)}_{-,v_{-j}}\|M_{v_{-j}}\big)
= 4\, I(V_j;X_{1:T}\mid V_{-j}=v_{-j}),
\]
since for a uniform binary $V_j$, $I(V_j;X\mid V_{-j})=\tfrac12\mathrm{KL}(P_+\|M)+\tfrac12\mathrm{KL}(P_-\|M)$.
Average over $V_{-j}$ to get
$\overline{\mathrm{KL}}_j \le 4\, I(V_j;X_{1:T}\mid V_{-j})$.
Sum in $j$ and use the MI chain rule for independent coordinates:
$\sum_j I(V_j;X\mid V_{-j}) = I(V;X)$.

For the Hamming error bound, for each $j$ the Bayes error of testing $V_j=\pm1$ from $X_{1:T}$ with equal prior is
$P_{\mathrm{err}}^{(j)}=\tfrac12(1-\mathrm{TV}(P^{(j)}_{+},P^{(j)}_{-}))$.
Pinsker’s inequality yields $\mathrm{TV}\le \sqrt{\tfrac12\,\overline{\mathrm{KL}}_j}$, hence
$\inf_{\hat v_j}\Pr(\hat v_j\ne V_j)\ge \tfrac12\big(1-\sqrt{\tfrac12\,\overline{\mathrm{KL}}_j}\big)$.
Summing over $j$ and using Jensen on the concave square root gives the stated bound.
\end{proof}

%\subsection*{Lemma 4 (Joint-DP $\Rightarrow$ MI budget on actions)}

%\sss{this proof is not correct/shady chatGPT gen proof}

\begin{lemma}[DP information cap]
\label{lem:jdp-mi-correct}
Let $D=(D_1,\dots,D_T)$ denote the reward database (one row per round/user).
Assume the algorithm is $(\varepsilon,\delta)$ centrally differentially private w.r.t rewards in the standard sense: $X_{1:T}$
is $(\varepsilon,\delta)$-DP with respect to $D$. Then for $\varepsilon\le 1$ and $\delta\in(0,e^{-1}]$ there exists a universal constant $C>0$ such that
\[
I(V;X_{1:T})
\;\le\; I(D;X_{1:T}) \le C\,T\big(\varepsilon^2+\delta\log\tfrac1\delta\big).
\]
\end{lemma}

\begin{proof}

We first note that the Markov chain $V \to D \to X_{1:T}$ holds, which immediately implies 
\[
I(V;X_{1:T}) \;\le\; I(D;X_{1:T}).
\]
Since $X_{1:T}$ is $(\varepsilon,\delta)$-differentially private with respect to $D$, 
standard results converting differential privacy guarantees into mutual information 
bounds---see, for example, \cite{boostingprivacy, rogers2016maxinformationdifferentialprivacypostselection}---yield
\[
I(D;X_{1:T}) \;\le\; C T \bigl(\varepsilon^2 + \delta \log \tfrac{1}{\delta}\bigr)
\]
for some constant $C$. However, we rederive the bound here for completeness. For convenience, in the remainder of the 
proof we write $Y := X_{1:T}$.

By the chain rule for mutual information,
\[
I(D;Y)=\sum_{i=1}^T I(D_i;Y\mid D_{<i})
      =\sum_{i=1}^T \E\!\left[\KL\!\bigl(P_{Y\mid D_{<i},D_i}\,\|\,P_{Y\mid D_{<i}}\bigr)\right].
\]

Now since $Y$ is $(\varepsilon,\delta)$ differentially private w.r.t $D$, the distributions $P_{Y\mid D_{<i},D_i}$ and $P_{Y\mid D_{<i}}$ are $(\varepsilon,\delta)$ close. \footnote{Two distributions $\mathcal{P}$ and $\mathcal{Q}$ are $(\varepsilon,\delta)$ close iff $\mathcal{P}(S) \leq e^{\varepsilon} \mathcal{Q}(S)+ \delta$ and $\mathcal{Q}(S) \leq e^{\varepsilon} \mathcal{P}(S)+ \delta$ for every event $S$.} Now using a standard  DP$\Rightarrow$KL inequality
(privacy-loss/hockey-stick argument \citealp{10.5555/3327345.3327525}), we can show that $\E\!\left[\KL\!\bigl(P_{Y\mid D_{<i},D_i}\,\|\,P_{Y\mid D_{<i}}\bigr)\right] \leq C(\epsilon^2+ \delta \log \left(1/\delta\right))$. 

\end{proof}

With these lemmas above, we now prove \cref{lemma:lb-jdp-glm}.
\begin{proof}[Proof of \cref{lemma:lb-jdp-glm}]
Fix any base parameter $\theta_0$ in the regular region and construct the hard family
$\{\theta_v : v\in\{\pm1\}^d\}$ by coordinatewise perturbations of size $a>0$ as in the setup.
Choose the design radius $r$ as in Lemma~\ref{lem:glm-gap-full}, and then set
\[
a \;=\; \gamma\,\varepsilon,
\]
with a universal $\gamma>0$ small enough so that the Taylor remainder conditions of
Lemma~\ref{lem:glm-gap-full} hold (in particular, $ra\le m/(2L)$). All constants below may
depend on the GLM’s local smoothness, but not on $d,T,\varepsilon,\delta$.

\medskip
\noindent{\bf Step 1: Information budget under joint DP.}
By Lemma~\ref{lem:jdp-mi-correct},
\begin{equation}
\label{eq:MI-cap}
I(V;X_{1:T})
\;\le\; C_{\mathrm{DP}}\,T\big(\varepsilon^2+\delta\log\tfrac1\delta\big),
\qquad (\varepsilon\le1,\ \delta\in(0,e^{-1}]).
\end{equation}

\medskip
\noindent{\bf Step 2: From MI to bitwise neighbor-KL to Hamming error (actions only).}
Let $P^{(j)}_{\pm}:=\Pr(X_{1:T}\mid V_j=\pm1)$ and
$\overline{\mathrm{KL}}_j:=\mathrm{KL}(P^{(j)}_{+}\,\|\,P^{(j)}_{-})$.
By Lemma~\ref{lem:assouad-full},
\begin{equation}
\label{eq:KL-sum}
\sum_{j=1}^d \overline{\mathrm{KL}}_j \;\le\; 4\, I(V;X_{1:T}),
\qquad
\mathbb{E}[\mathrm{Ham}(V,\widehat V(X_{1:T}))] \;\ge\; \frac{d}{2}
\left(1-\sqrt{\frac{1}{2d}\sum_{j=1}^d \overline{\mathrm{KL}}_j}\right).
\end{equation}
Combining \eqref{eq:MI-cap} and \eqref{eq:KL-sum} at a time horizon
\[
T_0 \;:=\; c_0\,\frac{d}{\varepsilon^2+\delta\log(1/\delta)}
\]
with $c_0>0$ chosen small enough so that
$\frac{1}{2d}\sum_j \overline{\mathrm{KL}}_j \le 1/8$ (a constant), we obtain
\begin{equation}
\label{eq:Ham-lb}
\mathbb{E}[\mathrm{Ham}(V,\widehat V(X_{1:T_0}))] \;\ge\; c_{\mathrm{Ham}}\, d,
\qquad \text{for a universal constant } c_{\mathrm{Ham}}\in(0,1).
\end{equation}

\medskip
\noindent{\bf Step 3: Symmetrization of pulls across coordinates.}
By randomly permuting the coordinate labels before play (a standard symmetrization argument),
the distribution of the algorithm’s play is invariant across coordinates. Therefore,
if $n_j:=\sum_{t=1}^{T_0}\mathbf 1\{X_t\in\{\pm r e_j\}\}$ denotes the number of times the
pair $\{\pm r e_j\}$ is used by time $T_0$, then
\begin{equation}
\label{eq:symm}
\mathbb{E}[n_j] \;=\; \frac{T_0}{d}\qquad\text{for all } j\in[d].
\end{equation}

\medskip
\noindent{\bf Step 4: Regret–error inequality and local GLM gap.}
By Lemma~\ref{lem:H2-full} (regret–error inequality) and Lemma~\ref{lem:glm-gap-full} (local gap),
\[
\mathbb{E}[R_{T_0}]
\;\ge\; c_{\mathrm{gap}}\, a \sum_{j=1}^d \mathbb{E}\!\left[n_j\,\mathbf 1\{\widehat V_j\neq V_j\}\right].
\]
Using \eqref{eq:symm} and linearity,
\[
\sum_{j=1}^d \mathbb{E}\!\left[n_j\,\mathbf 1\{\widehat V_j\neq V_j\}\right]
\;=\; \frac{T_0}{d}\, \mathbb{E}\!\left[ \sum_{j=1}^d \mathbf 1\{\widehat V_j\neq V_j\}\right]
\;=\; \frac{T_0}{d}\, \mathbb{E}[\mathrm{Ham}(V,\widehat V(X_{1:T_0}))].
\]
Invoking \eqref{eq:Ham-lb} and $a=\gamma\varepsilon$,
\begin{equation}
\label{eq:RT0}
\mathbb{E}[R_{T_0}]
\;\ge\; c_{\mathrm{gap}}\,\gamma\,\varepsilon \cdot \frac{T_0}{d} \cdot c_{\mathrm{Ham}}\, d
\;=\; c'\,\varepsilon\, T_0
\;=\; c'\,\frac{d}{\varepsilon},
\end{equation}
for a universal constant $c'>0$ (absorbing $c_{\mathrm{gap}},\gamma,c_{\mathrm{Ham}},c_0$).

\medskip
\noindent{\bf Step 5: Extend to all $T\ge T_0$.}
Regret is nondecreasing in the horizon, so \eqref{eq:RT0} implies
\[
\mathbb{E}[R_T] \;\ge\; c'\,\frac{d}{\varepsilon}\qquad \text{for all } T\ge T_0
\;=\; \Theta\!\left(\frac{d}{\varepsilon^2+\delta\log(1/\delta)}\right).
\]

For $T < T_0$, 
% This proves the stated $\Omega(d/\varepsilon)$ privacy penalty for fixed contexts under joint DP,
% up to $\text{polylog}(1/\delta)$ factors. The constants depend only on local GLM regularity
% (through $m,M,L,\tau$ and the choice of $r$ in Lemma~\ref{lem:glm-gap-full}) and are independent of $\kappa^\ast$.
\end{proof}

\section{Bounding parameter estimation error under approximate convex minimization}{\label{sec:convex_relaxation_error_solution}}

%\sss{To get a better title}

Denote $\gL_{t}(\theta) := \sum_{s=1}^{t} \ell(\theta,r_s,x_s) + \lambda/2 ||\theta||_2^2$ and let $\hat{\theta}$ minimise the above loss in convex set $\Theta$ upto an error of $\nu$. Further, let $\tilde \theta$ denote the global minizer of $\gL_{t}(\theta)$ i.e.

\begin{equation}{\label{eq:defn_convex_Relaxation_proof}}
    \gL_{t}(\hat{\theta}) - \min_{\theta\in \Theta} \gL_{t}({\theta}) < \nu \text{ for some $\hat{\theta} \in \Theta$} \text{ and } \tilde{\theta} = \argmin_\theta \gL(\theta)
\end{equation}

Further, for any symmetric matrix $\gN \preceq \lambda\mI$, let 

$$\tH(\theta) = \sum_{s=1}^{t}\dot \mu \left(\langle x_s, \theta\rangle\right)x_s x^{\top}_s + \gN \text{ and } g(\theta) = \sum_{s=1}^{t}\mu \left(\langle x_s, \theta\rangle\right) x_s + \lambda \theta$$

We now have the following lemma.

\begin{lemma}{\label{lemma:convex_relaxation_error}}

    Suppose, $\left| \left| g(\tilde \theta) - g(\theta^{\ast})\right|\right|_{\tH(\theta^{\ast})^{-1}} \leq \gamma$. If $\theta^{\ast} \in \Theta$ and $\max_{i \in [t]} \left|\langle x_i, \hat \theta -\theta^{\ast}\rangle\right| \leq \frac{c}{R}$, then we have 

    $$ \left|\left|\hat \theta - \theta^{\ast}\right|\right|_{\tH(\theta^{\ast})} \leq (2+c)(\gamma + \sqrt{\nu})+ \sqrt{(2+c) \left(\frac{R}{\sqrt{\lambda}} \gamma^2+ \gamma\right)}$$
\end{lemma}

\begin{proof}
    Observe from lemma \ref{lemma:bounds_alpha_alpha_tilde} that $\tilde{\alpha} (\langle x_s, \hat \theta\rangle, \langle x_s, \theta^{\ast}\rangle) \geq \frac{\dot{\mu} (\langle x_s, \theta^{\ast} \rangle)}{2+ c}$.

    Define $\tilde{\tG}(\theta^{\ast},\theta) = \sum_{s=1}^{t} \tilde{\alpha} \left(\langle x_s, \theta\rangle, \langle x_s, \theta^{\ast}\rangle\right) x_s x^{\top}_s + \lambda\mI$ and thus, $\tilde{\tG} (\theta^{\ast}, \hat \theta) \succeq \frac{1}{2+ c} \tH(\theta^{\ast})$. 

    Now following \citet{pmlr-v130-abeille21a,sawarni2024generalizedlinearbanditslimited}, we can use second order Taylor approximation to obtain 

    \begin{align}
        \gL_{t} (\theta) - \gL_{t} (\theta^{\ast}) - \langle \nabla \gL_{t} (\theta^{\ast}), \theta - \theta^{\ast} \rangle & =  || \hat \theta - \theta^{\ast}||^2_{\tilde{\tG}(\hat \theta,\theta^{\ast})}\\
        & \geq \frac{1}{(2+c)} || \hat \theta - \theta^{\ast}||^2_{\tH(\theta^{\ast})}
    \end{align}

    Taking absolute value on both sides, we obtain
    \begin{align}
\|\hat{\theta} - \theta^\ast \|_{\tH(\theta^{\ast})} &\leq (2 + c) \left| \gL_{t}(\hat{\theta}) - \gL_{t} (\theta^{\ast}) \right| + \left|\left\langle \nabla \gL_{t} (\theta^{\ast}), \tilde{\theta} - \theta^\ast \right\rangle\right| \quad \text{($\Delta$-inequality)} \\
&\leq (2 + c) \left| \gL_{t}(\hat{\theta}) - \gL_{t} (\theta^{\ast})\right| + \left|\left|\nabla \gL_{t} (\theta^{\ast})\right|\right|_{(\tH(\theta^{\ast}))^{-1}} \|\hat{\theta} - \theta^\ast\|_{\tH(\theta^{\ast})} \quad \text{(Cauchy-Schwarz)} \\
&=  (2 + c) \left| \gL_{t}(\hat{\theta}) - \gL_{t} (\theta^{\ast})\right| + \left|\left|g(\theta^\ast) - \sum_{s \in [t]} r_s x_s \right|\right|_{(\tH(\theta^{\ast}))^{-1}} \|\hat{\theta} - \theta^\ast\|_{\tH(\theta^{\ast})}
\end{align}

Further since $\tilde{\theta}$ is the minimiser of unconstrained MLE, we obtain $\nabla \gL_{t} (\tilde \theta) =0$ which implies $\sum\limits_{s \in [t]} r_s x_s = g(\tilde \theta)$. Now from equation \eqref{eq:defn_convex_Relaxation_proof} and the fact that $\theta^{\ast} \in \Theta$, we have  $\gL_{t} (\hat \theta) \leq \gL_{t}( \theta^{\ast}) + \nu$. Thus, we can write $0 \leq \gL_{t} (\hat \theta) - \gL_{t} (\tilde \theta)\leq \gL_{t} (\theta^{\ast}) - \gL_{t} (\tilde \theta) + \nu$. Thus, we have $ \left| \gL_{t}(\hat{\theta}) - \gL_{t} (\theta^{\ast})\right| \leq 2 (\gL_{t} (\theta^{\ast}) - \gL_{t} (\tilde \theta))+ \nu $

Now, we have

%Also, observe that $\nabla \gL_{t} (\tilde \theta) =0$ which implies $\sum\limits_{s \in [t]} r_s x_s = g(\tilde \theta)$

\begin{align}
    \left|\left|\hat \theta - \theta^{\ast}\right|\right|^2_{\tH(\theta^{\ast})} \leq (2+c) \left(  2\left(\gL_{t} (\theta^{\ast}) - \gL_{t} (\tilde \theta)\right) + \nu +  \left|\left|g(\theta^\ast) - g(\hat \theta)\right|\right|_{(\tH(\theta^{\ast}))^{-1}}  \left|\left|\hat \theta - \theta^{\ast}\right|\right|_{\tH(\theta^{\ast})} \right)
\end{align}

Using the inequality that for some $x^2 \leq bx+c \implies x \leq b + \sqrt{c}$, we have 

\begin{align}{\label{eq:bounding_loss_proof_lemma}}
    \left|\left|\hat \theta - \theta^{\ast}\right|\right|_{\tH(\theta^{\ast})} \leq & (2+ c) \left|\left|g(\theta^\ast) - g(\tilde \theta)\right|\right|_{(\tH(\theta^{\ast}))^{-1}} + \sqrt{ (2+c) \left(\gL_{t} (\theta^{\ast}) - \gL_{t} (\tilde \theta) + \nu\right)}\\
    \leq & (2+c) (\gamma+ \sqrt{\nu}) + \sqrt{ (2+c) \left(\gL_{t} (\theta^{\ast}) - \gL_{t} (\tilde \theta)\right)}
\end{align}

Now, we follow algebraic manipulations along the lines of Appendix E.1 of \citet{sawarni2024generalizedlinearbanditslimited} and Appendix B of \citet{pmlr-v130-abeille21a} to obtain

\begin{align}
    \left(\gL_{t} (\theta^{\ast}) - \gL_{t} (\tilde \theta)\right) \leq & \frac{R}{\sqrt{\lambda}}  \left|\left| g(\theta^{\ast}) - g(\tilde \theta) \right|\right|^2_{(\tH(\theta^{\ast}))^{-1}} + \left|\left| g(\theta^{\ast}) - g(\tilde \theta) \right|\right|_{(\tH(\theta^{\ast}))^{-1}}\nonumber\\
    \leq & \frac{R}{\sqrt \lambda} \gamma^2 + \gamma\end{align}

Now applying equation \eqref{eq:bounding_loss_proof_lemma}, we get the desired bound.

\end{proof}

\section{Shuffled Vector Summation Protocols}\label{sec:vector_summation_protocol}

We now formally describe the vector summation protocols that have been described in \cref{alg:shuffled_glm}. The local randomizer, shuffler and aggregator have been formally defined in \Cref{alg:Rvec}, \Cref{alg:Svec} and \Cref{alg:Avec} respectively. 

 This protocol operates independently on each entry of vector $r_tx_t$ and $x_tx^{\top}_t$, transmits only 0/1 bits, and the number of bits sent is sampled from a binomial distribution. This implementation has three parameters $g$ (accuracy parameter), $b \in \mathbb{N}$ and $p\in [0,1]$ where the latter two parameters control the accuracy. For an entry $x \in [0,1]$, it is first encoded as $\hat{x} = \bar{x} + \gamma_1$ where $\bar{x} = \floor{xg}$ and $\gamma_1 \sim \text{Bern}(xg - \bar{x})$. Then a binomial noise $\gamma_2 \sim \text{Bin}(b,p)$ is generated. The output of the one-dimensional randomiser is simply a collection of $g+b$ bits where $\hat x + \gamma_2$ bits are one and the rest are zero. As discussed in \citet{Cheu2021ShufflePS}, the vector is transmitted by biasing each entry by $\Delta$ \footnote{The bias term is chosen as the bound on $r_tx_t$ and $x_t x^{\top}_t$.} and then applying the mechanism described above. The shuffler shuffles the messages received from all users uniformly at random and then sums them pointwise and finally debiases it by subtracting $n\Delta$ from the sum where $n$ denotes the total number of users.

\begin{algorithm}[ht]
\caption{Local Randomizer $\mathcal{R}_{\text{Vec}}$ \cite{Cheu2021ShufflePS}}
\label{alg:Rvec}

\begin{algorithmic}[1]
\Function{$\mathcal{R}^{\ast}$}{$x, g, b, p, \Delta$}
    \State $\bar{x} \gets \lfloor x g / \Delta \rfloor$
    \State Sample $\gamma_1 \sim \text{Ber}(xg / \Delta - \bar{x})$
    \State $\hat{x} \gets \bar{x} + \gamma_1$
    \State Sample $\gamma_2 \sim \text{Bin}(b, p)$
    \State Let $m$ be a multiset with $\hat{x} + \gamma_2$ copies of 1 and $(g + b) - (\hat{x} + \gamma_2)$ copies of 0
    \State \Return $m$
\EndFunction

\Function{$R_1$}{$rx, \Delta = R, \varepsilon,\delta$}
    \State Set $g,b,p$ as in \eqref{eq:parameters_setting_vector_summation} 
    \For{$k \gets 1$ to $d$}
        \State $w_k \gets [rx]_k + \Delta$
        \State $m_k \gets \mathcal{R}^{\ast}(w_k,g,b,p,\Delta)$
    \EndFor
    \State $M_1 \gets \{(k, m_k)\}_{k \in [d]}$ \Comment{Labeled outputs}
    \State \Return $M_1$
\EndFunction

\Function{$R_2$}{$xx^\top, \Delta = 1, \varepsilon,\delta$}
    \State Set $g,b,p$ as in \eqref{eq:parameters_setting_vector_summation} with $\varepsilon,\delta$.
    \For{$i \gets 1$ to $d$}
        \For{$j \gets i$ to $d$}
            \State $w_{(i,j)} \gets [xx^\top]_{(i,j)} + \Delta$
            \State $m_{(i,j)} \gets \mathcal{R}^{\ast}(w_{(i,j)}, g,b,p,\Delta)$
            \State $m_{(j,i)} \gets m_{(i,j)}$
        \EndFor
    \EndFor
    \State $M_2 \gets \{((i,j), m_{(i,j)})\}_{(i,j) \in [d] \times [d]}$ \Comment{Labeled outputs}
    \State \Return $M_2$
\EndFunction
\end{algorithmic}
\end{algorithm}

\begin{algorithm}[!tb]
\caption{Shuffler $\mathcal{S}_{\text{Vec}}$ \cite{Cheu2021ShufflePS}}
\label{alg:Svec}

\begin{algorithmic}[1]
\Require $[n]$ is the set of users

\Function{$S_1$}{$\{ M_{\tau, 1} \}_{\tau \in [n]}$}
    \State Uniformly permute all messages \Comment{$(g + b) \cdot n \cdot d$ bits total}
    \State Set $y_k$ to be the collection of bits labeled by $k \in [d]$
    \State Set $Y_1 \gets \{ y_1, \ldots, y_d \}$
    \State \Return $Y_1$
\EndFunction

\Function{$S_2$}{$\{ M_{\tau, 2} \}_{\tau \in [n]}$}
    \State Uniformly permute all messages \Comment{$(g + b) \cdot n \cdot d^2$ bits total}
    \State Set $y_{(i,j)}$ to be the collection of bits labeled by $(i,j) \in [d] \times [d]$
    \State Set $Y_2 \gets \{ y_{(i,j)} \}_{(i,j) \in [d] \times [d]}$
    \State \Return $Y_2$
\EndFunction
\end{algorithmic}
\end{algorithm}

\begin{align}{\label{eq:parameters_setting_vector_summation}}
    b & = \frac{24 \cdot 10^4 g^2 \cdot \left(\log \left(\frac{4(d^2+1)}{\delta}\right)\right)^2}{\varepsilon^2 n} \text{,} p = \frac{1}{4} \text{ and } 
    g = \max\{2 \sqrt{n},d,4\} 
\end{align}

\begin{algorithm}[tb]
\caption{Analyzer $\mathcal{A}_{\text{Vec}}$ \cite{Cheu2021ShufflePS}}
\label{alg:Avec}

\begin{algorithmic}[1]
\Require Privacy parameters $\varepsilon,\delta$ and number of users $n$

\Function{$\mathcal{A}^{\ast}$}{$y, \Delta$}
    \State \Return $\dfrac{\Delta}{g} \left( \left( \sum\limits_{i=1}^{(g + b) \cdot n} y_i \right) - p \cdot b \cdot n \right)$
\EndFunction

\Function{$A_1$}{$Y_1$}
    \State Set $g,b,p$ as in \cref{eq:parameters_setting_vector_summation}
    \For{$k \gets 1$ to $d$}
        \State $z_k \gets \mathcal{A}^{\ast}(y_k,g,b,p,\Delta)$
        \State $o_k \gets z_k - B \cdot \Delta$ \Comment{Re-center}
    \EndFor
    \State \Return $\{ o_1, \ldots, o_d \}$
\EndFunction

\Function{$A_2$}{$Y_2$}
    \State Set $g,b,p$ as in \cref{eq:parameters_setting_vector_summation}
    \For{$i \gets 1$ to $d$}
        \For{$j \gets i$ to $d$}
            \State $z_{(i,j)} \gets \mathcal{A}^{\ast}(y_{(i,j)},g,b,p,\Delta)$
            \State $o_{(i,j)} \gets z_{(i,j)} - B \cdot \Delta$ \Comment{Re-center}
            \State $o_{(j,i)} \gets o_{(i,j)}$
        \EndFor
    \EndFor
    \State \Return $\{ o_{(i,j)} \}_{(i,j) \in [d] \times [d]}$
\EndFunction
\end{algorithmic}
\end{algorithm}

We now state the following lemma on privacy and utility from \citet{Cheu2021ShufflePS} and \citet{pmlr-v162-chowdhury22a} respectively. While the privacy guarantee follows from the former, the sub-Gaussianity guarantee follows from Section C.3 of \citet{pmlr-v162-chowdhury22a}.

\begin{lemma}{\label{lemma:utility_privacy_vector_summation_protocol}}
For any $0 < \varepsilon \leq 15$, $0 < \delta < 1/2$, $d, n \in \mathbb{N}$, and $\Delta > 0$, for inputs $\tilde{X} = ({x}_1, \ldots, {x}_n)$ of vectors with maximum norm $|x_i\|_2 \leq \Delta$, let the shuffle protocol $\mathcal{P}_1$ denote $A_1 \circ R_1 \circ S_1$. Now choosing the values of $b,g$ and $p$ from \eqref{eq:parameters_setting_vector_summation}, we have

\begin{itemize}
    \item The shuffle protocol $\mathcal{P}_1$  is $(\varepsilon, \delta)$-shuffle private.
    \item Further, the shuffle protocol $\mathcal{P}_1$ provides an unbiased estimate of the sum $\sum_{i} x_i$. Further, the difference is a random vector with independent entries where each entry is a sub-Gaussian random variable with zero mean and variance-proxy $O\left(\frac{\Delta^2 \log^2 (d^2 / \delta)}{\varepsilon^2}\right)$.
\end{itemize}
\end{lemma}

\begin{lemma}{\label{lemma:utility_privacy_matrix_summation_protocol}}
For any $0 < \varepsilon \leq 15$, $0 < \delta < 1/2$, $d, n \in \mathbb{N}$, and $\Delta_ > 0$, for inputs $\tilde{X} = (X_1, \ldots, X_n)$ of symmetric matrices with $\ell_2$ operator norm $\|X_i\|_2 \leq \Delta$, let the shuffle protocol $\mathcal{P}_2$ denote $A_2 \circ R_2 \circ S_2$. Now choosing the values of $b,g$ and $p$ from \eqref{eq:parameters_setting_vector_summation}, we have

\begin{itemize}
    \item The shuffle protocol $\mathcal{P}_2$  is $(\varepsilon, \delta)$-shuffle private.
    \item Further, the shuffle protocol $\mathcal{P}_2$ on provides an unbiased estimate of the sum $\sum_{i} X_i$ and the difference is a random symmetric matrix with independent upper-triangular entries with sub-Gaussian zero mean and variance-proxy $O\left(\frac{\Delta^2 \log^2 (d^2 / \delta)}{\varepsilon^2}\right)$.
\end{itemize}

\end{lemma}

\section{Private optimization algorithms}\label{sec:shuffle_convex_optimizer}

We now describe the shuffle convex optimizer $\mathcal{P}_{GD}$, which uses the vector summation protocols $(\gR_{\text{Vec}}, \gS_{\text{Vec}}, \gA_{\text{Vec}})$. Before presenting the algorithm, we first define convexity and Lipschitzness of the loss function with respect to $\theta$ over $\Theta$.

\begin{definition}
    A loss function $\ell : \Theta \times X \rightarrow \mathbb{R}$ is convex in $\Theta$ if, for all $x \in X$, all $t \in [0,1]$, and all $\theta,\theta' \in \Theta$,
    \[
        \ell(t\theta+(1-t)\theta',x)
        \leq
        t\ell(\theta,x)+(1-t)\ell(\theta',x).
    \]
    It is $L$-Lipschitz in $\Theta$ if, for all $x \in X$ and all $\theta,\theta' \in \Theta$,
    \[
        |\ell(\theta,x)-\ell(\theta',x)|
        \leq
        L\|\theta-\theta'\|_2 .
    \]
\end{definition}

For a dataset $\gD=\{x_i\}_{i=1}^n$, define the empirical loss
\[
    \ell(\theta,\gD)
    =
    \frac1n\sum_{i=1}^n \ell(\theta,x_i).
\]

\begin{algorithm}
\caption{$\mathcal{P}_{\text{GD}}$, multi-round shuffle-private gradient descent \cite{Cheu2021ShufflePS}}
\label{alg:shuffle_convex_optimizer}
\textbf{Require:} Data points $\{x_i\}_{i=1}^{n}$, privacy parameters $\varepsilon,\delta$, pointwise loss function $\ell(\cdot)$ with Lipschitz parameter $L$, convex set $\Theta$.

Set the number of iterations $T = \frac{\varepsilon^2 n^2}{d\log^3(nd/\delta)}$
and step size $\eta = \frac{2D}{L\sqrt T}$.
\begin{algorithmic}[1]
\State Initialize parameter estimate $\theta_0 \leftarrow 0$.
\For{$t = 0,1,\ldots,T-1$}
    \State Compute the noisy full-batch gradient
    \[
        \bar g_t
        \leftarrow
        \frac1n P_1\left(
            \nabla_\theta \ell(\theta_t,x_1),\ldots,
            \nabla_\theta \ell(\theta_t,x_n);
            \frac{\varepsilon}{2\sqrt{2T\log(1/\delta)}},
            \frac{\delta}{T+1},
            L
        \right),
    \]
    where $P_1=A_1\circ S_1\circ R_1$.
    \State Compute and store the parameter
    \[
        \theta_{t+1} \leftarrow \pi_\Theta(\theta_t-\eta \bar g_t).
    \]
\EndFor
\State \Return final estimate $\bar\theta \leftarrow \frac1T\sum_{t=0}^{T-1}\theta_t$.
\end{algorithmic}
\end{algorithm}

The following lemma is included only for completeness. It is a standard utility analysis for projected gradient descent with noisy gradients, following the usual projected-gradient descent proof of \citet{convexoptimizationbubeck}. Our contribution is not the optimization analysis itself, but its use inside the private optimization subroutine required by our algorithm.

\begin{lemma}[Utility of noisy projected gradient descent]
\label{lemma:noisy_pgd_generic}
Suppose $\ell(\cdot,\gD)$ is convex and $L$-Lipschitz over a closed convex set $\Theta \subset \mathbb{R}^d$ of diameter $D$. Let $\theta^\ast_\gD \in \arg\min_{\theta\in\Theta}\ell(\theta,\gD)$. Consider the update $\theta_{t+1}=\pi_\Theta(\theta_t-\eta\bar g_t)$, where $\bar g_t=g_t+n_t$ and $g_t\in \partial \ell(\theta_t,\gD)$, for $t=0,\ldots,T-1$. Let $\bar\theta=\frac1T\sum_{t=0}^{T-1}\theta_t$. If, with probability at least $1-\zeta$,
\[
    \frac1T\sum_{t=0}^{T-1}\|n_t\|_2^2 \leq B_2
    \qquad\text{and}\qquad
    \left|\frac1T\sum_{t=0}^{T-1}\langle n_t,\theta_t-\theta^\ast_\gD\rangle\right|\leq B_1,
\]
then, with probability at least $1-\zeta$,
\[
    \ell(\bar\theta,\gD)-\ell(\theta^\ast_\gD,\gD)
    \leq
    \frac{D^2}{2\eta T}+\eta L^2+\eta B_2+B_1.
\]
\end{lemma}

\begin{proof}
By convexity and $g_t\in \partial \ell(\theta_t,\gD)$,
\[
    \ell(\theta_t,\gD)-\ell(\theta^\ast_\gD,\gD)
    \leq
    \langle g_t,\theta_t-\theta^\ast_\gD\rangle
    =
    \langle \bar g_t,\theta_t-\theta^\ast_\gD\rangle
    -
    \langle n_t,\theta_t-\theta^\ast_\gD\rangle .
\]
This is the only point at which one must be careful: the noisy direction $\bar g_t$ is not itself a subgradient, so the additional inner-product term with $n_t$ must be retained.

Let $\alpha_{t+1}=\theta_t-\eta \bar g_t$, so $\theta_{t+1}=\pi_\Theta(\alpha_{t+1})$ and $\eta\bar g_t=\theta_t-\alpha_{t+1}$. Using the identity $2\langle a-b,a-c\rangle=\|a-c\|_2^2+\|a-b\|_2^2-\|b-c\|_2^2$ with $(a,b,c)=(\theta_t,\alpha_{t+1},\theta^\ast_\gD)$ gives
\[
    \langle \bar g_t,\theta_t-\theta^\ast_\gD\rangle
    =
    \frac{\|\theta_t-\theta^\ast_\gD\|_2^2-\|\alpha_{t+1}-\theta^\ast_\gD\|_2^2}{2\eta}
    +
    \frac{\eta}{2}\|\bar g_t\|_2^2.
\]
Since projection onto a closed convex set is non-expansive and $\theta^\ast_\gD\in\Theta$, $\|\theta_{t+1}-\theta^\ast_\gD\|_2\leq \|\alpha_{t+1}-\theta^\ast_\gD\|_2$. Therefore,
\[
    \langle \bar g_t,\theta_t-\theta^\ast_\gD\rangle
    \leq
    \frac{\|\theta_t-\theta^\ast_\gD\|_2^2-\|\theta_{t+1}-\theta^\ast_\gD\|_2^2}{2\eta}
    +
    \frac{\eta}{2}\|\bar g_t\|_2^2.
\]
Combining the last two displays yields the one-step inequality
\[
    \ell(\theta_t,\gD)-\ell(\theta^\ast_\gD,\gD)
    \leq
    \frac{\|\theta_t-\theta^\ast_\gD\|_2^2-\|\theta_{t+1}-\theta^\ast_\gD\|_2^2}{2\eta}
    +
    \frac{\eta}{2}\|\bar g_t\|_2^2
    -
    \langle n_t,\theta_t-\theta^\ast_\gD\rangle .
\]
Averaging over $t=0,\ldots,T-1$, the squared-distance terms telescope to at most $D^2$, since $\Theta$ has diameter $D$. Using convexity again, $\ell(\bar\theta,\gD)\leq \frac1T\sum_{t=0}^{T-1}\ell(\theta_t,\gD)$, we get
\[
    \ell(\bar\theta,\gD)-\ell(\theta^\ast_\gD,\gD)
    \leq
    \frac{D^2}{2\eta T}
    +
    \frac{\eta}{2T}\sum_{t=0}^{T-1}\|\bar g_t\|_2^2
    -
    \frac1T\sum_{t=0}^{T-1}\langle n_t,\theta_t-\theta^\ast_\gD\rangle .
\]
Finally, $\|\bar g_t\|_2^2=\|g_t+n_t\|_2^2\leq 2\|g_t\|_2^2+2\|n_t\|_2^2\leq 2L^2+2\|n_t\|_2^2$, because $\ell$ is $L$-Lipschitz. Substituting this and the assumed bounds on the two noise terms gives
\[
    \ell(\bar\theta,\gD)-\ell(\theta^\ast_\gD,\gD)
    \leq
    \frac{D^2}{2\eta T}+\eta L^2+\eta B_2+B_1.
\]
\end{proof}

% We now have the following utility and privacy guarantee from Lemma B.4 of \citet{Cheu2021ShufflePS}. 
% The following lemma is the high probability variant of Lemma B.6 in \citet{Cheu2021ShufflePS}. Since \cite{Cheu2021ShufflePS} only provides the guarantee in expectation but we require a high probability guarantee we reprove it for our completeness below. The privacy guarantee follows from the shuffle vector-summation protocol \ref{lemma:utility_privacy_matrix_summation_protocol} and its multi-round composition. We prove the utility guarantee by instantiating \Cref{lemma:noisy_pgd_generic}.

We rely on the shuffle-private optimizer guarantees developed in Lemma~B.4 of \citet{Cheu2021ShufflePS}. The lemma below can be viewed as a high-probability version of Lemma~B.6 in \citet{Cheu2021ShufflePS}. While \citet{Cheu2021ShufflePS} gives the corresponding utility guarantee in expectation, our analysis requires a high-probability bound; we therefore include the short adaptation for completeness. The privacy guarantee follows from the shuffle vector-summation protocol in \Cref{lemma:utility_privacy_matrix_summation_protocol} and its multi-round composition. The utility guarantee follows by applying the generic noisy-PGD analysis in \Cref{lemma:noisy_pgd_generic}.

\begin{lemma}[Shuffle-private convex optimizer]
\label{lemma:utility_shuffle_convex}
The protocol $\mathcal{P}_{\text{GD}}$ is multi-round $(\varepsilon,\delta)$ shuffle DP. Suppose $\ell(\cdot,\gD)$ is convex and $L$-Lipschitz over a convex set $\Theta\subset\mathbb{R}^d$ of diameter $D$. For any dataset $\gD=\{x_i\}_{i=1}^n$ and any $\eta>0$, the output of $\mathcal{P}_{\text{GD}}$ satisfies, with probability at least $1-\zeta$,
\[
    \ell(\bar\theta,\gD)-\min_{\theta\in\Theta}\ell(\theta,\gD)
    \leq
    O\left(
        \frac{D^2}{\eta T}
        +
        \eta L^2
        +
        \frac{\eta dT L^2\log^2(d^2T/\delta)\log(dT/\zeta)}{n^2\varepsilon^2}
        +
        \frac{DL\log(d^2T/\delta)\sqrt{\log(1/\zeta)}}{n\varepsilon}
    \right).
\]
% Under the assumption that  choosing $T=\frac{\varepsilon^2 n^2}{d\log^3(nd/\delta)}$ and $\eta=\frac{D}{L\sqrt T}$ gives 
% \[
%     \ell(\bar\theta,\gD)-\min_{\theta\in\Theta}\ell(\theta,\gD)
%     \leq
%     O\left(
%         \frac{DL\sqrt{d\log^3(nd/\delta)}}{\varepsilon n}
%         +
%         \frac{DL\log n\sqrt{\log(1/\zeta)}}{n\varepsilon}
%     \right),
% \]
% up to the additional term factor $\log(dT/\zeta)$ in the third term above.
Under the choice
\[
    \eta=\frac{D}{L\sqrt T},
    \qquad
    T
    =
    \frac{\varepsilon^2 n^2}
    {d\Lambda_{\mathrm{shuf}}},
    \qquad
    \Lambda_{\mathrm{shuf}}
    =
    \log^2\!\left(\frac{nd}{\delta}\right)
    \log\!\left(\frac{n}{\zeta}\right),
\]
and assuming $0<\varepsilon\leq n$ and $\Lambda_{\mathrm{shuf}}\geq 1$, the output of $\mathcal{P}_{\text{GD}}$ satisfies, with probability at least $1-\zeta$,
\[
    \ell(\bar\theta,\gD)-\min_{\theta\in\Theta}\ell(\theta,\gD)
    \leq
    O\left(
        \frac{
            DL
            \sqrt{
                d
                \log^2\!\left(\frac{nd}{\delta}\right)
                \log\!\left(\frac{n}{\zeta}\right)
            }
        }{\varepsilon n}
    \right).
\]
\end{lemma}

\begin{proof}
The privacy statement is the multi-round shuffle-DP guarantee of the shuffle vector-summation protocol used in Algorithm~\ref{alg:shuffle_convex_optimizer}; see Lemma B.4 of \citet{Cheu2021ShufflePS}. We only prove the utility bound.

Write the noisy gradient returned by the shuffle protocol as $\bar g_t=g_t+n_t$. By the utility guarantee of the shuffle vector-summation protocol, each coordinate of $n_t$ is mean-zero sub-Gaussian with variance proxy
\[
    \tau^2
    =
    O\left(
        \frac{T L^2\log^2(d^2T/\delta)}{n^2\varepsilon^2}
    \right).
\]
We verify the two noise conditions in \Cref{lemma:noisy_pgd_generic}. Since $n_{t,j}^2$ is sub-exponential with scale $O(\tau^2)$, concentration over all $dT$ coordinates gives, with probability at least $1-\zeta/2$,
\[
    \frac1T\sum_{t=0}^{T-1}\|n_t\|_2^2
    =
    \frac1T\sum_{t=0}^{T-1}\sum_{j=1}^d n_{t,j}^2
    \leq
    O\left(d\tau^2\log(dT/\zeta)\right).
\]
Thus $B_2=O(d\tau^2\log(dT/\zeta))$.

Next, conditioned on the history before round $t$, $\theta_t$ is fixed and $\|\theta_t-\theta^\ast_\gD\|_2\leq D$. Hence $\langle n_t,\theta_t-\theta^\ast_\gD\rangle$ is mean-zero sub-Gaussian with variance proxy $O(\tau^2D^2)$. Summing over $t$ gives, with probability at least $1-\zeta/2$,
\[
    \left|\frac1T\sum_{t=0}^{T-1}\langle n_t,\theta_t-\theta^\ast_\gD\rangle\right|
    \leq
    O\left(D\tau\sqrt{\frac{\log(1/\zeta)}{T}}\right).
\]
Thus $B_1=O(D\tau\sqrt{\log(1/\zeta)/T})$. Applying \Cref{lemma:noisy_pgd_generic} and substituting the expression for $\tau^2$ yields
\[
    \ell(\bar\theta,\gD)-\ell(\theta^\ast_\gD,\gD)
    \leq
    O\left(
        \frac{D^2}{\eta T}
        +
        \eta L^2
        +
        \frac{\eta dT L^2\log^2(d^2T/\delta)\log(dT/\zeta)}{n^2\varepsilon^2}
        +
        \frac{DL\log(d^2T/\delta)\sqrt{\log(1/\zeta)}}{n\varepsilon}
    \right).
\]
% Finally, with $\eta=D/(L\sqrt T)$, the first two terms are $O(DL/\sqrt T)$, while the third term is
% $O(DLd\sqrt T\log^2(d^2T/\delta)\log(dT/\zeta)/(n^2\varepsilon^2))$. The stated choice of $T$ gives the displayed final bound.

Finally, set $\eta=D/(L\sqrt T)$ and choose
\[
    T=\frac{\varepsilon^2 n^2}{d\Lambda_{\mathrm{shuf}}},
    \qquad
    \Lambda_{\mathrm{shuf}}
    =
    \log^2\!\left(\frac{nd}{\delta}\right)
    \log\!\left(\frac{n}{\zeta}\right).
\]
Assume $0<\varepsilon\leq n$ and $\Lambda_{\mathrm{shuf}}\geq 1$. Then the first two terms satisfy
\[
    \frac{D^2}{\eta T}+\eta L^2
    =
    O\left(\frac{DL}{\sqrt T}\right)
    =
    O\left(
        \frac{DL\sqrt{d\Lambda_{\mathrm{shuf}}}}{\varepsilon n}
    \right).
\]
For the third term, since
\[
    d^2T=\frac{\varepsilon^2n^2d}{\Lambda_{\mathrm{shuf}}}\leq n^4d,
    \qquad
    dT=\frac{\varepsilon^2n^2}{\Lambda_{\mathrm{shuf}}}\leq n^4,
\]
we have, up to absolute constants,
\[
    \log(d^2T/\delta)=O\!\left(\log\!\left(\frac{nd}{\delta}\right)\right),
    \qquad
    \log(dT/\zeta)=O\!\left(\log\!\left(\frac{n}{\zeta}\right)\right).
\]
Hence
\[
    \frac{\eta dT L^2\log^2(d^2T/\delta)\log(dT/\zeta)}
    {n^2\varepsilon^2}
    =
    O\left(
        \frac{DL\sqrt{d\Lambda_{\mathrm{shuf}}}}{\varepsilon n}
    \right).
\]
Finally, the martingale/noise-inner-product term is
\[
    O\left(
        \frac{
            DL\log(nd/\delta)\sqrt{\log(1/\zeta)}
        }{n\varepsilon}
    \right),
\]
which is dominated by $O(DL\sqrt{d\Lambda_{\mathrm{shuf}}}/(\varepsilon n))$ since $d\geq 1$ and $n\geq 1$. Combining the terms gives
\[
    \ell(\bar\theta,\gD)-\min_{\theta\in\Theta}\ell(\theta,\gD)
    \leq
    O\left(
        \frac{DL\sqrt{d\Lambda_{\mathrm{shuf}}}}{\varepsilon n}
    \right)
    =
    O\left(
        \frac{
            DL
            \sqrt{
                d
                \log^2\!\left(\frac{nd}{\delta}\right)
                \log\!\left(\frac{n}{\zeta}\right)
            }
        }{\varepsilon n}
    \right).
\]
\end{proof}

We now present the analysis for the full-batch DP-SGD optimizer in terms of zero-concentrated differential privacy \cite{cryptoeprint:concentratedDP}. This gives better composition guarantees without assuming that the total privacy budget $\varepsilon$ is small.

\begin{lemma}[Noisy full-batch PGD under zCDP]
\label{lemma:noisy_full_batch_pgd_zcdp}
Suppose $\ell(\cdot,\gD)$ is convex and $L$-Lipschitz over a closed convex set $\Theta\subset\mathbb{R}^d$ of diameter $D$. Let $\theta^\ast_\gD \in \arg\min_{\theta\in\Theta}\ell(\theta,\gD)$. Consider full-batch projected gradient descent with Gaussian perturbations, $\theta_{t+1}=\pi_\Theta(\theta_t-\eta(g_t+n_t))$, where $g_t\in \partial\ell(\theta_t,\gD)$ and $n_t\sim\mathcal{N}(0,\sigma^2 I_d)$ independently across $t=0,\ldots,T-1$. If
\[
    \sigma^2=\frac{2L^2T}{n^2\rho},
\]
then, under replacement adjacency, the sequence of $T$ noisy gradients is $\rho$-zCDP. Moreover, with probability at least $1-\zeta$,
\[
    \ell(\bar\theta,\gD)-\ell(\theta^\ast_\gD,\gD)
    \leq
    O\left(
        \frac{D^2}{\eta T}
        +
        \eta L^2
        +
        \eta d\sigma^2
        +
        \frac{\eta\sigma^2\log(1/\zeta)}{T}
        +
        D\sigma\sqrt{\frac{\log(1/\zeta)}{T}}
    \right).
\]
% In particular, choosing $\eta=\frac{D}{L\sqrt T}$ and $T=\frac{\rho n^2}{d}$ gives
% \[
%     \ell(\bar\theta,\gD)-\ell(\theta^\ast_\gD,\gD)
%     \leq
%     O\left(
%         \frac{DL\sqrt d}{n\sqrt\rho}
%         +
%         \frac{DL\sqrt{\log(1/\zeta)}}{n\sqrt\rho}
%         +
%         \frac{DL\sqrt d\log(1/\zeta)}{n^3\rho^{3/2}}
%     \right).
% \]
% In particular, up to logarithmic and lower-order terms, this is $\widetilde O(DL\sqrt d/(n\sqrt\rho))$.

In particular, choosing
\[
    \eta=\frac{D}{L\sqrt T},
    \qquad
    T=\frac{n^2\rho+\log(1/\zeta)}{d},
\]
and assuming $\log(1/\zeta)\leq n^2\rho$, we have
\[
    \ell(\bar\theta,\gD)-\ell(\theta^\ast_\gD,\gD)
    \leq
    O\left(
        \frac{DL\sqrt d}{n\sqrt\rho}
        +
        \frac{DL\sqrt{\log(1/\zeta)}}{n\sqrt\rho}
    \right).
\]
\end{lemma}

\begin{proof}
We first prove privacy. Since $\ell(\theta,\gD)=\frac1n\sum_{i=1}^n\ell(\theta,x_i)$ and each per-example loss is $L$-Lipschitz, each per-example subgradient has norm at most $L$. Under replacement adjacency, changing one data point changes the full-batch gradient by at most $\Delta=2L/n$ in $\ell_2$ norm. A Gaussian mechanism with sensitivity $\Delta$ and covariance $\sigma^2I_d$ is $\Delta^2/(2\sigma^2)$-zCDP, so each noisy gradient release is
\[
    \rho_0=\frac{(2L/n)^2}{2\sigma^2}=\frac{2L^2}{n^2\sigma^2}.
\]
By adaptive composition of zCDP mechanisms \cite{cryptoeprint:concentratedDP}, the $T$ releases are $T\rho_0=2TL^2/(n^2\sigma^2)$-zCDP. Substituting $\sigma^2=2L^2T/(n^2\rho)$ gives total privacy $\rho$-zCDP.

For utility, apply \Cref{lemma:noisy_pgd_generic}. Since $\sum_{t=0}^{T-1}\|n_t\|_2^2/\sigma^2\sim \chi^2_{dT}$, chi-square concentration gives, with probability at least $1-\zeta/2$,
\[
    \frac1T\sum_{t=0}^{T-1}\|n_t\|_2^2
    \leq
    O\left(d\sigma^2+\frac{\sigma^2\log(1/\zeta)}{T}\right).
\]
Thus $B_2=O(d\sigma^2+\sigma^2\log(1/\zeta)/T)$. Next, conditioned on the history before round $t$, $\langle n_t,\theta_t-\theta^\ast_\gD\rangle$ is mean-zero Gaussian with variance at most $\sigma^2D^2$, because $\|\theta_t-\theta^\ast_\gD\|_2\leq D$. Hence, with probability at least $1-\zeta/2$,
\[
    \left|\frac1T\sum_{t=0}^{T-1}\langle n_t,\theta_t-\theta^\ast_\gD\rangle\right|
    \leq
    O\left(D\sigma\sqrt{\frac{\log(1/\zeta)}{T}}\right),
\]
so $B_1=O(D\sigma\sqrt{\log(1/\zeta)/T})$. The first displayed utility bound follows from \Cref{lemma:noisy_pgd_generic}.

Finally, substitute $\eta=D/(L\sqrt T)$ and
$\sigma^2=2L^2T/(n^2\rho)$. The first two terms are
$O(DL/\sqrt T)$, and the remaining noise terms become
\[
    \eta d\sigma^2
    =
    O\left(
        \frac{DLd\sqrt T}{n^2\rho}
    \right),
    \qquad
    \frac{\eta\sigma^2\log(1/\zeta)}{T}
    =
    O\left(
        \frac{DL\log(1/\zeta)}{n^2\rho\sqrt T}
    \right),
\]
and
\[
    D\sigma\sqrt{\frac{\log(1/\zeta)}{T}}
    =
    O\left(
        \frac{DL\sqrt{\log(1/\zeta)}}{n\sqrt\rho}
    \right).
\]
Thus the $T$-dependent terms are
\[
    O\left(
        \frac{DL}{\sqrt T}
        \left(
            1+\frac{\log(1/\zeta)}{n^2\rho}
        \right)
        +
        \frac{DLd\sqrt T}{n^2\rho}
    \right).
\]
Choose
\[
    T=\frac{n^2\rho+\log(1/\zeta)}{d}.
\]
To verify the balancing, let $A=n^2\rho$ and $B=\log(1/\zeta)$. Then the
$T$-dependent terms become
\[
    O\left(
        DL\sqrt d
        \frac{1+B/A}{\sqrt{A+B}}
        +
        DL\sqrt d
        \frac{\sqrt{A+B}}{A}
    \right).
\]
Since
\[
    \frac{1+B/A}{\sqrt{A+B}}
    =
    \frac{A+B}{A\sqrt{A+B}}
    =
    \frac{\sqrt{A+B}}{A},
\]
both terms are of the same order. Hence the $T$-dependent terms are
\[
    O\left(
        DL\sqrt d
        \frac{\sqrt{A+B}}{A}
    \right)
    =
    O\left(
        \frac{DL\sqrt d}{n\sqrt\rho}
        \sqrt{
            1+\frac{\log(1/\zeta)}{n^2\rho}
        }
    \right).
\]
Using $\sqrt{1+x}\leq 1+\sqrt{x}$, this is
\[
    O\left(
        \frac{DL\sqrt d}{n\sqrt\rho}
        +
        \frac{DL\sqrt{d\log(1/\zeta)}}{n^2\rho}
    \right).
\]
Under the assumption $\log(1/\zeta)\leq n^2\rho$, the second term is dominated
by the first. Adding the martingale/noise-inner-product term gives
\[
    \ell(\bar\theta,\gD)-\ell(\theta^\ast_\gD,\gD)
    \leq
    O\left(
        \frac{DL\sqrt d}{n\sqrt\rho}
        +
        \frac{DL\sqrt{\log(1/\zeta)}}{n\sqrt\rho}
    \right).
\]
This proves the stated bound.
\end{proof}

\end{document}